\documentclass[11pt,a4paper]{article}
\usepackage{a4wide}
\usepackage{graphicx}
\usepackage[fleqn]{amsmath}
\usepackage{amssymb}
\usepackage{amsthm}
\usepackage{bm}
\usepackage{url}
\newcommand{\dd}{~\mathrm{d}}
\newcommand{\sgn}{\mathrm{sgn}\,}
\newcommand{\LandauO}{\mathcal{O}}
\newcommand{\transpose}{^{\mathrm{T}}}
\newcommand{\bbbr}{\mathbb{R}}
\theoremstyle{plain}
\newtheorem{theorem}{Theorem}
\newtheorem{corollary}{Corollary}
\theoremstyle{remark}
\newtheorem{remark}{Remark}
\allowdisplaybreaks
\begin{document}
\title{Analysis of Amoeba Active Contours}
\author{Martin Welk\\ University for Health Sciences, Medical Informatics and Technology (UMIT),\\ Eduard-Walln\"ofer-Zentrum 1, 6060 Hall/Tyrol, Austria\\ \url{martin.welk@umit.at} }
\date{July 10, 2014}
\maketitle

\begin{abstract} Subject of this paper is the theoretical analysis of structure-adaptive median filter algorithms that approximate curvature-based PDEs for image filtering and segmentation.  These so-called morphological amoe\-ba filters are based on a concept introduced by Lerallut et al. They achieve similar results as the well-known geodesic active contour and self-snakes PDEs.  In the present work, the PDE approximated by amoeba active contours is derived for a general geometric situation and general amoeba metric. This PDE is structurally similar but not identical to the geodesic active contour equation.  It reproduces the previous PDE approximation results for amoeba median filters as special cases.  Furthermore, modifications of the basic amoeba active contour algorithm are analysed that are related to the morphological force terms frequently used with geodesic active contours.  Experiments demonstrate the basic behaviour of amoeba active contours and its similarity to geodesic active contours.

\bigskip \noindent%
 \end{abstract}

\sloppy \section{Introduction} Introduced by Lerallut et al.\ \cite{Lerallut-ismm05,Lerallut-IVC07}, morphological amoeba filtering is a class of discrete image filtering procedures based on image-adaptive structuring elements. These structuring elements are defined by a so-called amoeba metric that combines spatial proximity and grey-value similarity.  Amoeba filters adapt flexibly to image structures.  For example, iterated \emph{amoeba median filtering} (AMF) improves the favourable edge-preserving denoising capabilities of traditional iterated median filtering \cite{Tukey-Book71} by removing its tendency to dislocate edges, and introducing even edge-enhancing behaviour.

This paper is an extended version of the conference paper \cite{Welk-ssvm13}. Continuing the author's earlier work with co-authors \cite{Welk-ssvm11,Welk-JMIV11}, it is concerned with comparing AMF methods to two curvature-based PDEs of image processing. Firstly, we consider \emph{geodesic active contours} \cite{Caselles-iccv95,Caselles-IJCV97,Kichenassamy-iccv95,Kichenassamy-ARMA96}
\begin{equation} \label{gac} u_t = \lvert\bm{\nabla} u\rvert~\mathrm{div}\left( g(\lvert\bm{\nabla} f\rvert^2)\, \frac{\bm{\nabla} u}{\lvert\bm{\nabla} u\rvert} \right) \end{equation}
which can be used to segment a given image $f$ by evolving a contour towards regions of high contrast in $f$. The evolving contour is encoded as zero-level set of the function $u$. The (decreasing, nonnegative) edge-stopping function $g$ can be chosen e.g.\ as a Perona-Malik-type function \cite{Perona-PAMI90}
\begin{equation}\label{peronamalikg} g(s^2)=\frac1{1+s^2/\lambda^2}\;,\qquad \lambda>0\;.  \end{equation}
Secondly, we are interested in \emph{self-snakes} \cite{Sapiro-icip96}, a PDE filter for a single image $u$ that is obtained from \eqref{gac} by identifying $f$ with the evolving function $u$.  \medskip

As shown in \cite{Welk-JMIV11}, AMF is linked to the self-snakes equation in a way similar to the connection of traditional median filtering to (mean) curvature motion \cite{Alvarez-SINUM92} that was proven by Guichard and Morel \cite{Guichard-sana97}: One amoeba median filtering step asymptotically approximates a time step of size $\varrho^2/6$ of an explicit time discretisation for the self-snakes PDE when the radius $\varrho$ of the structuring element goes to zero.  The exact shape of the (decreasing, nonnegative) edge-stopping function $g$ depends on the specific choice of the amoeba metric, with the Perona-Malik-type function \eqref{peronamalikg} being associated to the $L^2$ amoeba metric.

Building on this amoeba/self-snakes connection, \cite{Welk-ssvm11} proposed a morphological amoeba algorithm for active contour segmentation. Experimentally, this process behaves similar to geodesic active contours, with a tendency to refined adaptation to structure details, see \cite[Fig.~2]{Welk-ssvm11}.  Analysis in \cite{Welk-ssvm11} was restricted to a rotationally symmetric situation where asymptotic equivalence to geo\-de\-sic active contours \eqref{gac} could be proven.  A more comprehensive asymptotic equivalence result proven in \cite{Welk-ssvm13} for the case of an $L^2$ amoeba metric brings out that perfect equivalence between amoeba and geodesic active contours does not hold in general geometric situations but amoeba active contours approximate a PDE similar to geodesic active contours.  The main goal of the present paper is to extend this theoretical analysis to general amoeba metrics. All previous findings on amoeba median filters can be recovered from the new general result as special cases.

Already in \cite{Welk-ssvm11} the possibility was mentioned to introduce into amoeba active contours a force term similar to the \emph{balloon force} proposed by Cohen \cite{Cohen-CVGIPIU91} or its modifications in more recent works \cite{Caselles-NUMA93,Kichenassamy-ARMA96,Malladi-PAMI95}. The benefits of such a force term are that contour evolution in homogeneous image regions is accelerated, that evolution can be prevented from getting caught in undesired local minima, and that initialisation is also possible with contours inside the region to be segmented. The corresponding modifications to the amoeba active contour algorithm that were proposed in \cite{Welk-ssvm11} have not been analysed theoretically so far. To close this gap is a further goal of the present paper.

\paragraph{Our contribution.} We extend the analytical investigation of amoeba median algorithms.  First, we derive the PDE corresponding to the amoeba active contour method in a general geometric situation and for a general amoeba metric. As already in the $L^2$ case \cite{Welk-ssvm13}, this PDE is no longer fully identical to the geodesic active contour equation.  The proof strategy follows that introduced in \cite{Welk-ssvm13}, which differs substantially from the one used in \cite{Welk-ssvm11,Welk-JMIV11}.  Based on the approximation result, qualitative differences between geodesic and amoeba active contours are discussed for the $L^2$ amoeba metric case.  {\sloppy\par}

In a further step, we analyse in detail the modification of the amoeba active contour method by a bias that was proposed in \cite{Welk-ssvm11} to mimick the force terms often used in connection with geodesic active contours. In the context of this analysis, we will also propose a further variant of this bias.

While the focus in the present paper is on theoretical analysis, we demonstrate segmentation via amoeba active contours with two experiments, which are extended from \cite{Welk-ssvm11}.

\paragraph{Structure of the paper.} We give a short account of the basic concepts of amoeba filtering in Section~\ref{sec-amf}.  Our main theoretical result on PDE approximation is proven in Section~\ref{sec-aaca}. Relations to previous results on PDE approximation by amoeba median filtering algorithms are established in Section~\ref{sec-spec}.  On the ground of the PDE approximation result, a comparison between amoeba active contours and geodesic active contours is made in Section~\ref{sec-aacgac}.  Force terms in active contour methods are considered in Section~\ref{sec-bias}.  Experiments are presented in Section~\ref{sec-exp}.  The paper ends with a conclusion in Section~\ref{sec-conc}.

\section{Amoeba Filters}\label{sec-amf} In this section we recall shortly the definition of amoeba metrics and amoeba filters. We assume that a 2D image is given as a smooth function $f:\varOmega\to\bbbr$ on a closed domain $\varOmega\subset\bbbr^2$.

\subsection{Morphological amoebas} Following the spatially continuous formulation of the amoeba framework in \cite{Welk-ssvm11,Welk-JMIV11}, we associate with $f$ the image manifold $\varGamma\subset\bbbr^3$ consisting of the points $(x,y,\beta\,f(x,y))$.  The construction of \emph{morphological amoebas} as adaptive structuring elements relies on introducing an \emph{amoeba metric} on $\varGamma$.

To this end, we start by choosing a function $\nu:\bbbr\to\bbbr_0^+$ with $\nu(-s)=\nu(s)$, which is increasing on $\bbbr_0$, and for which $\lVert(s,t)\rVert_{\nu}:=t\,\nu(\lvert s/t\rvert)$ is a norm in $\bbbr^2$.  In the following, we will give a general definition of an amoeba metric based on $\nu$ but pay special attention to the following two cases:
\begin{itemize} \item the $L^2$ amoeba metric with $\nu(s)=\sqrt{1+s^2}$, where $\lVert(s,t)\rVert_{\nu}$ is the Euclidean norm, and \item the $L^1$ amoeba metric given by $\nu(s)=1+\lvert s\rvert$.  \end{itemize}

To construct from $\nu$ an amoeba metric, we consider regular curves $\bm{c}:[0,1]\to\varGamma$ with $t\mapsto \bm{c}(t)\equiv(x(t),y(t),f(x(t),y(t)))$ and $\mathrm{d}\bm{c}/\mathrm{d}t =: (\dot{x},\dot{y},\dot{f})$.  For such a curve, we define a curve length $L_{\nu}(\bm{c})$ as
\begin{equation} L_{\nu}(\bm{c}) := \int\limits_{0}^{1} \nu\left(\frac{\beta\,\dot{f}}{\sqrt{\dot{x}^2+\dot{y}^2}}\right)\, \sqrt{\dot{x}^2+\dot{y}^2}\, \dd t \;.  \label{Lnuc} \end{equation}
Note that in the $L^2$ case, $L_{\nu}(\bm{c})$ is just the standard curve length on $\varGamma$ induced by the Euclidean metric of the surrounding space $\bbbr^3$.

The \emph{amoeba distance} $d(\bm{p},\bm{q})$ between two points $\bm{p}$, $\bm{q}$ of the image domain is then the minimum of $L_\nu(\bm{c})$ among all curves $\bm{c}$ connecting $\bm{p}$ with $\bm{q}$. (The minimising curve $\bm{c}$ is called geodesic between $\bm{p}$ and $\bm{q}$.)

In \eqref{Lnuc}, the use of the Euclidean norm $\sqrt{\dot{x}^2+\dot{y}^2}$ in the spatial component ensures rotational invariance of the amoeba metric, while the combination of spatial and tonal distances is governed by $\nu$. The factor $\beta$ is a contrast scale that balances the spatial and tonal information.

The choice of $\beta$ in practical image filtering with amoeba filters is not quite obvious. The same holds for its scaling behaviour when resampling the image.  We will not discuss here strategies how to choose $\beta$.  However, in the light of our results later in this paper the choice of $\beta$ appears analogous to the choice of contrast parameters for Perona-Malik diffusion \cite{Perona-PAMI90}, for which heuristics based on statistics of gradient magnitudes in the image have been proposed, see e.g.\ \cite{Chao-IVC08}.

For amoeba filters \cite{Lerallut-ismm05,Lerallut-IVC07}, one defines a structuring element $\mathcal{A}_{\bm{p}}$ for each point $\bm{p}\in\varOmega$ as the set of all $\bm{q}\in\varOmega$ such that $d(\bm{p},\bm{q})\le\varrho$, where the global parameter $\varrho$ is the \emph{amoeba radius}.  Note that the same kind of image patches has also been used in \cite{Spira-TIP07} for short-time Beltrami kernels.

\subsection{Continuous-scale amoeba filtering formulation} With the so defined structuring elements several morphological filters can be applied straightforward. For the purpose of the present work, morphological filters are characterised by their invariance under automorphisms of the image plane (translations, rotations) and under strictly monotonically increasing transformations of the intensities. This notion, compare e.g.\ \cite{Maragos-icip09}, naturally includes median and other rank-order filters.

In particular, for amoeba median filtering (AMF), the median of the intensity values of the given image $f$ within $\mathcal{A}_{\bm{p}}$ becomes the filtered intensity at $\bm{p}$. Like traditional median filtering, this filter can be applied iteratively. This process was studied in \cite{Welk-JMIV11}.

\subsection{Amoeba active contours} The amoeba active contour method described in \cite{Welk-ssvm11} acts in a similar way: Structuring elements are determined as before but on the basis of the given image $f$, and are used for median-filtering the evolving level-set function $u$.  In analysing amoeba active contours, the amoeba contrast parameter $\beta$ can be fixed to $1$ since a change of this parameter is equivalent to a simple rescaling of the steering function $f$.

\subsection{Discrete amoeba filtering algorithms} Practically, computations are carried out on discrete images.  To this end, a discrete version of the above-mentioned amoeba distance is defined by restricting curves to paths in the neighbourhood graph of the image grid, either with 4-neighbourhoods as in \cite{Lerallut-ismm05,Lerallut-IVC07} or with 8-neighbourhoods as in \cite{Welk-ssvm11,Welk-JMIV11}. More sophisticated constructions using geometric distance transforms \cite{Borgefors-CVGIP86,Ikonen-IVC05} would be possible but are not investigated here due to our focus on space-continuous analysis.

\section{Analysis of Amoeba Active Contours}\label{sec-aaca} We study an amoeba median filter in which $f$ is a smooth function from which the amoeba structuring elements are generated, and $u$ is another smooth function, to which the median filter is applied.  Note that the role played by $f$ here can be compared to that of a ``pilot image'' in some works on adaptive morphology, see e.g.\ \cite{Lerallut-IVC07}. In such a setup, the pilot image usually is some prefiltered version of the same input image that is processed later on by the morphological filter, with the structuring elements derived from the pilot image.  This setting (which we do not consider in detail) is obviously also covered by our analysis in the sequel. However, our hypothesis does not require any relation between $f$ and $u$.

In our subsequent analysis, local orthonormal bases aligned to the gradient and level-line directions of both functions will play an important role.  Given a location $\bm{x}_0$ in the image domain, we will therefore denote by $\bm{\chi}=(\cos\varphi,\sin\varphi)\transpose$ the normalised gradient vector of $f$ at $\bm{x}_0$. The unit vector $\bm{\zeta}\perp\bm{\chi}$ then indicates the local level line direction of $f$.  At locations with $\bm{\nabla}f=\bm{0}$, the directions $\bm{\chi}$, $\bm{\zeta}$ are not well-defined.  For the following derivations we therefore assume that $\bm{\nabla}f\ne0$. However, we will see that the resulting PDE still describes a well-defined evolution.

Analogously, we denote by $\bm{\eta}$ a normalised gradient vector for $u$, and by $\bm{\xi}\perp\bm{\eta}$ the unit vector in the level line direction. The angle between the gradient directions will be called $\alpha$, such that $\bm{\eta}=\bigl(\cos(\varphi+\alpha), \sin(\varphi+\alpha)\bigr)\transpose$.  We will prove the following fact.

\begin{theorem}\label{thm} One step of amoeba median filtering of a smooth function $u$ governed by amoebas generated from $f$ with an amoeba radius of $\varrho$ asymptotically approximates for $\varrho\to0$ a time step of size $\tau=\varrho^2/6$ of an explicit time discretisation for the PDE \begin{align} u_t &= \frac{u_{\bm{\xi\xi}}} {\nu(\lvert\bm{\nabla}f\rvert\,\sin\alpha)^2}-\frac32\,\nu(\lvert\bm{\nabla}f\rvert\,\sin\alpha)\, \lvert\bm{\nabla}u\rvert \times{}\notag\\*&\qquad\quad{}\times\bigl(J_1(\lvert\bm{\nabla}f(\bm{x})\rvert,\alpha)\,f_{\bm{\zeta\zeta}}  +2\,J_2(\lvert\bm{\nabla}f(\bm{x})\rvert,\alpha)\,f_{\bm{\zeta\chi}}  +J_3(\lvert\bm{\nabla}f(\bm{x})\rvert,\alpha)\,f_{\bm{\chi\chi}}\bigr) \label{genaacmed} \end{align} where $J_1$, $J_2$, $J_3$ at the location $\bm{x}$ are given by \begin{align} J_1(s,\alpha) &= \int\limits_{\alpha-\pi/2}^{\alpha+\pi/2} \frac{\nu'(s\cos\vartheta)} {\nu(s\cos\vartheta)^4}\, \sin^2\vartheta \dd\vartheta\;, \label{J1} \\ J_2(s,\alpha) &= \int\limits_{\alpha-\pi/2}^{\alpha+\pi/2} \frac{\nu'(s\cos\vartheta)} {\nu(s\cos\vartheta)^4}\, \sin\vartheta\cos\vartheta \dd\vartheta\;, \label{J2} \\ J_3(s,\alpha) &= \int\limits_{\alpha-\pi/2}^{\alpha+\pi/2} \frac{\nu'(s\cos\vartheta)} {\nu(s\cos\vartheta)^4}\, \cos^2\vartheta \dd\vartheta\;.  \label{J3} \end{align} \end{theorem}

\begin{remark} The PDE \eqref{genaacmed} describes an evolution process similar but not identical to geodesic active contours.  An interpretation of the individual terms on the right-hand side of the PDE on an intuitive level is not straightforward.  Several detailed results in Sections~\ref{sec-spec} and \ref{sec-aacgac} give an account of communities and differences between \eqref{genaacmed} and geodesic active contours. In particular, Corollaries~\ref{cor1} and \ref{cor2} in Section~\ref{sec-spec} specify cases where \eqref{genaacmed} coincides with a geodesic active contour or self-snakes equation.  Corollaries~\ref{cor-aacl2} and \ref{cor-aacl1} specialise \eqref{genaacmed} to the $L^2$ and $L^1$ amoeba metrics where it is easier to compare to geodesic active contours. Section~\ref{sec-aacgac} studies the differences between both evolutions in the $L^2$ case based on special cases.  \end{remark}

\begin{remark} As pointed out above, $\bm{\chi}$ and $\bm{\eta}$ are well-defined only at locations where the gradient of $f$ does not vanish.  Let us study therefore what happens with the evolution \eqref{genaacmed} at singular points where $\bm{\nabla}f$ vanishes.  Firstly, if $\nu$ fulfils $\nu'(0)=0$, one has simply $J_1(0,\alpha)=J_2(0,\alpha)=J_3(0,\alpha)=0$ (remember that $\nu(0)>0$ holds by definition), making the right-hand side of \eqref{genaacmed} collaps into the well-defined expression $u_{\bm{\xi\xi}}/\nu(0)^2$.  If $\nu'(0)\ne0$, we notice that the set of singular points can be subdivided into boundary and interior points. Interior points form constant regions of $f$ such that the second derivatives of $f$ vanish, too, making the right-hand side of \eqref{genaacmed} collaps into the same well-defined expression as mentioned before. The remaining boundary points, however, form curves and isolated points. In both cases, smoothness of the gradient field $\bm{\nabla}f$ implies that the resulting definition gaps of the right-hand side of \eqref{genaacmed} can be continuously closed such that again a unique and continuous evolution $u$ is obtained.  \end{remark}

\begin{remark} In the theorem and its forthcoming proof we have fixed $\beta$ to $1$, as mentioned before. However, it is obvious how to adapt the statement to variable $\beta$ because it just takes to replace $f$ with $\beta\,f$ in all places.  \end{remark}

The remainder of the present Section~\ref{sec-aaca} is devoted to the proof of this theorem. In Subsection~\ref{ssec-proofstrat} the overall strategy of the proof is outlined. It involves two main steps that are subsequently treated in Subsection~\ref{ssec-contour} and Subsection~\ref{ssec-imbalance}, respectively.

\subsection{Remark on the proof strategy} \label{ssec-proofstrat} In \cite{Welk-ssvm11,Welk-JMIV11}, related but more restricted results were proven (which are repeated as Corollaries~\ref{cor1} and \ref{cor2} in Section~\ref{sec-spec} below).  The proofs in \cite{Welk-ssvm11,Welk-JMIV11} were based on measuring level line segments within the amoeba. The structure of $f$ and $u$ was represented by their Taylor coefficients up to second order in the calculations. This strategy is well suitable for the amoeba median filter considered in \cite{Welk-JMIV11} where the same image from which the structuring elements are obtained is also being filtered.  It is also useful when analysing more general amoeba filters than median filters, which is a subject of forthcoming work.  For analysing amoeba active contours the same approach is still manageable in the special case treated in \cite{Welk-ssvm11}.  However, the complexity of such calculations would increase a lot in the general case we are about to discuss.

In the following proof of the theorem we therefore pursue a different strategy that was introduced in the proof in \cite{Welk-ssvm13} in the more restricted case of the $L^2$ amoeba metric.  Instead of measuring areas of segments of amoebas, this approach considers sectors of amoebas via a polar coordinate representation.  Level lines other than the one through the amoeba centre are not considered directly any more.

\subsection{Finding the amoeba contour} \label{ssec-contour} As the first part of our proof of Theorem~\ref{thm}, we want to determine the shape of the amoeba $\mathcal{A}:=\mathcal{A}_{\bm{x}_0}$ around a point $\bm{x}_0\in\varOmega$. To this end, we start by considering the 1D case: given $f:\bbbr\to\bbbr$, we seek $z_{\pm}\in\bbbr$ such that the arc-length of the image graph of $f$ between $x_0$ and each of $x_0+z_+$, $x_0-z_-$ equals $\varrho$. Certainly, $z_\pm\le\varrho$.

Using Taylor expansions for $f$ and $\nu$, we have for the arc-length from $x_0$ to $x_0+z$ (where $z>0$) \begin{align} \kern1em&\kern-1em \int\limits_{x_0}^{x_0+z}\nu(f'(x))\dd x =z\,\nu(f'(x_0)) +\frac{z^2}2\,\nu'(f'(x_0))\,f''(x_0) +\LandauO(z^3)\;.  \label{arclengthzplus} \end{align} Equating this to $\varrho$, and taking into account that $\LandauO(z^3)$ is also $\LandauO(\varrho^3)$ within the amoeba, yields a quadratic equation in $z$ with the solutions \begin{align} z_{1,2} &= \frac{\nu(f'(x_0))} {\nu'(f'(x_0))\,f''(x_0)} \left(-1\pm\sqrt{1+ \frac{2\,\varrho\,\nu'(f'(x_0))\,f''(x_0)} {\nu(f'(x_0))^2}}\,\right) +\LandauO(\varrho^3) \;.  \label{zplus} \end{align} The first solution with the ``$+$'' sign, i.e.\ $z_1$, is in fact the sought $z_+$ (because of $z>0$). Note that the second, negative solution, $z_2$, is \emph{not} $z_-$ but refers, for small $\varrho$, to a location far outside the amoeba, and does not go to $0$ when $\varrho\to0$. In fact, this second solution is just a spurious solution introduced by our perturbation approach via the truncated Taylor series (a common behaviour in this kind of approximation).

To find also $z_-$, an expression analogous to \eqref{arclengthzplus} is written down for the arc length from $x_0-z$ to $x_0$ (with $z>0$), leading again to a quadratic equation with $z_-$ as one of its solutions.

Using the Taylor expansion $\sqrt{1+t} = 1 + \frac12 t - \frac18 t^2 + \LandauO(t^3)$ for the square root in \eqref{zplus} and the analogous expression for $z_-$, both results can be combined into \begin{align} z_\pm &= \frac{\varrho}{\nu(f'(x_0))} \mp\frac{\varrho^2\,\nu'(f'(x_0))\,f''(x_0)} {2\,\nu(f'(x_0))^3} +\LandauO(\varrho^3)\;.  \end{align} Turning to the 2D case, we approximate each shortest path in the amoeba metric from $\bm{x}_0$ to a point on the amoeba contour by a Euclidean straight line in the image plane. This introduces only an $\LandauO(\varrho^3)$ error for the path length. We consider now the straight line through $\bm{x}_0$ in the direction of a given unit vector $\bm{v}\in\bbbr^2$.  By our previous 1D result, with the directional derivatives $f_{\bm{v}}(\bm{x}_0)=\langle\bm{v}, \bm{\nabla}f(\bm{x}_0)\rangle$ and $f_{\bm{vv}}(\bm{x}_0)=\bm{v}\transpose\, \mathrm{D}^2f(\bm{x}_0)\,\bm{v}$, we see that said straight line intersects the amoeba contour at $\bm{x}_0\pm z_{\pm}(\bm{v})\cdot\bm{v}$ with \begin{align} z_\pm(\bm{v}) &= \frac{\varrho}{\nu(\langle\bm{v},\bm{\nabla}f(\bm{x}_0)\rangle)} \mp\frac{\varrho^2\, \nu'(\langle\bm{v},\bm{\nabla}f(\bm{x}_0)\rangle)\, \bm{v}\transpose\,\mathrm{D}^2f(\bm{x}_0)\,\bm{v}} {2\,\nu(\langle\bm{v},\bm{\nabla}f(\bm{x}_0)\rangle)^3} +\LandauO(\varrho^3)\;.  \label{xipm-v} \end{align}

\subsection{Contributions to the amoeba median} \label{ssec-imbalance} The second part of our proof of Theorem~\ref{thm} consists in analysing the median of $u$ within the structuring element $\mathcal{A}$ whose polar coordinate representation has been derived in the preceding subsection.

This median equals $u(\bm{x}_0)$ if (a) the amoeba is point-symmetric w.r.t.\ $\bm{x}_0$, and (b) the level lines of $u$ are straight: The central level line $u(\bm{x})=u(\bm{x}_0)$ of $u$ then bisects $\mathcal{A}$, i.e.\ $\mathcal{A}_+:=\{\bm{x}\in\mathcal{A}~|~u(\bm{x})\ge u(\bm{x}_0)\}$ and $\mathcal{A}_-:=\{\bm{x}\in\mathcal{A}~|~u(\bm{x})\le u(\bm{x}_0)\}$ have equal area. For a similar bisection approach in a gradient descent for segmentation compare \cite{HoltzmanGazit-TIP06,Kimmel-bookchapter03}.

\begin{figure*}[t] \setlength{\unitlength}{0.001\textwidth} \begin{picture}(1000,350) \put(60,0){\includegraphics[height=0.35\textwidth]{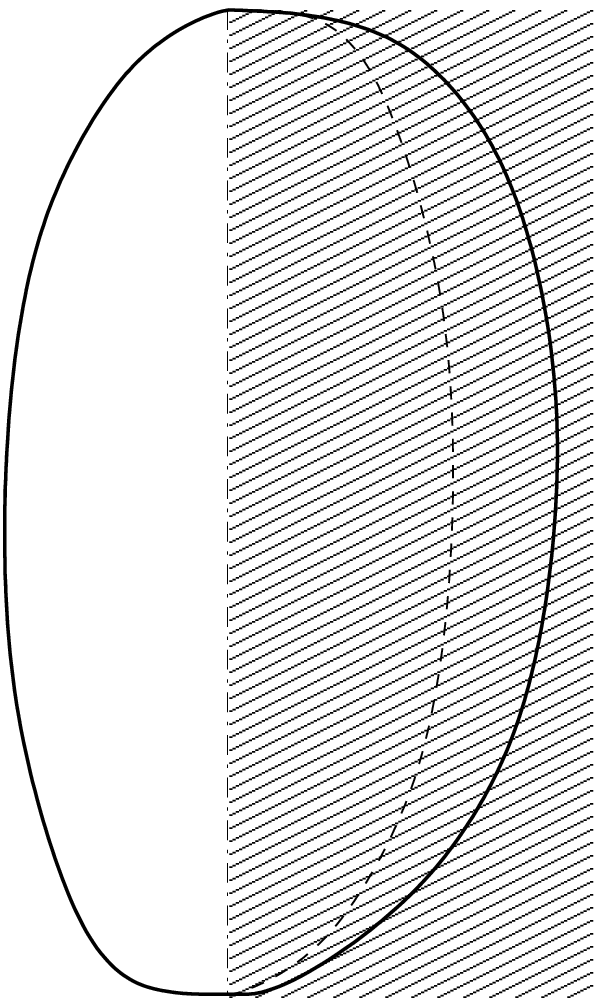}} \put(40,0){(a)} \put(135,165){$\bullet$} \put(112,160){$\bm{x}_0$} \put(140,263){\line(1,0){10}} \put(152,260){{\tiny straight}} \put(152,247){{\tiny level}} \put(152,234){{\tiny line}} \put(85,120){$\mathcal{A}_-$} \put(171,120){$\mathcal{A}_+$} \put(240,200){\line(1,0){25}} \put(265,194){$\varDelta_1$} \put(236,290){\line(1,0){12}} \put(252,300){{\tiny asymmetric}} \put(252,287){{\tiny amoeba}} \put(380,0){\includegraphics[height=0.35\textwidth]{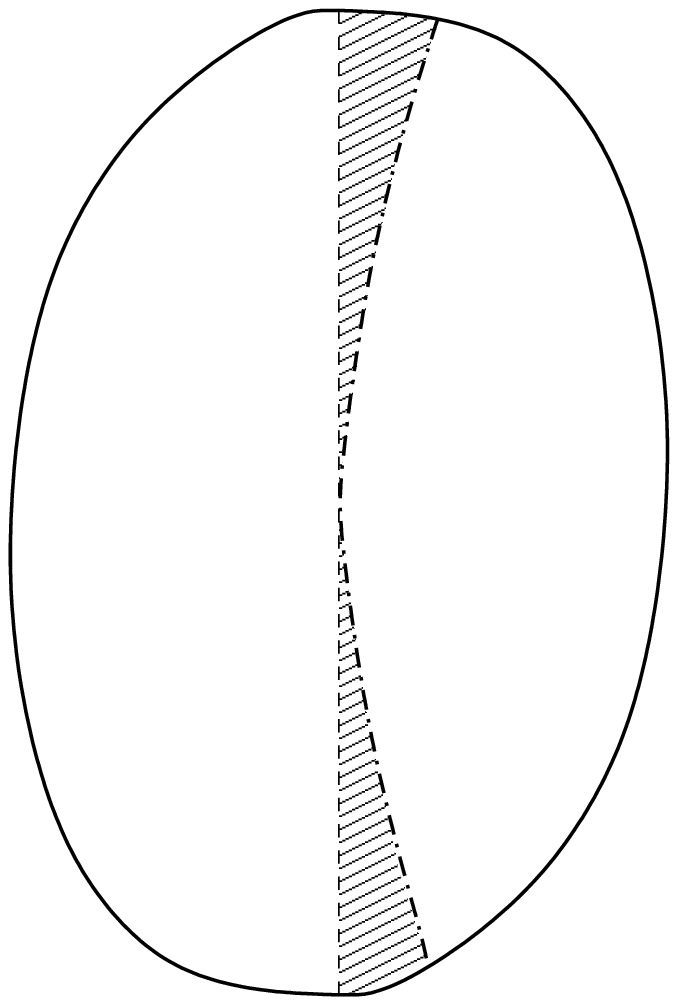}} \put(492,165){$\bullet$} \put(469,160){$\bm{x}_0$} \put(510,263){\line(1,0){10}} \put(522,260){{\tiny curved}} \put(522,247){{\tiny level}} \put(522,234){{\tiny line}} \put(432,120){$\mathcal{A}_-$} \put(528,120){$\mathcal{A}_+$} \put(510,50){\line(1,0){15}} \put(525,44){$\varDelta_2$} \put(593,290){\line(1,0){12}} \put(609,300){{\tiny symmetric}} \put(609,287){{\tiny amoeba}} \put(360,0){(b)} \put(740,0){\includegraphics[height=0.35\textwidth]{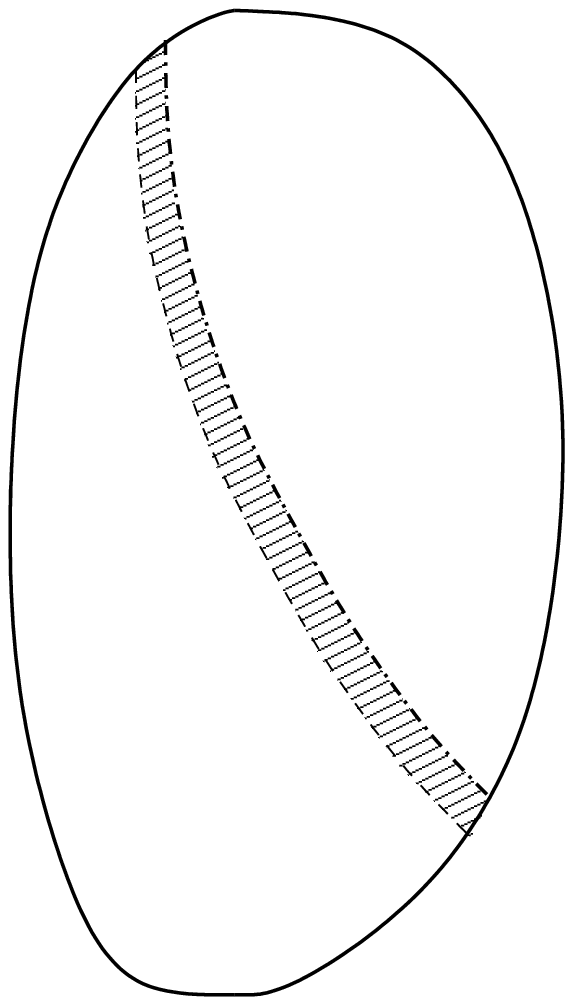}} \put(817,165){$\bullet$} \put(794,160){$\bm{x}_0$} \put(798,290){\line(1,0){12}} \put(812,287){{\tiny shifted}} \put(812,274){{\tiny level}} \put(812,261){{\tiny line}} \put(780,100){$\mathcal{A}_-$} \put(860,220){$\mathcal{A}_+$} \put(860,120){\line(1,0){20}} \put(880,114){$\varDelta$} \put(720,0){(c)} \end{picture} \caption{\label{fi-delta}\textbf{Left to right:} \textbf{(a)} Area difference $\varDelta_1$ in an asymmetric amoeba with straight level lines. -- \textbf{(b)} Area difference $\varDelta_2$ in a symmetric amoeba with curved level lines. -- \textbf{(c)} Compensation of the area difference $\varDelta$ by shifting the central level line (schematic). -- From \cite{Welk-ssvm13}.} \end{figure*}

Deviations from conditions (a) and (b) lead to imbalances between $\mathcal{A}_+$ and $\mathcal{A}_-$. The median is determined by the shift of the central level line that is necessary to compensate for the resulting area difference.  The separate area effects of asymmetry of the amoeba, and curvature of $u$'s level lines are of order $\mathcal{O}(\varrho^3)$, while any cross-effects are at least of order $\mathcal{O}(\varrho^4)$, and can be neglected for the purpose of our analysis. Therefore, the two effects can be studied independently by considering the two special cases in which only one of the effects takes place.

\subsubsection{Asymmetry of the amoeba} We start by analysing the effect of asymmetries of the point set $\mathcal{A}$, compare Figure~\ref{fi-delta}(a).  As the amoeba shape is governed by $f$, we will use the $\bm{\zeta}$, $\bm{\chi}$ local coordinates.  For an arbitrary unit vector $\bm{v}=\bigl(\cos(\varphi+\vartheta),\sin(\varphi+\vartheta)\bigr)\transpose$ we have then \begin{align} f_{\bm{v}}(\bm{x}_0) &= \lvert\bm{\nabla}f(\bm{x}_0)\rvert\cos\vartheta\;, \label{fv} \\ \bm{v}\transpose\,\mathrm{D}^2f(\bm{x}_0)\,\bm{v} &= f_{\bm{\zeta\zeta}}\sin^2\vartheta +2\,f_{\bm{\zeta\chi}}\cos\vartheta\,\sin\vartheta +f_{\bm{\chi\chi}}\cos^2\vartheta \end{align} which can be inserted into \eqref{xipm-v} to obtain $z_\pm(\varphi+\vartheta):=z_\pm(\bm{v})$.

Consider now the case in which $u$ has straight level lines; remember that $\varphi+\alpha$ is the direction angle of its gradient direction.  Since the amoeba shape is given by $z_\pm(\bm{v})$ in polar coordinates, the areas of $\mathcal{A}_+$ and $\mathcal{A}_-$ can be written down using the standard integral for the area enclosed by a function graph in polar coordinates as \begin{align} \lvert\mathcal{A}_{\pm}\rvert &= \int\limits_{\varphi+\alpha-\pi/2}^{\varphi+\alpha+\pi/2} \frac12 z_{\pm}(\vartheta)^2\dd\vartheta\;, \end{align} such that the sought area difference is then obtained as \begin{align} \varDelta_1&:= \lvert\mathcal{A}_+\rvert-\lvert\mathcal{A}_-\rvert = \int\limits_{\varphi+\alpha-\pi/2}^{\varphi+\alpha+\pi/2} \bigl(z_+(\vartheta)-z_-(\vartheta)\bigr)\, \frac{z_+(\vartheta)+z_-(\vartheta)}2\dd\vartheta +\LandauO(\varrho^4) \;.  \label{Delta1} \end{align} The integral on the right-hand side equals \begin{align} & -\varrho^3 \int\limits_{\alpha-\pi/2}^{\alpha+\pi/2} \frac{\nu'(\lvert\bm{\nabla}f(\bm{x}_0)\rvert\cos\vartheta)} {\nu(\lvert\bm{\nabla}f(\bm{x}_0)\rvert\cos\vartheta)^4} \left( f_{\bm{\zeta\zeta}}\,\sin^2\vartheta +2\,f_{\bm{\zeta\chi}}\,\cos\vartheta\,\sin\vartheta +f_{\bm{\chi\chi}}\,\cos^2\vartheta\right) \dd\vartheta\;, \end{align} thus we have \begin{align} \varDelta_1 &= -\varrho^3\left( J_1\,f_{\bm{\zeta\zeta}}+2\,J_2\,f_{\bm{\zeta\chi}}+J_3\,f_{\bm{\chi\chi}} \right) +\LandauO(\varrho^4) \label{Del1J1J2J3} \end{align} where $J_1$, $J_2$, $J_3$ are as stated in Theorem~\ref{thm}.

\subsubsection{Curvature of the level lines} The second source of area imbalance between $\mathcal{A}_+$ and $\mathcal{A}_-$ is the curvature of the level line of $u$ through $\bm{x}_0$.  To study this contribution, we consider the case in which the amoeba is symmetric, such that the level line curvature is the single source of area imbalance.

Using the $\bm{\xi}$, $\bm{\eta}$ local coordinates pertaining to $u$, the level line curvature equals $u_{\bm{\xi\xi}}/(2\,\lvert\bm{\nabla}u\rvert)$.  The resulting area difference is \begin{align} \varDelta_2 &:= \lvert\mathcal{A}_+\rvert-\lvert\mathcal{A}_-\rvert = -2\int\limits_{-z_-(\varphi+\alpha+\pi/2)}^{z_+(\varphi+\alpha+\pi/2)} -\frac{u_{\bm{\xi\xi}}}{2\,\lvert\bm{\nabla}u\rvert}z^2\dd z +\LandauO(\varrho^4) \notag\\* &= \frac23\,\frac{u_{\bm{\xi\xi}}}{\lvert\bm{\nabla}u\rvert}\, \frac{\varrho^3} {\nu(\lvert\bm{\nabla}f\rvert\sin\alpha)^3} + \LandauO(\varrho^4) \;.  \label{Del2} \end{align}

\subsection{Median calculation}\label{ssec-median} We return now to the general situation in which both effects discussed in Subsection~\ref{ssec-imbalance} occur, making the area difference between $\mathcal{A}_+$ and $\mathcal{A}_-$ equal $\varDelta_1+\varDelta_2$ up to higher order terms.

As the median $\mu$ of $u$ within $\mathcal{A}$ belongs to the level line of $u$ that bisects the area of the amoeba, the difference $\mu-u(\bm{x}_0)$ corresponds to a shift of the central level line that compensates the area difference $\varDelta_1+\varDelta_2$. This compensation is obtained when \begin{align} & 2\,\frac{\mu-u(\bm{x}_0)}{\lvert\bm{\nabla}u\rvert} \bigl(z_+(\varphi+\alpha+\pi/2)+z_-(\varphi+\alpha+\pi/2)\bigr) = \varDelta_1+\varDelta_2 + \LandauO(\varrho^4)\;.  \label{compens} \end{align} Inserting \eqref{xipm-v} and \eqref{fv} on the left-hand side, and \eqref{Del1J1J2J3} and \eqref{Del2} on the right-hand side takes \eqref{compens} into \begin{align} & 2\,\frac{\mu-u(\bm{x}_0)}{\lvert\bm{\nabla}u\rvert} \,\frac{2\,\varrho}{\nu(\lvert\bm{\nabla}f(\bm{x}_0)\rvert\sin\alpha)} \notag\\*&\quad{} = \varrho^3 \biggl( -J_1\,f_{\bm{\zeta\zeta}}-2\,J_2\,f_{\bm{\zeta\chi}}-J_3\,f_{\bm{\chi\chi}} + \frac{2\,u_{\bm{\xi\xi}}}{3\,\lvert\bm{\nabla}u\rvert\, \nu(\lvert\bm{\nabla}f\rvert\sin\alpha)^3} \biggr) + \LandauO(\varrho^4)\;.  \label{compens2} \end{align} Solving for $\mu$ gives \begin{align} \mu &= u(\bm{x}_0)+ \frac{\varrho^2}{6} \biggl( \frac{u_{\bm{\xi\xi}}} {\nu(\lvert\bm{\nabla}f\rvert\,\sin\alpha)^2} -\frac32\,\nu(\lvert\bm{\nabla}f\rvert\,\sin\alpha)\, \lvert\bm{\nabla}u\rvert \times{}\notag\\*&\quad{}\times\Bigl(J_1(\lvert\bm{\nabla}f(\bm{x})\rvert,\alpha)\,f_{\bm{\zeta\zeta}} +2\,J_2(\lvert\bm{\nabla}f(\bm{x})\rvert,\alpha)\,f_{\bm{\zeta\chi}} +J_3(\lvert\bm{\nabla}f(\bm{x})\rvert,\alpha)\,f_{\bm{\chi\chi}}\Bigr) \biggr)  + \LandauO(\varrho) \;, \end{align} which is the claimed explicit time step for \eqref{genaacmed}.  This concludes the proof of Theorem~\ref{thm}.

\section{Special Cases}\label{sec-spec} In the following we relate Theorem~\ref{thm} to earlier results referring to more specialised configurations.

\subsection{$L^2$ and $L^1$ Amoeba Metrics} For general $\nu$, the integrals $J_1$, $J_2$, and $J_3$ in Theorem~\ref{thm} can often only be treated numerically.  For specific amoeba norms, however, the integrals can be evaluated in closed form. The following corollary states the PDE for amoeba active contours in the case of the $L^2$ amoeba metric.

\begin{corollary}\label{cor-aacl2} Amoeba median filtering of a smooth function $u$ governed by amoebas generated from $f$ with amoeba radius $\varrho$ and $L^2$ amoeba norm asymptotically approximates the PDE \begin{align} u_t &= \frac{u_{\bm{\xi\xi}}} {1+\lvert\bm{\nabla}f\rvert^2\,\sin^2\alpha} \notag\\*&\quad{} -\frac{\lvert\bm{\nabla}f\rvert\,\lvert\bm{\nabla}u\rvert} {1+\lvert\bm{\nabla}f\rvert^2\,\sin^2\alpha}\cdot \Biggl( \frac{ f_{\bm{\zeta\zeta}}\, \cos^3\alpha} {1+\lvert\bm{\nabla}f\rvert^2} +2\, f_{\bm{\zeta\chi}}\, \sin^3\alpha \notag\\*&\qquad{} +\frac{ f_{\bm{\chi\chi}}\, \cos\alpha\,\bigl( 2+\sin^2\alpha+3\,\lvert\bm{\nabla}f\rvert^2\, \sin^2\alpha\bigr)} {\bigl(1+\lvert\bm{\nabla}f\rvert^2\bigr)^2} \Biggr) \label{l2aacmed} \end{align} in the sense of Theorem~\ref{thm}.  \end{corollary}

\begin{remark} This corollary reproduces the statement of Theorem~1 in~\cite{Welk-ssvm13}.  \end{remark}

\begin{proof} For the $L^2$ amoeba norm, one has $\nu(s)=\sqrt{1+s^2}$ and thus $\nu'(s)=s/\sqrt{1+s^2}$.  In this case, the integrals $J_1$, $J_2$, $J_3$ from Theorem~\ref{thm} reduce to \begin{align} J_1(s,\alpha)&= \int\limits_{\alpha-\pi/2}^{\alpha+\pi/2} \frac{s\,\sin^2\vartheta\,\cos\vartheta} {(1+s^2\cos^2\vartheta)^{5/2}}\dd\vartheta \notag\\*&= s\,\left[\frac{\sin^3\vartheta}{3(1+s^2)(1+s^2\cos^2\vartheta)^{3/2}} \right]^{\vartheta=\alpha+\pi/2}_{\vartheta=\alpha-\pi/2} \notag\\*&= \frac23\,s\, \frac{\cos^3\alpha} {\bigl(1+s^2\bigr) \bigl(1+s^2\sin^2\alpha\bigr)^{3/2}}\;, \\ J_2(s,\alpha)&= \int\limits_{\alpha-\pi/2}^{\alpha+\pi/2} \frac{s\,\sin\vartheta\,\cos^2\vartheta} {(1+s^2\cos^2\vartheta)^{5/2}}\dd\vartheta \notag\\*&= s\,\left[\frac{-\cos^3\vartheta}{3(1+s^2\cos^2\vartheta)^{3/2}} \right]^{\vartheta=\alpha+\pi/2}_{\vartheta=\alpha-\pi/2} \notag\\*&= \frac23\,s\, \frac{\sin^3\alpha} {\bigl(1+s^2\sin^2\alpha\bigr)^{3/2}}\;, \\ J_3(s,\alpha)&= \int\limits_{\alpha-\pi/2}^{\alpha+\pi/2} \frac{s\,\cos^3\vartheta} {(1+s^2\cos^2\vartheta)^{5/2}}\dd\vartheta \notag\\*&= s\,\left[\frac{\sin\vartheta(2+\cos^2\vartheta+3\,s^2\cos^2\vartheta)} {3(1+s^2)^2(1+s^2\cos^2\vartheta)^{3/2}} \right]^{\vartheta=\alpha+\pi/2}_{\vartheta=\alpha-\pi/2} \notag\\*&= \frac23\,s\, \frac{\cos\alpha\,\bigl( 2+\sin^2\alpha+ 3\,s^2\sin^2\alpha \bigr)} {\bigl(1+s^2\bigr)^2 \bigl(1+s^2\,\sin^2\alpha\bigr)^{3/2}} \;.  \end{align} Inserting these into \eqref{Del1J1J2J3} and \eqref{compens} yields the claim.  \end{proof}

\begin{figure*}[t!] \centerline{\begin{tabular}{@{}cccc@{}} \includegraphics[width=0.24\textwidth]{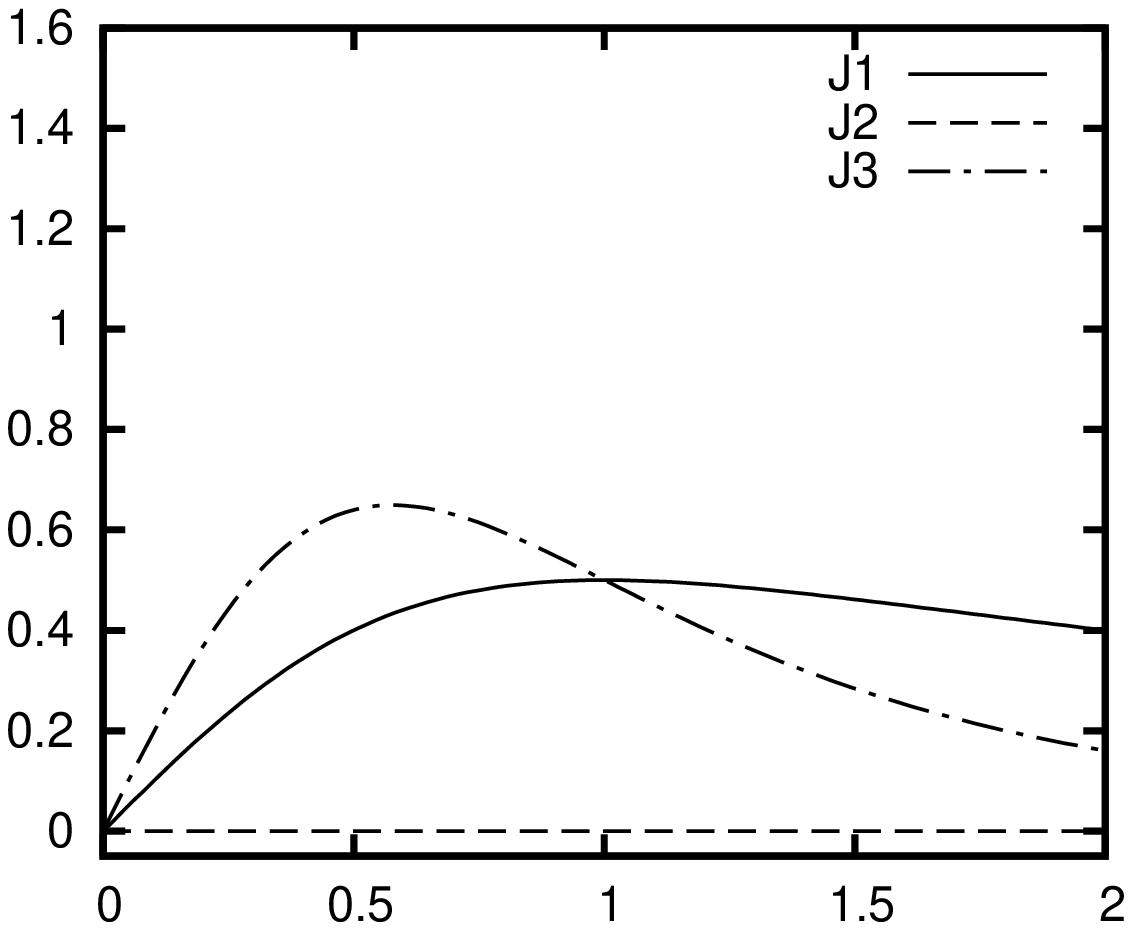}& \includegraphics[width=0.24\textwidth]{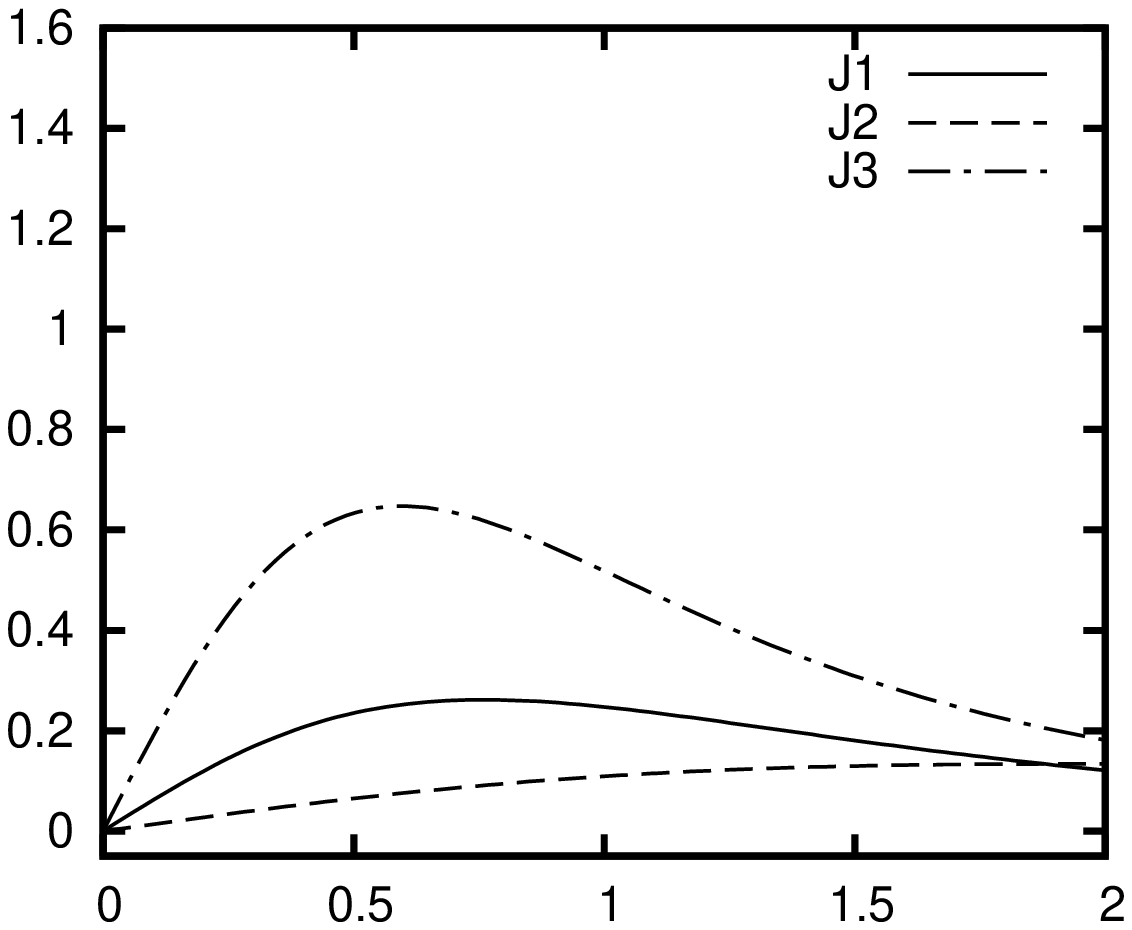}& \includegraphics[width=0.24\textwidth]{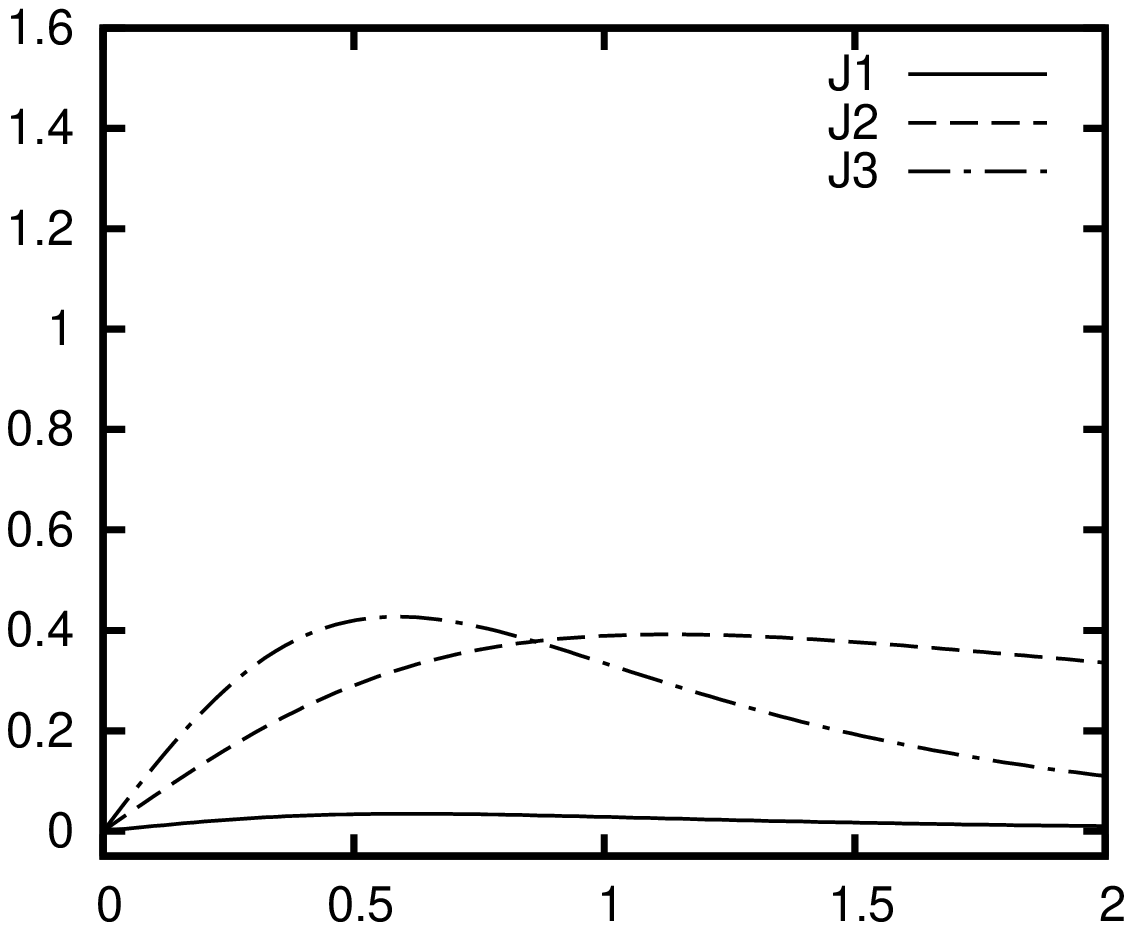}& \includegraphics[width=0.24\textwidth]{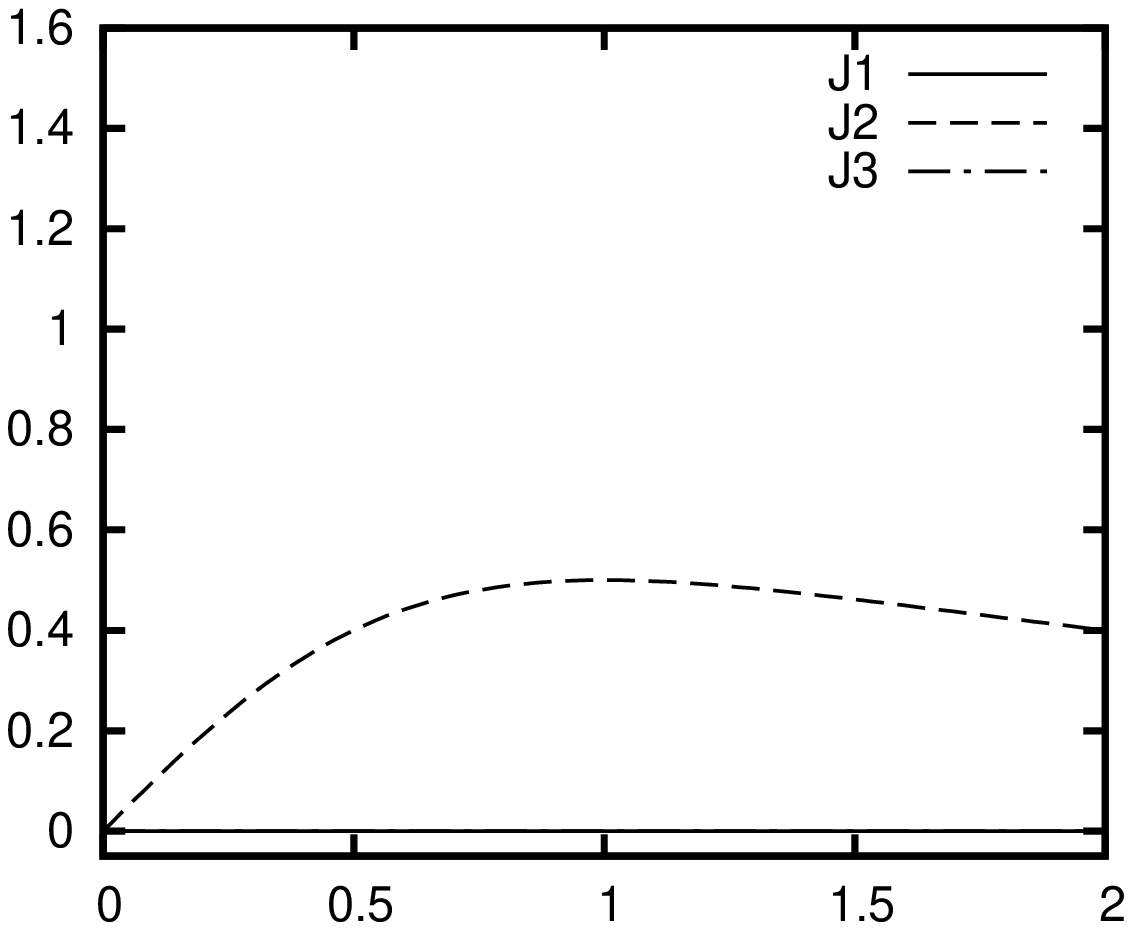}\\ \parbox[t]{0.24\textwidth}{\centering\small\textbf{(a)} $\alpha=0$, $s\in[0,2]$}& \parbox[t]{0.24\textwidth}{\centering\small\textbf{(b)} $\alpha=\pi/6$, $s\in[0,2]$}& \parbox[t]{0.24\textwidth}{\centering\small\textbf{(c)} $\alpha=\pi/3$, $s\in[0,2]$}& \parbox[t]{0.24\textwidth}{\centering\small\textbf{(d)} $\alpha=\pi/2$, $s\in[0,2]$} \end{tabular}} \caption{\label{fi-j123-l2-s} Coefficients $\frac32(1+s^2\sin^2\alpha)^{1/2}J_k$, $k=1,2,3$ for $f_{\bm{\zeta\zeta}}$, $f_{\bm{\zeta\chi}}$, $f_{\bm{\chi\chi}}$, respectively, in \eqref{genaacmed} with $L^2$ amoeba metric for four fixed angles $\alpha$ and $s\in[0,2]$.} \end{figure*}

\begin{figure*}[t!] \centerline{\begin{tabular}{@{}cccc@{}} \includegraphics[width=0.24\textwidth]{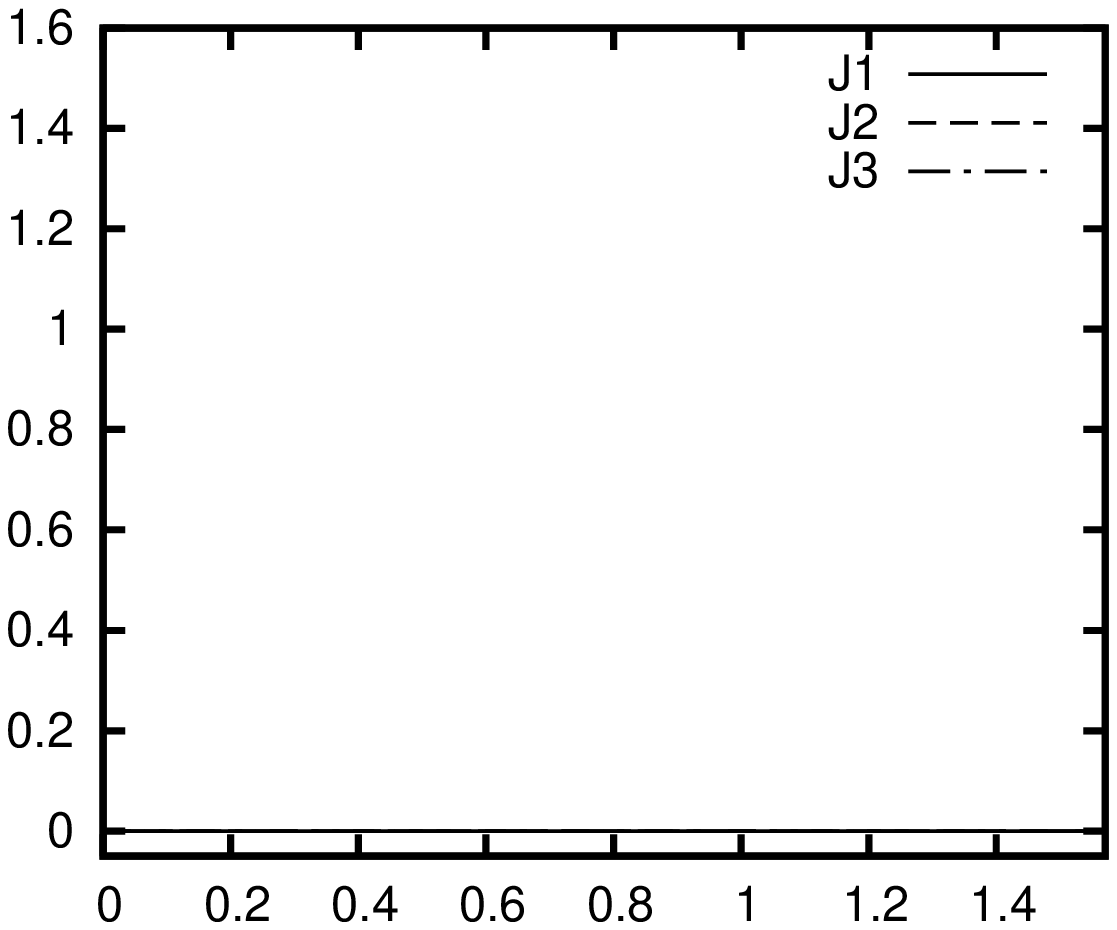}& \includegraphics[width=0.24\textwidth]{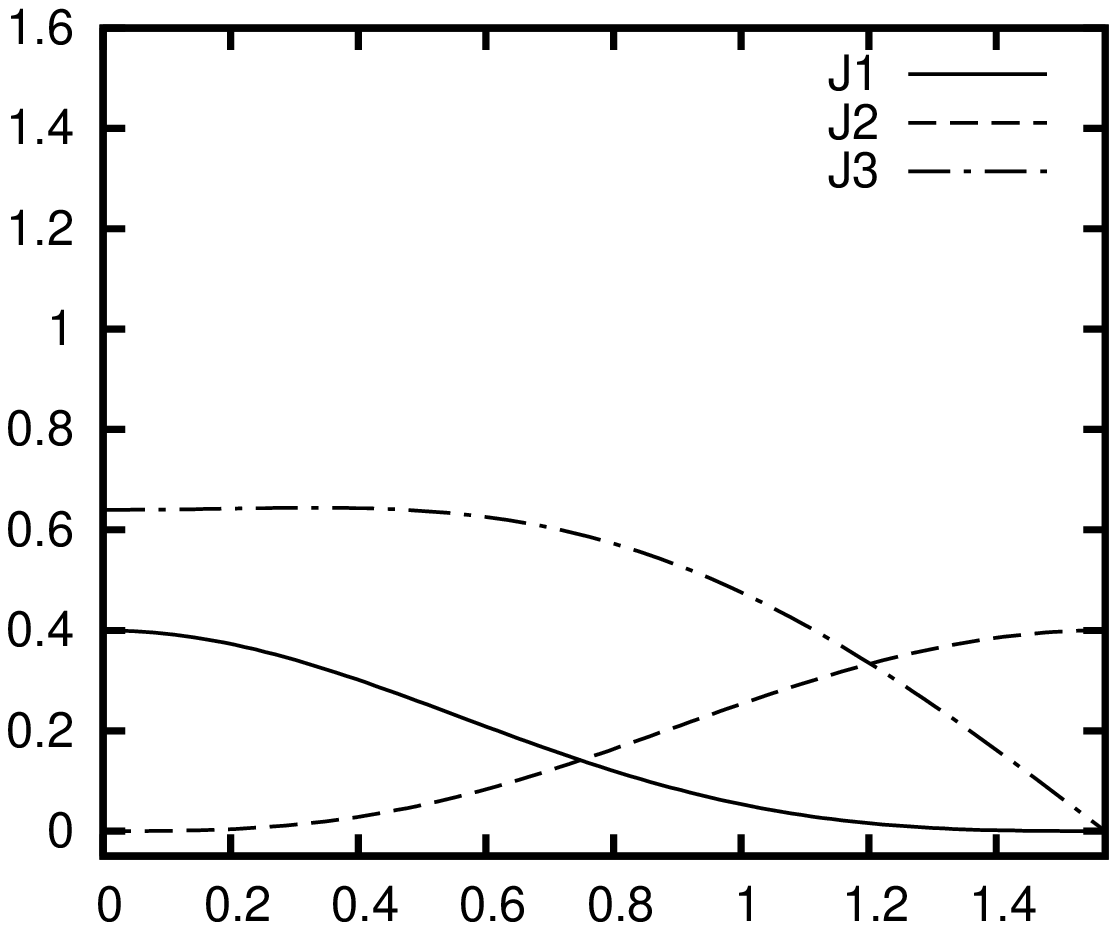}& \includegraphics[width=0.24\textwidth]{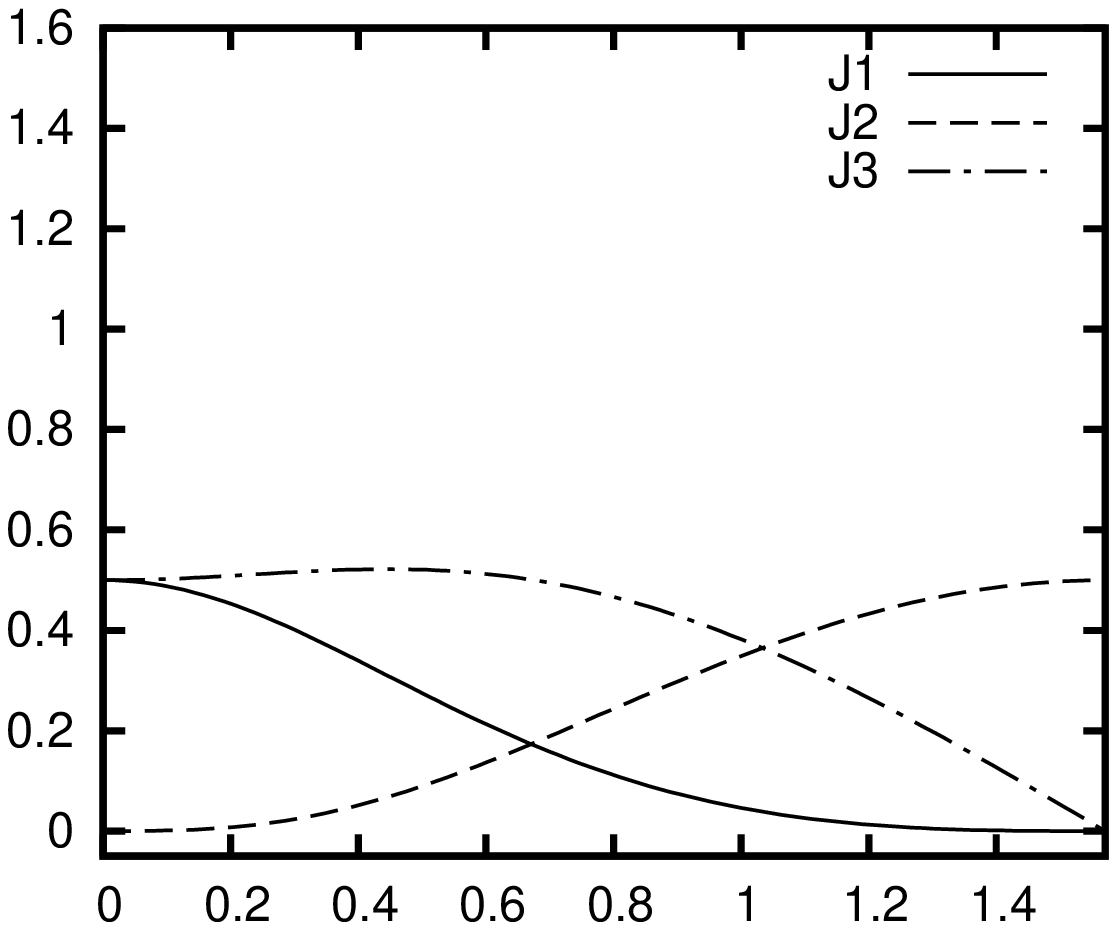}& \includegraphics[width=0.24\textwidth]{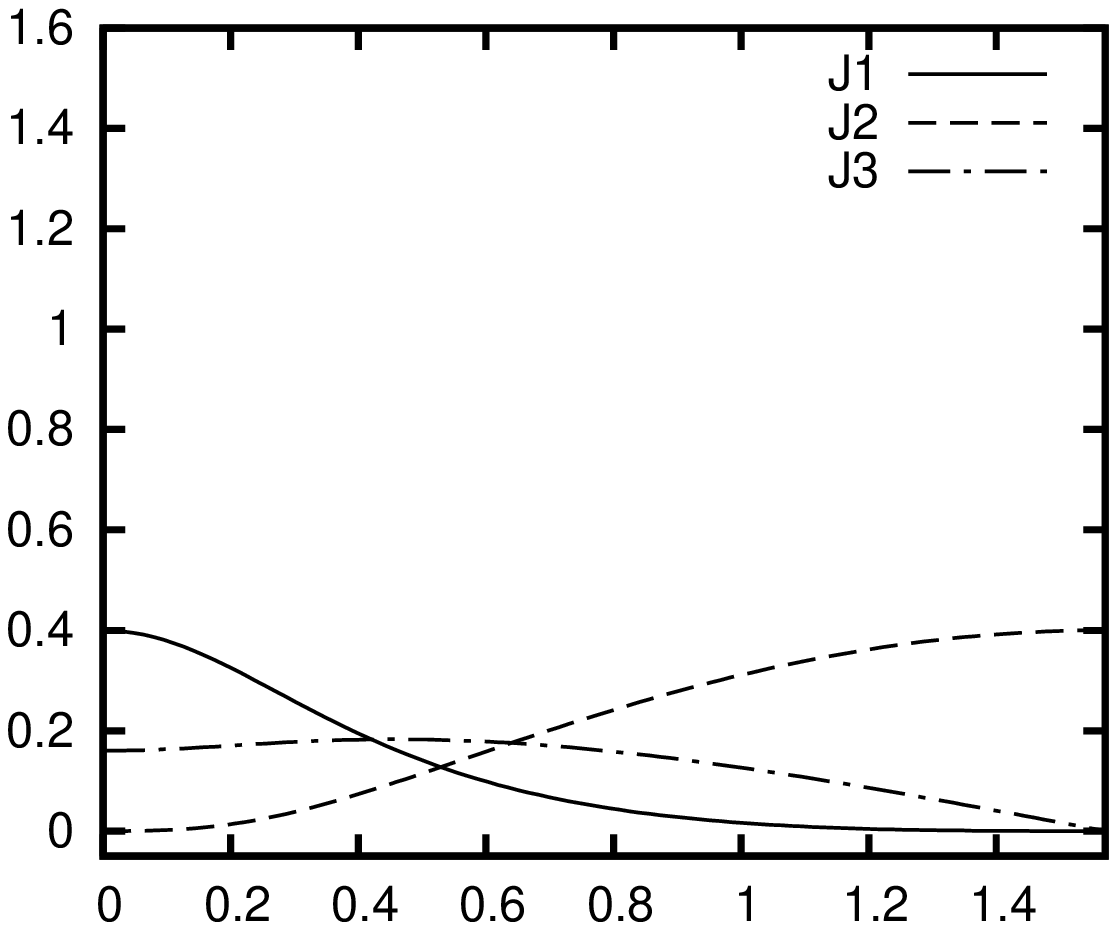}\\ \parbox[t]{0.24\textwidth}{\centering\small\textbf{(a)} $\alpha\in[0,\pi/2]$, $s=0$}& \parbox[t]{0.24\textwidth}{\centering\small\textbf{(b)} $\alpha\in[0,\pi/2]$, $s=1/2$}& \parbox[t]{0.24\textwidth}{\centering\small\textbf{(c)} $\alpha\in[0,\pi/2]$, $s=1$}& \parbox[t]{0.24\textwidth}{\centering\small\textbf{(d)} $\alpha\in[0,\pi/2]$, $s=2$} \end{tabular}} \caption{\label{fi-j123-l2-alpha} Coefficients $\frac32(1+s^2\sin^2\alpha)^{1/2}J_k$, $k=1,2,3$ for $f_{\bm{\zeta\zeta}}$, $f_{\bm{\zeta\chi}}$, $f_{\bm{\chi\chi}}$, respectively, in \eqref{genaacmed} with $L^2$ amoeba metric for four fixed gradient magnitudes $s=\lvert\bm{\nabla}f\rvert$ and $\alpha\in[0,\pi/2]$.} \end{figure*}

In Figs.~\ref{fi-j123-l2-s} and \ref{fi-j123-l2-alpha} the behaviour of the coefficients for $f_{\bm{\zeta\zeta}}$, $f_{\bm{\zeta\chi}}$ and $f_{\bm{\chi\chi}}$ for variable $s$ and $\alpha$ is shown.

A similar result for the $L^1$ norm follows.

\begin{corollary}\label{cor-aacl1} Amoeba median filtering of a smooth function $u$ governed by amoebas generated from $f$ with amoeba radius $\varrho$ and $L^1$ amoeba norm asymptotically approximates in the sense of Theorem~\ref{thm} the PDE \eqref{genaacmed} where $J_2$ is given by \begin{align} J_2(s,\alpha) &= \frac{\sin^2\alpha\,(3+\lvert\sin\alpha\rvert)} {3\,(1+s\,\lvert\sin\alpha\rvert)^3}\;, \label{J2L1} \end{align} while $J_1$ and $J_3$ are given for $s>1$ by \begin{align} J_1(s,\alpha) &= \frac1{(s^2-1)^{5/2}} \ln\frac {\sqrt{s+1}\,(1+\lvert\sin\alpha\rvert)+\sqrt{s-1}\,\cos\alpha} {\sqrt{s+1}\,(1+\lvert\sin\alpha\rvert)-\sqrt{s-1}\,\cos\alpha} \notag\\*&\quad{} -\frac{ \left\{ \begin{array}{@{}r@{}} \cos\alpha\bigl((2\,s^3+s)\sin^2\alpha+3\,(s^2+1)\,\lvert\sin\alpha\rvert\\ -2\,s^3+5\,s\bigr) \end{array} \right.  } {3(s^2-1)^2(1+s\,\lvert\sin\alpha\rvert)^2}\;, \label{J1L1sgt1}\\ J_3(s,\alpha) &= \frac{-4\,s^2-1}{2\,(s^2-1)^{7/2}} \ln\frac {\sqrt{s+1}\,(1+\lvert\sin\alpha\rvert)+\sqrt{s-1}\,\cos\alpha} {\sqrt{s+1}\,(1+\lvert\sin\alpha\rvert)-\sqrt{s-1}\,\cos\alpha} \notag\\*&\quad{} +\frac{ \left\{ \begin{array}{@{}l@{}} \cos\alpha\bigl(s\,(2\,s^2+13)\\ \quad{}+3\,(2\,s^4+9\,s^2-1)\,\lvert\sin\alpha\rvert\\ \quad{}+s\,(6\,s^4+10\,s^2-1)\sin^2\alpha\bigr) \end{array} \right.  } {3\,(s^2-1)^3(1+s\,\lvert\sin\alpha\rvert)^3}\;, \label{J3L1sgt1} \end{align} for $s=1$ by \begin{align} J_1(1,\alpha) &= \frac{2\cos^3\alpha\,(4+\lvert\sin\alpha\rvert)} {15\,(1+\lvert\sin\alpha\rvert)^4}\;, \label{J1L1seq1}\\ J_3(1,\alpha) &= \frac{2\cos\alpha\,(8+32\,\lvert\sin\alpha\rvert+52\sin^2\alpha +13\,\lvert\sin^3\alpha\rvert)} {105\,(1+\lvert\sin\alpha\rvert)^4} \label{J3L1seq1} \end{align} and for $0\le s<1$ by \begin{align} J_1(s,\alpha) &= \frac2{(1-s^2)^{5/2}} \arctan\left(\sqrt{\frac{1-s}{1+s}}\, \frac{\cos\alpha}{1+\lvert\sin\alpha\rvert} \right) \notag\\*&\quad{} -\frac{ \left\{ \begin{array}{@{}r@{}} \cos\alpha\bigl((2\,s^3+s)\sin^2\alpha+3\,(s^2+1)\,\lvert\sin\alpha\rvert\\ -2\,s^3+5\,s\bigr) \end{array} \right.  } {3\,(s^2-1)^2(1+s\,\lvert\sin\alpha\rvert)^2}\;, \label{J1L1slt1}\\ J_3(s,\alpha) &= \frac{-8\,s^2-2}{2\,(s^2-1)^{7/2}} \arctan\left(\sqrt{\frac{1-s}{1+s}}\,\frac{\cos\alpha}{1+\sin\alpha} \right) \notag\\*&\quad{} +\frac{ \left\{ \begin{array}{@{}l@{}} \cos\alpha\bigl(s\,(2\,s^2+13)\\ \quad{}+3\,(2\,s^4+9\,s^2-1)\,\lvert\sin\alpha\rvert\\ \quad{}+s\,(6\,s^4+10\,s^2-1)\sin^2\alpha\bigr) \end{array} \right.  } {3\,(s^2-1)^3(1+s\,\lvert\sin\alpha\rvert)^3}\;.  \label{J3L1slt1} \end{align} \end{corollary}

\begin{proof} See Appendix~\ref{app-proofcor-aacl1}.  \end{proof}

\begin{figure*}[t!] \centerline{\begin{tabular}{@{}cccc@{}} \includegraphics[width=0.24\textwidth]{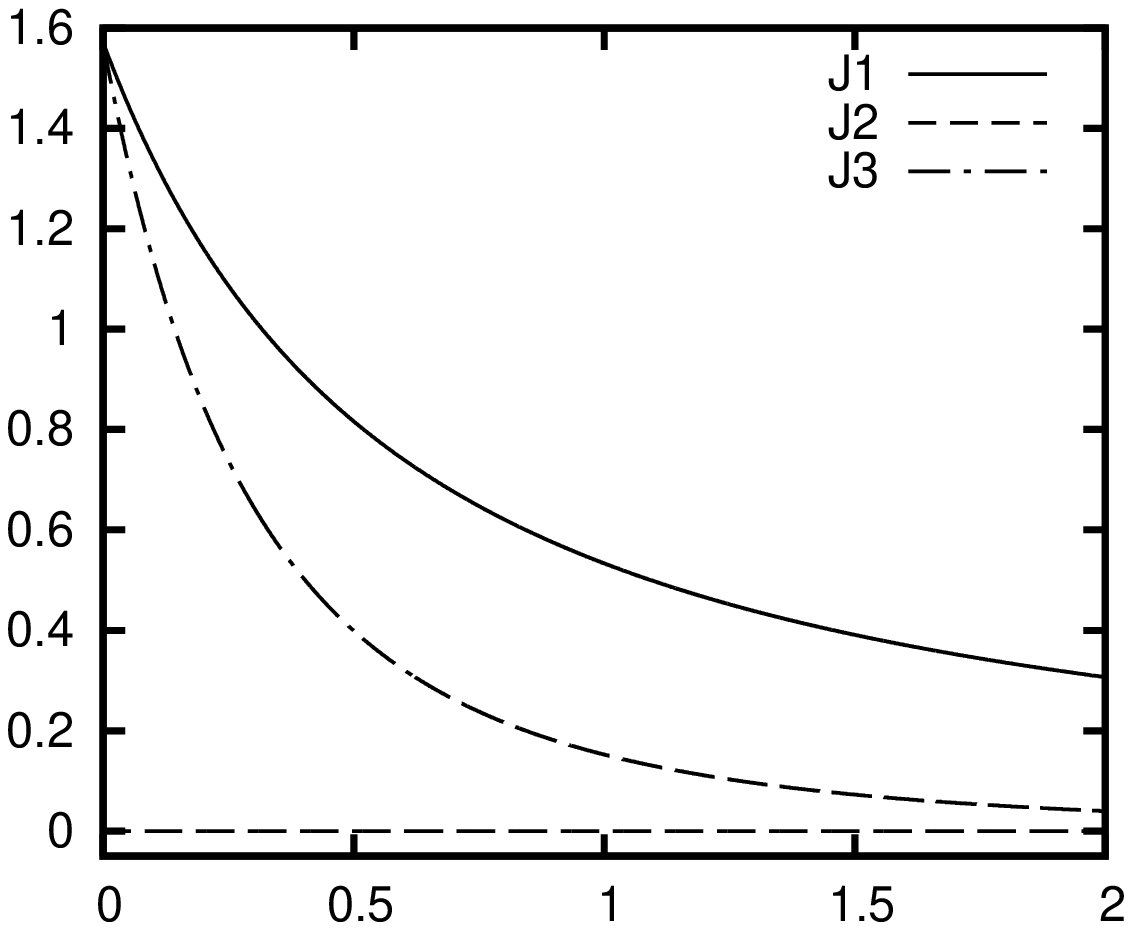}& \includegraphics[width=0.24\textwidth]{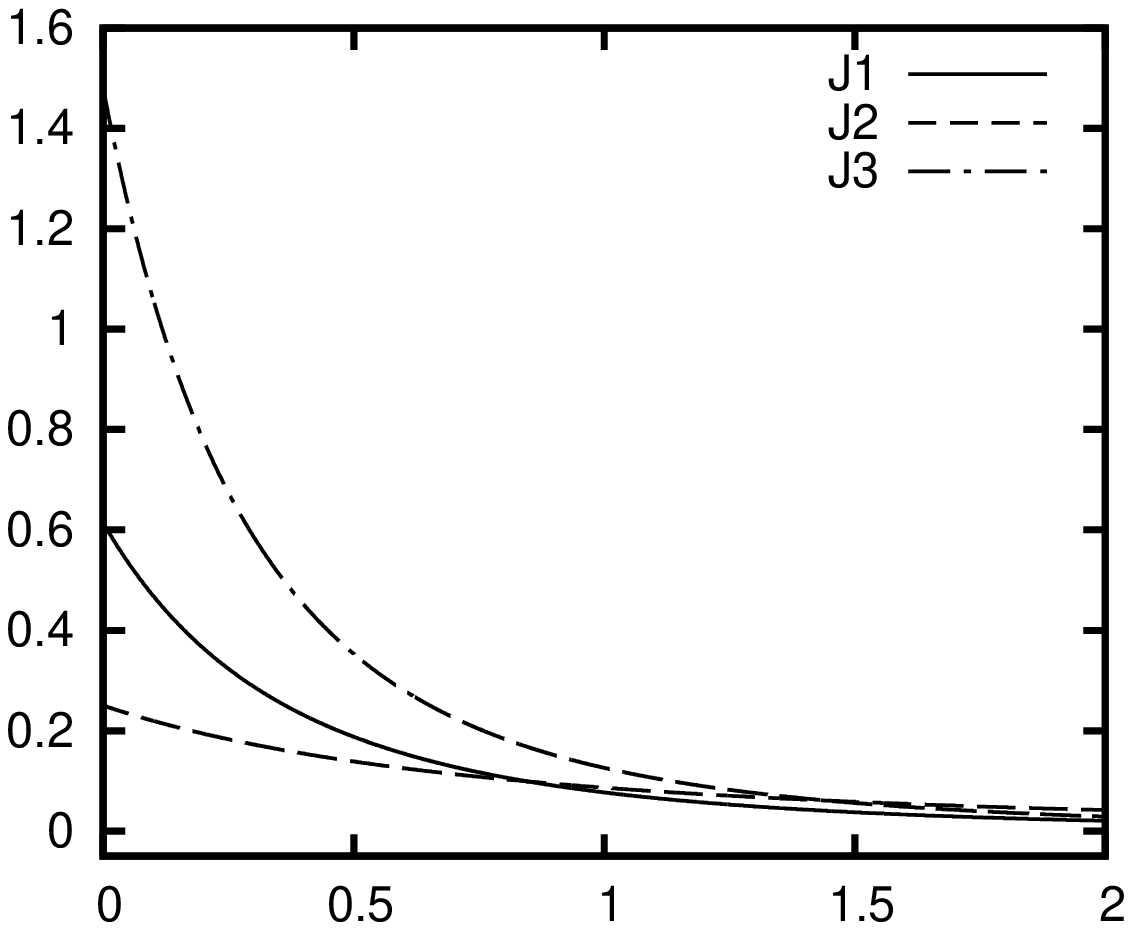}& \includegraphics[width=0.24\textwidth]{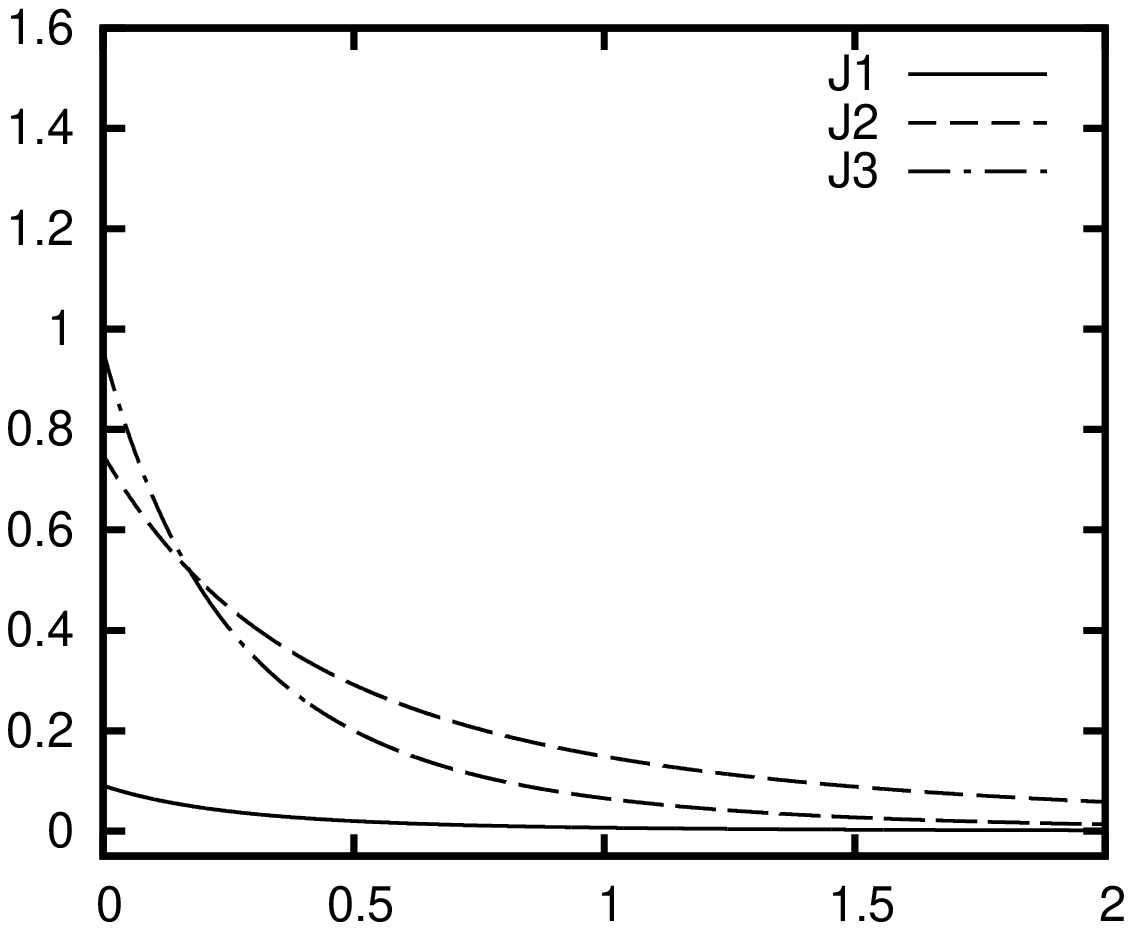}& \includegraphics[width=0.24\textwidth]{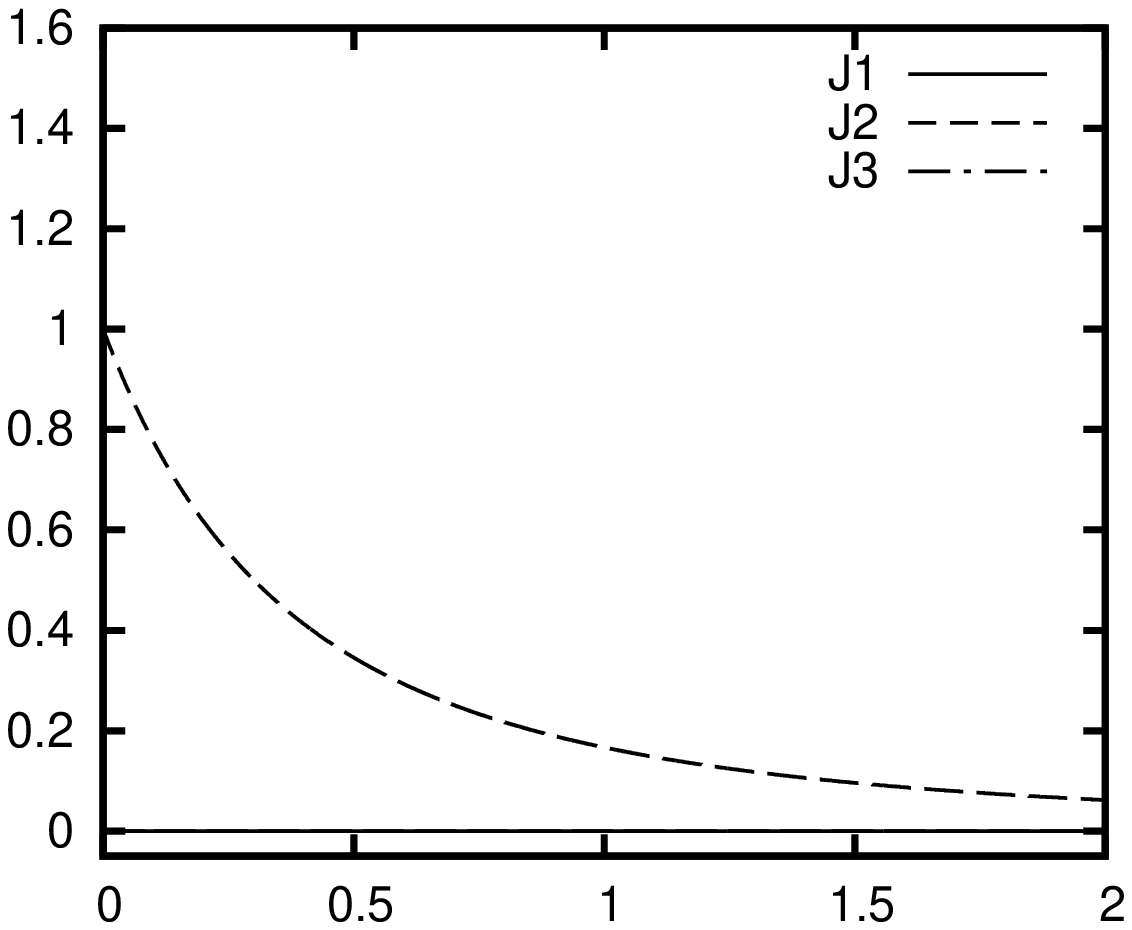}\\ \parbox[t]{0.24\textwidth}{\centering\small\textbf{(a)} $\alpha=0$, $s\in[0,2]$}& \parbox[t]{0.24\textwidth}{\centering\small\textbf{(b)} $\alpha=\pi/6$, $s\in[0,2]$}& \parbox[t]{0.24\textwidth}{\centering\small\textbf{(c)} $\alpha=\pi/3$, $s\in[0,2]$}& \parbox[t]{0.24\textwidth}{\centering\small\textbf{(d)} $\alpha=\pi/2$, $s\in[0,2]$} \end{tabular}} \caption{\label{fi-j123-l1-s} Coefficients $\frac32(1+s\sin\alpha)J_k$, $k=1,2,3$ for $f_{\bm{\zeta\zeta}}$, $f_{\bm{\zeta\chi}}$, $f_{\bm{\chi\chi}}$, respectively, in \eqref{genaacmed} with $L^1$ amoeba metric for four fixed angles $\alpha$ and $s\in[0,2]$.} \end{figure*}

\begin{figure*}[t!] \centerline{\begin{tabular}{@{}cccc@{}} \includegraphics[width=0.24\textwidth]{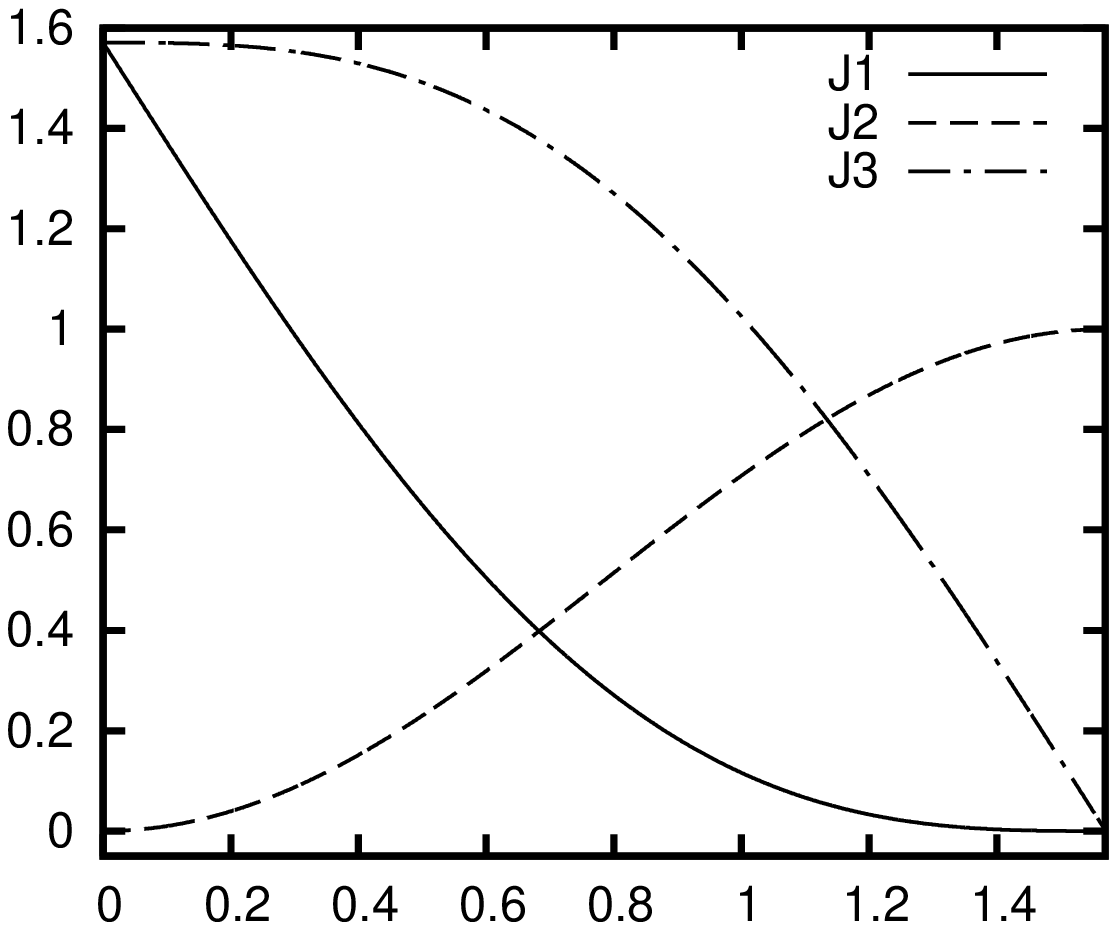}& \includegraphics[width=0.24\textwidth]{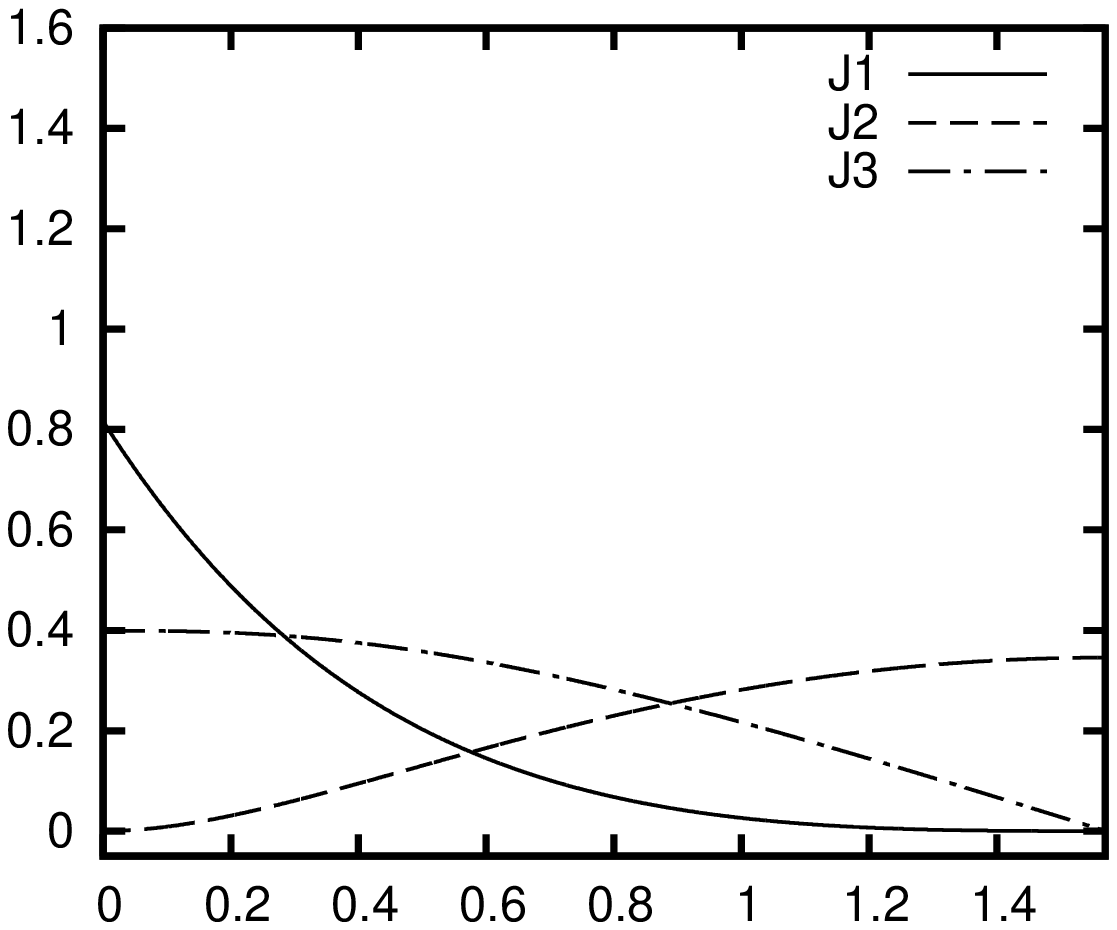}& \includegraphics[width=0.24\textwidth]{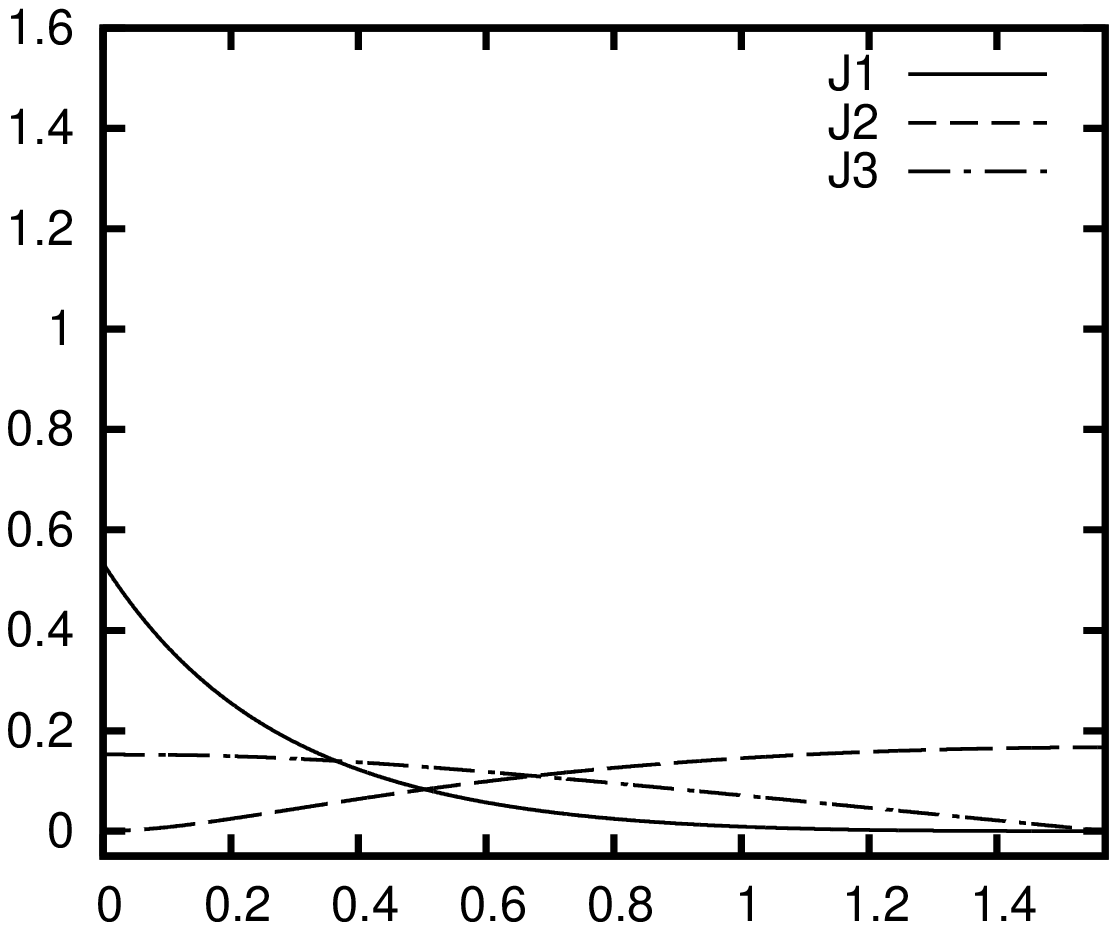}& \includegraphics[width=0.24\textwidth]{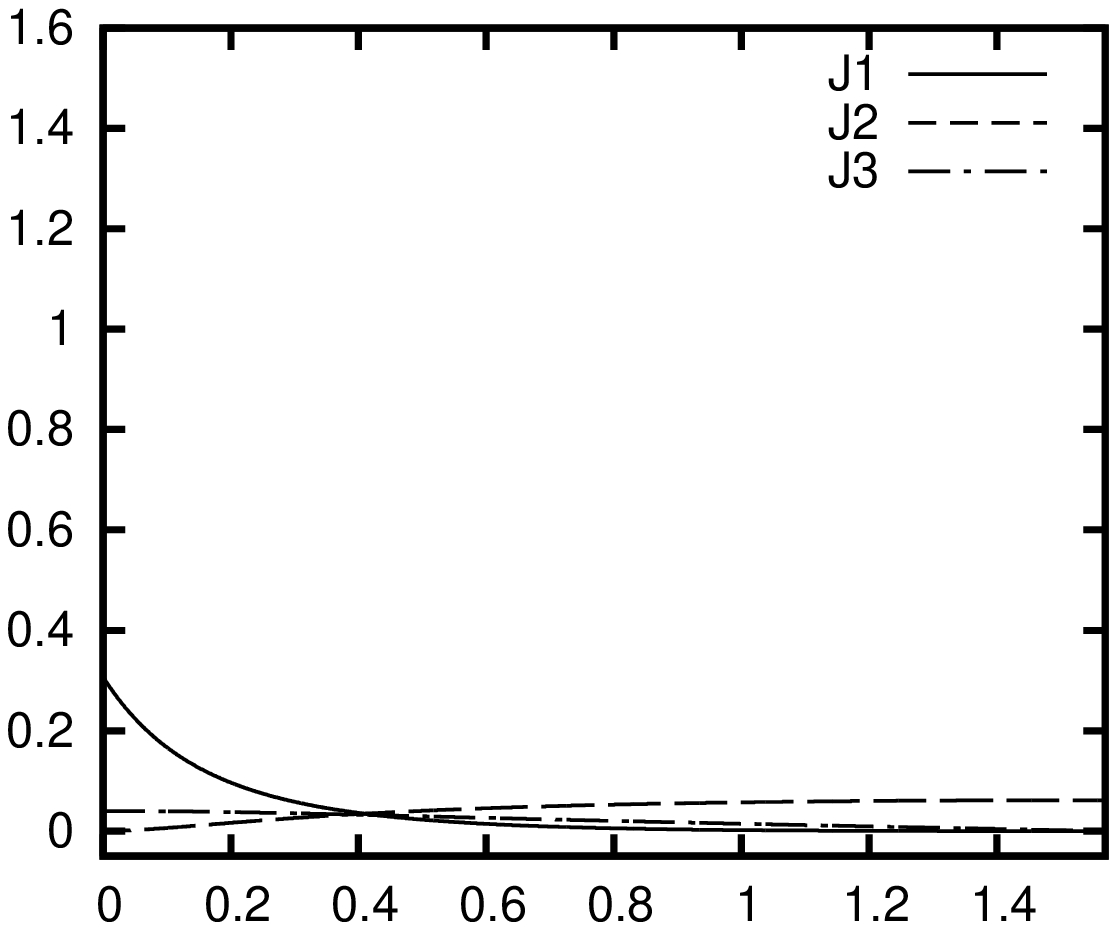}\\ \parbox[t]{0.24\textwidth}{\centering\small\textbf{(a)} $\alpha\in[0,\pi/2]$, $s=0$}& \parbox[t]{0.24\textwidth}{\centering\small\textbf{(b)} $\alpha\in[0,\pi/2]$, $s=1/2$}& \parbox[t]{0.24\textwidth}{\centering\small\textbf{(c)} $\alpha\in[0,\pi/2]$, $s=1$}& \parbox[t]{0.24\textwidth}{\centering\small\textbf{(d)} $\alpha\in[0,\pi/2]$, $s=2$} \end{tabular}} \caption{\label{fi-j123-l1-alpha} Coefficients $\frac32(1+s\sin\alpha)J_k$, $k=1,2,3$ for $f_{\bm{\zeta\zeta}}$, $f_{\bm{\zeta\chi}}$, $f_{\bm{\chi\chi}}$, respectively, in \eqref{genaacmed} with $L^1$ amoeba metric for four fixed gradient magnitudes $s=\lvert\bm{\nabla}f\rvert$ and $\alpha\in[0,\pi/2]$.} \end{figure*}

To illustrate the behaviour of the coefficients of $f_{\bm{\zeta\zeta}}$, $f_{\bm{\zeta\chi}}$ and $f_{\bm{\chi\chi}}$ in \eqref{genaacmed} for the $L^1$ amoeba metric, we display in Figs.~\ref{fi-j123-l1-s} and \ref{fi-j123-l1-alpha} graphs of all three coefficients for variable $s$ and variable $\alpha$, respectively. The most striking difference to the $L^2$ amoeba metric is that the coefficients start out with nonzero values already for $\lvert\bm{\nabla}f\rvert=0$ such that the level line curvature of $f$ influences the evolution also in almost flat regions. Combined with the fact that also the coefficient $1/\nu(\lvert\bm{\nabla}f\rvert\,\sin\alpha)^2$ of $u_{\bm{\xi\xi}}$ shows a faster decay away from $\lvert\bm{\nabla}f\rvert=0$ than in the $L^2$ case it becomes evident that the $L^1$ contour evolution reacts generally more sensitive to small image gradients.

\subsection{Special Evolutions} The following two statements reproduce the more specialised approximation results from \cite{Welk-JMIV11} (in the case of the $L^2$ amoeba metric) and \cite{Welk-ssvm11}, respectively.

\begin{corollary}\label{cor1} The amoeba median filter with $f\equiv \beta\,u$ approximates the self-snakes equation \begin{align} u_t &= \lvert\bm{\nabla}u\rvert~\mathrm{div}\left( g(\lvert\bm{\nabla}u\rvert)\, \frac{\bm{\nabla}u}{\lvert\bm{\nabla}u\rvert} \right) \label{ssn} \end{align} with \begin{align} g(s) &:= 1-\frac32\,\beta\,s\,\tilde{J}_1(\beta\,s)\;, \label{g} \\ \tilde{J}_1(s)&:= \int\limits_{-\pi/2}^{+\pi/2} \frac{\nu'(s\cos\vartheta)} {\nu(s\cos\vartheta)^4}\, \sin^2\vartheta \dd\vartheta \label{J1a0} \end{align} in the sense of Theorem~\ref{thm}.  \end{corollary}

\begin{proof} First one observes that the hypothesis of the corollary entails that the identities $\alpha=0$, $\bm{\zeta}=\bm{\xi}$, and $\bm{\chi}=\bm{\eta}$ hold everywhere.  Inserting these into \eqref{J1}--\eqref{J3}, all integrals run from $-\pi/2$ to $\pi/2$, and $J_2$ vanishes by the odd symmetry of its integrand.  The expressions $J_1$ and $J_3$ become \begin{align} J_1(s,0) &= \tilde{J}_1(s)\;, \\ J_3(s,0) &= \tilde{J}_3(s) := \int\limits_{-\pi/2}^{+\pi/2} \frac{\nu'(s\cos\vartheta)} {\nu(s\cos\vartheta)^4}\, \cos^2\vartheta \dd\vartheta\;.  \label{J3a0} \end{align} Substituting these together with $\nu(\lvert\bm{\nabla}f\rvert\,\sin\alpha)=\nu(0)=1$, $f\equiv\beta\,u$, $f_{\bm{\zeta\zeta}}\equiv\beta\,u_{\bm{\xi\xi}}$ and $f_{\bm{\chi\chi}}\equiv\beta\,u_{\bm{\eta\eta}}$ into \eqref{genaacmed} yields \begin{align} u_t &= g(\lvert\bm{\nabla}u\rvert)\,u_{\bm{\xi\xi}} + h(\lvert\bm{\nabla}u\rvert)\,u_{\bm{\eta\eta}} \label{ssn-gh} \end{align} with $g(s)$ as stated in \eqref{g}, and \begin{align} h(s) &= -\frac32\,\beta\,s\,\tilde{J}_3(\beta\,s)\;.  \end{align} A short calculation (see Appendix~\ref{app-j1j3}) verifies that \begin{equation} h(s) = s\,g'(s) \label{heqsdg} \end{equation} such that \eqref{ssn-gh} can be rewritten into \eqref{ssn}.  This completes the proof.  \end{proof}

\begin{remark} Corollary~\ref{cor1} reproduces the result from \cite{Welk-JMIV11} on the approximation of self-snakes by iterated amoeba median filtering with a general amoeba metric.  This is not quite obvious since due to the different proof strategy used in \cite{Welk-JMIV11} the actual integral expressions look fairly different. Appendix~\ref{app-ssn-jmiv11}, however, demonstrates that the coefficients are in fact identical.

Note also that in the case of the $L^2$ amoeba metric, $g$ coincides with the Perona-Malik function \eqref{peronamalikg} with $\lambda=1/\beta$.  \end{remark}

We return now to the active contour setting where the roles of the evolving function $u$ and the image $f$ governing the amoebas are separated, and consider the special geometric situation of both functions being rotationally symmetric with the same centre. In this case, the PDE approximated by amoeba active contours is identical to the geodesic active contour equation.

\begin{corollary}\label{cor2} If input image $f$ and evolving level-set image $u$ are rotationally symmetric with respect to the origin, amoeba active contours approximate the geodesic active contour equation \begin{align} u_t &= \lvert\bm{\nabla}u\rvert~\mathrm{div}\left(g(\lvert\bm{\nabla}f\rvert) \,\frac{\bm{\nabla}u}{\lvert\bm{\nabla}u\rvert} \right) \label{gac1} \end{align} with \begin{align} g(s) &= 1-\frac32\,s\,\tilde{J}_1(s) \label{g-aacrotinv} \end{align} and $\tilde{J}_1$ as in \eqref{J1a0} in the sense of Theorem~\ref{thm}.  \end{corollary}

\begin{proof} The assumed rotational symmetry implies $\alpha=0$, $\bm{\zeta}=\bm{\xi}$, $\bm{\chi}=\bm{\eta}$, $u_{\bm{\xi\eta}}\equiv f_{\bm{\xi\eta}}\equiv0$, and $u_{\bm{\xi\xi}}/\lvert\bm{\nabla}u\rvert\equiv f_{\bm{\xi\xi}}/\lvert\bm{\nabla}f\rvert$.  Substituting these identities into \eqref{genaacmed} leads to \begin{align} u_t &= u_{\bm{\xi\xi}} - \frac32\,\lvert\bm{\nabla}u\rvert\,\tilde{J}_1( \lvert\bm{\nabla}f\rvert)f_{\bm{\xi\xi}} -\frac32\,\lvert\bm{\nabla}u\rvert\,\tilde{J}_3( \lvert\bm{\nabla}f\rvert)f_{\bm{\eta\eta}} \notag\\* &= g(\lvert\bm{\nabla}f\rvert)\,u_{\bm{\xi\xi}} +h(\lvert\bm{\nabla}f\rvert)\,f_{\bm{\eta\eta}} \label{ut-aacrotinv-gh} \end{align} with $g(s)$ as given by \eqref{g-aacrotinv} and \begin{align} h(s) &= -\frac32\,s\,\tilde{J}_3(s)\;.  \end{align} Since $g$ and $h$ differ from their counterparts in the previous corollary only by the constant factor $\beta$, relation \eqref{heqsdg} transfers verbatim, such that \eqref{ut-aacrotinv-gh} can be transformed into \eqref{gac1}.  \end{proof}

\begin{remark} The statement of Corollary~\ref{cor2} generalises the equivalence of amoeba active contours and geodesic active contours that was proven in \cite{Welk-ssvm11} for the $L^2$ amoeba metric to a general amoeba metric.

From a practical viewpoint, the hypothesis of Corollary~\ref{cor2} may appear rather artificial at first glance. However, it mimicks a situation which is common in an active contour evolution when the evolving contour has almost attained its final state delineating a segment boundary: If the segment boundary in $f$ is given by a level line, along which the image contrast is uniform, and the evolving contour in $u$ is already close to it, then the level lines of $u$ and $f$ will be aligned and of equal curvature. The same is achieved in the rotationally symmetric scenario considered in Corollary~\ref{cor2}.  \end{remark}

\section{Comparison of Amoeba Active Contours with Geodesic Active Contours}\label{sec-aacgac} In the general amoeba active contour setting, however, it is evident that equation \eqref{genaacmed} does not exactly coincide with \eqref{gac}. For a better understanding of the differences between both active contour methods, we analyse further typical configurations. Throughout this section, we restrict our considerations to the $L^2$ amoeba metric, making the PDE \eqref{l2aacmed} our starting point.

\begin{figure*}[t] \setlength{\unitlength}{0.001\textwidth} \begin{picture}(1000,250) \put(30,0){\includegraphics[height=0.25\textwidth]{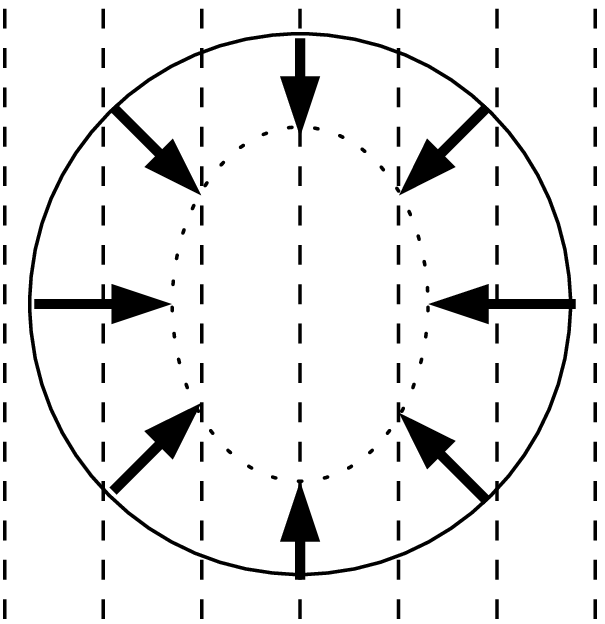}} \put(0,0){(a)} \put(340,0){\includegraphics[height=0.25\textwidth]{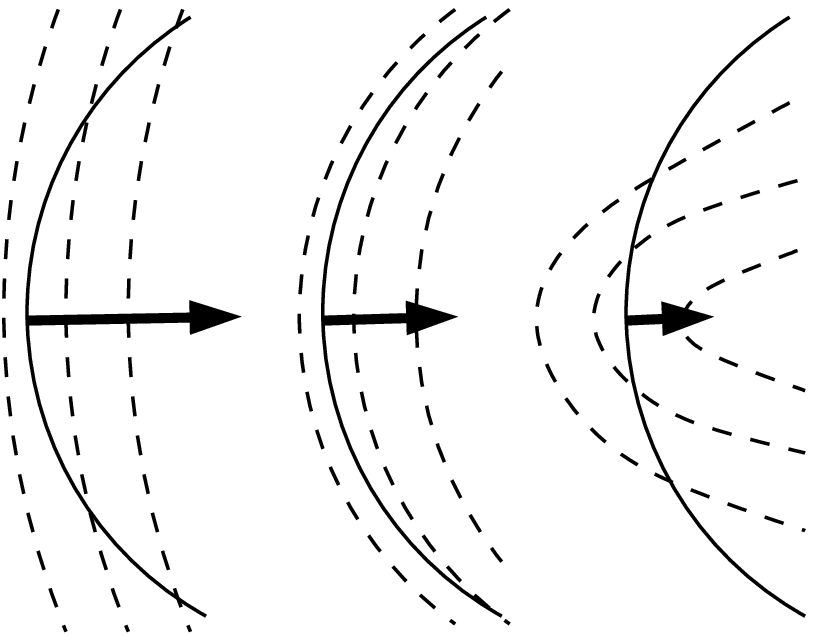}} \put(320,0){(b)} \put(730,0){\includegraphics[height=0.25\textwidth]{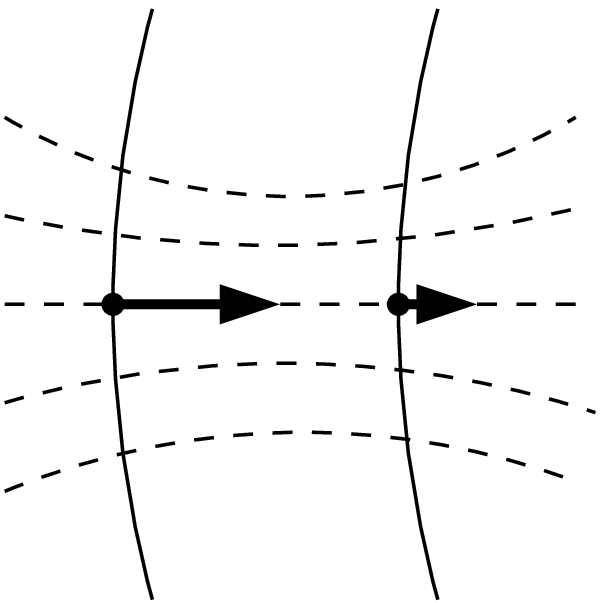}} \put(720,0){(c)} \end{picture} \caption{\label{fi-levellines}Evolution of level lines under the PDE \eqref{l2aacmed} in exemplary configurations (schematic).  Solid lines: level lines of $u$, dashed lines: level lines of $f$.  \textbf{Left to right:} \textbf{(a)} In a region with homogeneous $\bm{\nabla} f$, aligned level line segments of $u$ evolve faster. -- \textbf{(b)} At a location with aligned $\bm{\nabla}u$ and $\bm{\nabla}f$, the contour evolves inward faster when the curvature of $u$ exceeds that of $f$. -- \textbf{(c)} At locations with orthogonal $\bm{\nabla}u$ and $\bm{\nabla}f$, the curvature-dependent movement of the contour is attracted towards high-contrast regions of $f$. Assuming that $\bm{\eta}$ points to the right, $f_{\bm{\xi\eta}}<0$ holds in the left, and $f_{\bm{\xi\eta}}>0$ in the right part, while $u_{\bm{\xi\xi}}<0$ in both cases. -- From \cite{Welk-ssvm13}.  } \end{figure*}

\subsection{Homogeneous image gradients} In flat image regions ($\bm{\nabla}f\equiv\bm{0}$), geodesic active contours \eqref{gac} as well as amoeba active contours evolve the level set function $u$ by curvature motion. Let us consider now an image region with a homogeneous non-zero gradient, $\bm{\nabla}f=\mathrm{const}$.  In such a region, geodesic active contours still perform curvature motion, but at an evolution speed slowed down by the contrast-dependent factor $g(\lvert\bm{\nabla}f\rvert)=1/(1+\lvert\bm{\nabla}f\rvert^2)$.  The amoeba-based PDE \eqref{l2aacmed} in this case becomes \begin{equation} u_t = \frac{u_{\bm{\xi\xi}}}{1+\lvert\bm{\nabla}f\rvert^2\sin^2\alpha}\;.  \end{equation} This, too, describes a slowed-down curvature motion, but the evolution is slowed down the less, the more the level lines of $f$ and $u$ are aligned. This leads to a faster straightening of aligned contour segments, thereby boosting adaptation of $u$'s level lines to those of $f$, see the schematic representation in Figure~\ref{fi-levellines}(a).

\subsection{Aligned gradients}\label{ssec-aacgac-alig} Relaxing the condition of Corollary~\ref{cor2}, we assume now that the gradient directions of $f$ and $u$ coincide, $\alpha=0$, $\bm{\zeta}=\bm{\xi}$, $\bm{\chi}= \bm{\eta}$, but make no assumption on their curvatures. At such a location, \eqref{l2aacmed} takes the form \begin{align} u_t &= u_{\bm{\xi\xi}} - \lvert\bm{\nabla}f\rvert\,\lvert\bm{\nabla}u\rvert\, \left( \frac{f_{\bm{\xi\xi}}}{1+\lvert\bm{\nabla}f\rvert^2} +\frac{2\,f_{\bm{\eta\eta}}}{\bigl(1+\lvert\bm{\nabla}f\rvert^2\bigr)^2} \right) \notag\\* &= \frac{u_{\bm{\xi\xi}}}{1+\lvert\bm{\nabla}f\rvert^2} - \frac{2\,\lvert\bm{\nabla}f\rvert\,\lvert\bm{\nabla}u\rvert\, f_{\bm{\eta\eta}}}{\bigl(1+\lvert\bm{\nabla}f\rvert^2\bigr)^2}  +\frac{2\,\lvert\bm{\nabla}f\rvert^2\,\lvert\bm{\nabla}u\rvert} {1+\lvert\bm{\nabla}f\rvert^2} \left( \frac{u_{\bm{\xi\xi}}}{2\,u_{\bm{\eta}}} -\frac{f_{\bm{\xi\xi}}}{2\,f_{\bm{\eta}}} \right) \end{align} which coincides with the corresponding geodesic active contour evolution except for the last summand that speeds up the evolution if the level line curvature $u_{\bm{\xi\xi}}/(2\,u_{\bm{\eta}})$ of $u$ exceeds that of $f$, see Figure~\ref{fi-levellines}(b).  The same offset is obtained in the anti-aligned case, $\alpha=\pi$, $\bm{\zeta}=-\bm{\xi}$, $\bm{\chi}=-\bm{\eta}$; note that the curvature of $f$'s level lines is measured with respect to the orientation of $u$'s level lines.  Relative to geodesic active contours, this implies an accelerated removal of sharp contour corners that do not match the given image $f$.

\subsection{Orthogonal gradients}\label{ssec-aacgac-orth} Consider now the complementary situation where the gradient directions of $u$ and $f$ are orthogonal, i.e.\ $\alpha=\pi/2$, $\bm{\zeta}=\bm{\eta}$, $\bm{\chi}=-\bm{\xi}$. Then \eqref{l2aacmed} becomes \begin{align} u_t &= \frac{u_{\bm{\xi\xi}}}{1+\lvert\bm{\nabla}f\rvert^2} + \frac{2\,\lvert\bm{\nabla}f\rvert\,\lvert\bm{\nabla}u\rvert\, f_{\bm{\xi\eta}}}{1+\lvert\bm{\nabla}f\rvert^2} \end{align} where the last summand is by a factor $\bigl(1+\lvert\bm{\nabla}f\rvert^2\bigr)$ larger than in the corresponding geodesic active contour evolution.  This means that attraction of the contour in $u$ towards high-contrast regions in $f$ is strengthened, see Figure~\ref{fi-levellines}(c).

\medskip In summary, our findings in this section indicate that compared to geodesic active contours \eqref{gac} the amoeba active contour equation \eqref{l2aacmed} tends to attract the contour $u$ faster to high-contrast image regions and to strengthen the alignment of level lines of $u$ to those of $f$. These effects are in line with the somewhat finer adaptation of amoeba active contours to structure details that was observed in \cite{Welk-ssvm11}; see also Section~\ref{sec-exp}.

\section{Active Contours with Force Term}\label{sec-bias} From the literature on geodesic active contours \cite{Caselles-iccv95,Caselles-IJCV97,Kichenassamy-ARMA96} it is known that active contour evolutions tend to progress very slowly in image regions far from contours, and can also get stuck in undesired local minima away from the desired contour. To overcome this problem, it has been proposed already early in the active contour literature \cite{Caselles-NUMA93,Cohen-CVGIPIU91,Malladi-PAMI95} to introduce an additional \emph{force term} similar to morphological dilation or erosion into the active contour evolution.  An erosion force pushes the evolving contour in inward direction, allowing a faster evolution in homogeneous image areas and also the escape from local minima. A dilation force can be used to push the contour in outward direction, thereby also enabling the model to be used with an initial contour inside the desired segment.

\subsection{Force Terms in Active Contour PDEs} In \cite{Cohen-CVGIPIU91}, the so-called \emph{balloon force} is stated as a constant $k$ times the normal vector of the contour, equivalent to $k\,\lvert\bm{\nabla}u\rvert$ for the level set evolution, i.e.\ a plain erosion or dilation (dependent on the sign of $k$).  The geodesic active contour equation with such a force term would read \begin{align} u_t &= \lvert\bm{\nabla}u\rvert\,\mathrm{div}\left( g(\lvert\bm{\nabla}f\rvert)\,\frac{\bm{\nabla}u}{\lvert\bm{\nabla}u\rvert} \right) + k\,\lvert\bm{\nabla}u\rvert \;.  \label{gac-force-Cohen} \end{align} Already in \cite{Cohen-CVGIPIU91} it is mentioned that $k$ may be steered in such a way that erosion or dilation is reduced at high-gradient locations in order to achieve a more precise localisation of the final contour. Following this direction, \cite{Caselles-NUMA93,Kichenassamy-ARMA96,Malladi-PAMI95} couple $k$ to the edge-stopping function $g$ of the actual active contour evolution. The curve evolution in \cite{Kichenassamy-ARMA96} therefore reads \begin{align} u_t &= \lvert\bm{\nabla}u\rvert\,\mathrm{div}\left( g(\lvert\bm{\nabla}f\rvert)\,\frac{\bm{\nabla}u}{\lvert\bm{\nabla}u\rvert} \right) + g(\lvert\bm{\nabla}f\rvert)\,k\,\lvert\bm{\nabla}u\rvert \;.  \label{gac-force-Kichenassamy} \end{align}

\subsection{Force Terms in Amoeba Active Contours} \label{ssec-aac-force} To achieve a similar effect in connection with amoeba active contours, it was proposed in \cite{Welk-ssvm11} to bias the median filter.  Considering the ordered sequence $v_0,v_1,\ldots,v_p$ of intensity values within the amoeba, it was proposed to select not the middle element $v_{p/2}$, but either the element with index $q\,p$ for some $q\in[0,1]\setminus\{1/2\}$ (the $q$-quantile), or the element with index $p/2+b$ where $b$ is a fixed offset.  We will refer to these modifications as \emph{quantile bias} and \emph{fixed offset bias,} respectively.

In Subsection~\ref{ssec-bias-ana} we will analyse the correspondence between such a modified amoeba active contour method and force terms in the corresponding PDEs.  Since the PDE \eqref{genaacmed} approximated by amoeba active contours already coincides with geodesic active contours only in special cases, we expect also here an exact equivalence only in special cases.  Our analysis will eventually lead us to propose a new \emph{quadratic bias} strategy in which the offset is chosen proportional to the squared amoeba size.  Given that the correspondence between amoeba filters and PDEs involves the limit $\varrho\to0$, it is important to note that the procedures described above, with fixed quantile parameter $q$ or offset $b$, refer to a fixed amoeba radius $\varrho$. It will therefore be of particular interest whether and how $q$ or $b$ needs to be adapted for varying $\varrho$ in order to take comparable effect over different amoeba sizes, and therefore to allow a consistent limit for $\varrho\to0$.

\subsection{Analysis of the Correspondence}\label{ssec-bias-ana} In the space-continuous amoeba model, the above-mentioned modifications are implemented by choosing $\mu$ in such a way that the area difference $\lvert\mathcal{A}_+\rvert-\lvert\mathcal{A}_-\rvert$ is equalled not to zero as in Subsections~\ref{ssec-imbalance} and~\ref{ssec-median} but to some value $\delta\mathcal{A}$.  To mimick the two modification strategies mentioned in Subsection~\ref{ssec-aac-force}, $\delta\mathcal{A}$ can be chosen as constant for the fixed offset bias, or proportional to $\lvert\mathcal{A}\rvert$ to implement a quantile bias.  We remark that $\delta\mathcal{A}$ may still depend on $\varrho$.

With this modification, equation \eqref{compens} is changed to \begin{align} & 2\,\frac{\mu-u(\bm{x}_0)}{\lvert\bm{\nabla}u\rvert} \bigl(z_+(\varphi+\alpha+\pi/2)+z_-(\varphi+\alpha+\pi/2)\bigr) \notag\\*&\quad{} = \varDelta_1+\varDelta_2 + \delta\mathcal{A} + \LandauO(\varrho^4)\;, \label{compens-force} \end{align} such that the time step $u^{k+1}-u^k$ of the explicit PDE scheme being approximated is modified by an additional summand \begin{equation} \delta m := +\frac{\delta\mathcal{A}}{2\,\varrho}\, \nu(\lvert\bm{\nabla}f\rvert\,\sin\alpha)\,\lvert\bm{\nabla}u\rvert \;.  \label{aacgac-force-dm} \end{equation} In order for the biased filter to yield a finite limit for $\varrho\to0$, and thus a force term in the limiting PDE, it is necessary that $\delta m$ behaves as $\mathcal{O}(\varrho^2)$.

\subsubsection{Fixed Offset Bias} Considering first the case of a fixed offset, i.e.\ $\delta\mathcal{A}$ is constant (with regard to $u$ and $f$ but not necessarily to $\varrho$), we see that $\delta m\sim\varrho^2$ is obtained if \begin{equation} \delta\mathcal{A}=\frac{\gamma_b}3\,\varrho^3 \end{equation} with some fixed $\gamma_b$.  Then the PDE \eqref{genaacmed} from Theorem~\ref{thm} is effectively modified by the additional summand \begin{equation} +\gamma_b\,\nu(\lvert\bm{\nabla}f\rvert\,\sin\alpha)\,\lvert\bm{\nabla}u\rvert \;.  \end{equation} This is indeed a force term of the desired type.  In the rotationally symmetric case (compare Corollary~\ref{cor2}) where $\alpha=0$ and thus $\nu(\lvert\bm{\nabla}f\rvert\,\sin\alpha)\equiv1$, it coincides exactly with the constant balloon force as in \eqref{gac-force-Cohen}.

In the general situation, the factor $\nu(\lvert\bm{\nabla}f\rvert\,\sin\alpha)$ strengthens the balloon force in regions where the level lines of $u$ and $f$ are not aligned, compare also the similar findings for the basic amoeba active contour method in Subsection~\ref{ssec-aacgac-orth}.

\subsubsection{Quantile Bias} For a $q$-quantile, in contrast, one wants the amoeba to be cut into $\mathcal{A}_+$ and $\mathcal{A}_-$ with $\lvert\mathcal{A}_-\rvert=q\,\lvert\mathcal{A}\rvert$, $\lvert\mathcal{A}_+\rvert=(1-q)\,\lvert\mathcal{A}\rvert$, such that $\delta\mathcal{A}$ needs to be chosen as $(1-2\,q)\,\lvert\mathcal{A}\rvert$.

To insert this into \eqref{aacgac-force-dm}, we need to calculate $\lvert\mathcal{A}\rvert$. Similar to \eqref{Delta1}, one has \begin{align} \lvert\mathcal{A}\rvert &= \int\limits_{-\pi}^{\pi}\frac12\,z_+(\vartheta)^2\dd\vartheta =\frac{\varrho^2}2\,K(\lvert\bm{\nabla}f\rvert) \end{align} with \begin{align} K(s) := \int\limits_{-\pi}^{\pi}\frac{\dd\vartheta}{\nu( s\,\cos\vartheta)^2}\;.  \end{align} As a consequence, we have \begin{align} \delta m = \frac{\varrho}{4}\, (1-2\,q)\, \nu(\lvert\bm{\nabla}f\rvert\,\sin\alpha)\, K(\lvert\bm{\nabla}f\rvert)\, \lvert\bm{\nabla}u\rvert \end{align} from which it is evident that \begin{equation} q=\frac12-\frac{\gamma_q}{6\,\pi}\,\varrho \end{equation} with some fixed $\gamma_q$ is a suitable choice in order to have a finite force term in the PDE limit.  In this case the resulting force term to be introduced into the PDE \eqref{genaacmed} reads \begin{equation} +\gamma_q\, \nu(\lvert\bm{\nabla}f\rvert\,\sin\alpha)\, \frac{K(\lvert\bm{\nabla}f\rvert)}{2\,\pi}\, \lvert\bm{\nabla}u\rvert\;.  \end{equation}

Specialising to the $L^2$ amoeba metric $\nu(s)=\sqrt{1+s^2}$, the integral expression $K$ evaluates to \begin{align} K(s) &= \frac{2\,\pi}{\sqrt{1+s^2}\,} \end{align} such that the force term equals \begin{equation} +\gamma_q\, \frac{\sqrt{1+\lvert\bm{\nabla}f\rvert^2\sin^2\alpha}\,} {\sqrt{1+\lvert\bm{\nabla}f\rvert^2}\,}\, \lvert\bm{\nabla}u\rvert\;, \end{equation} which is in the rotationally symmetric case a term half-way between the constant balloon force as in \eqref{gac-force-Cohen} and the modulated force $\frac{k}{1+\lvert\bm{\nabla}f\rvert^2}\,\lvert\bm{\nabla}u\rvert$ as in \eqref{gac-force-Kichenassamy}. In the general case, the force term is strengthened as before by the alignment-dependent amplification factor $\sqrt{\lvert\bm{\nabla}f\rvert\,\sin\alpha}$.

\subsubsection{Quadratic Bias} The results from the previous two subsections motivate a third bias strategy in order to reproduce the force term of \eqref{gac-force-Kichenassamy} in the rotationally symmetric case with $L^2$ amoeba metric: The idea is to make $\delta\mathcal{A}$ proportional to the squared amoeba area $\lvert\mathcal{A}\rvert^2$, i.e.\ \begin{equation} \delta\mathcal{A} = \gamma_r\,\frac{\varrho}{3\,\pi^2}\, \lvert\mathcal{A}\rvert^2 \end{equation} leading to the force term \begin{equation} +\gamma_r\,\frac{\nu(\lvert\bm{\nabla}f\rvert\,\sin\alpha)} {\nu(\lvert\bm{\nabla}f\rvert)^2}\, \lvert\bm{\nabla}u\rvert\;, \end{equation} which becomes \begin{equation} +\gamma_r\, \frac{\sqrt{1+\lvert\bm{\nabla}f\rvert^2\sin^2\alpha}\,} {1+\lvert\bm{\nabla}f\rvert^2}\, \lvert\bm{\nabla}u\rvert \end{equation} for the $L^2$ amoeba metric.

Translating back to the discrete amoeba filter procedure, this means to select from the ordered sequence of intensity values within the amoeba the element with index $p/2+r\,p^2$.

\section{Experiments}\label{sec-exp} 
\begin{figure*}[t] \setlength{\unitlength}{0.001\textwidth} \begin{picture}(1000,328) \put(  0,168){\includegraphics[width=0.16\textwidth] {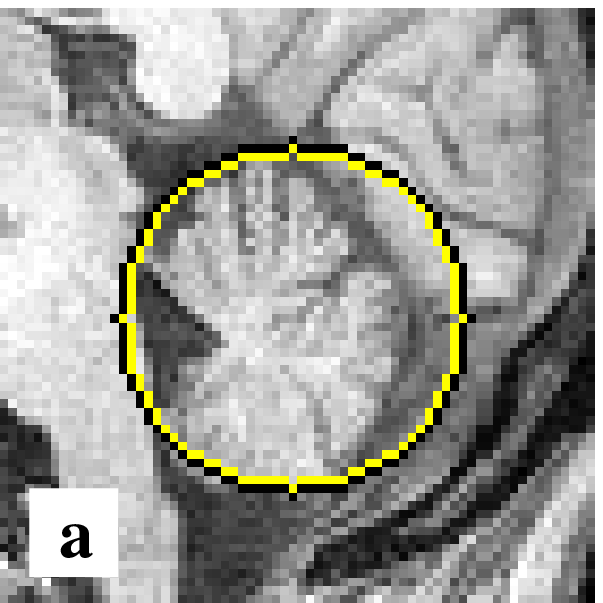}} \put(168,168){\includegraphics[width=0.16\textwidth] {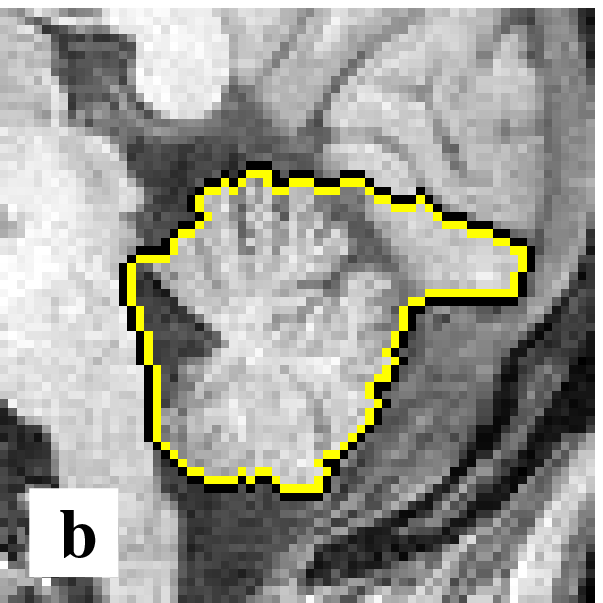}} \put(336,168){\includegraphics[width=0.16\textwidth] {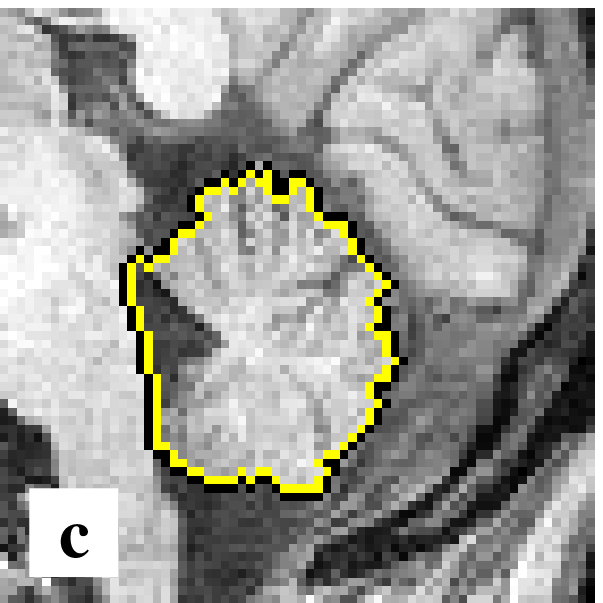}} \put(504,168){\includegraphics[width=0.16\textwidth] {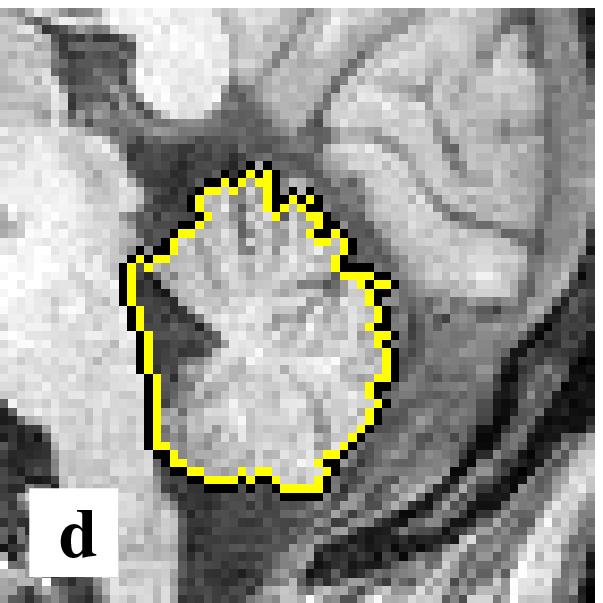}} \put(672,168){\includegraphics[width=0.16\textwidth] {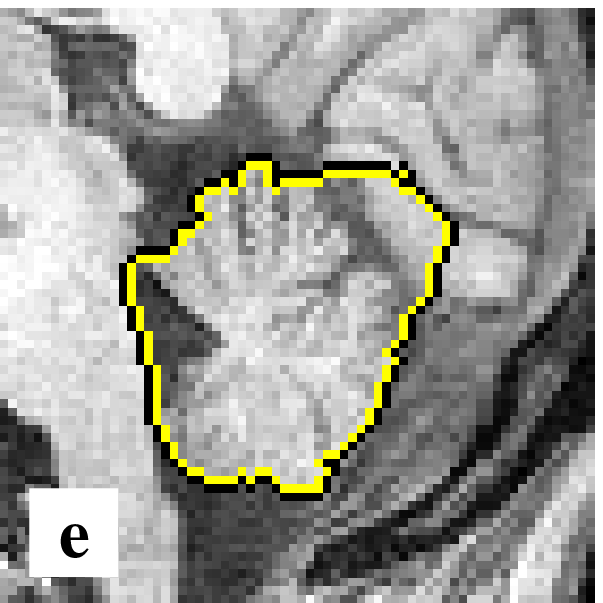}} \put(840,168){\includegraphics[width=0.16\textwidth] {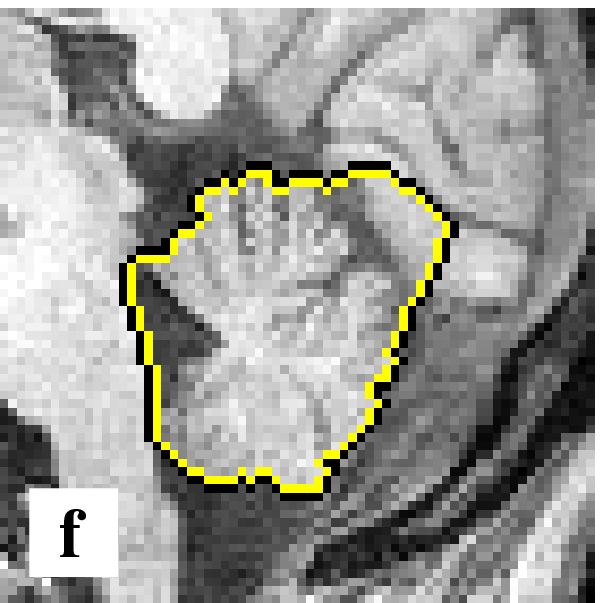}} \put(168,  0){\includegraphics[width=0.16\textwidth] {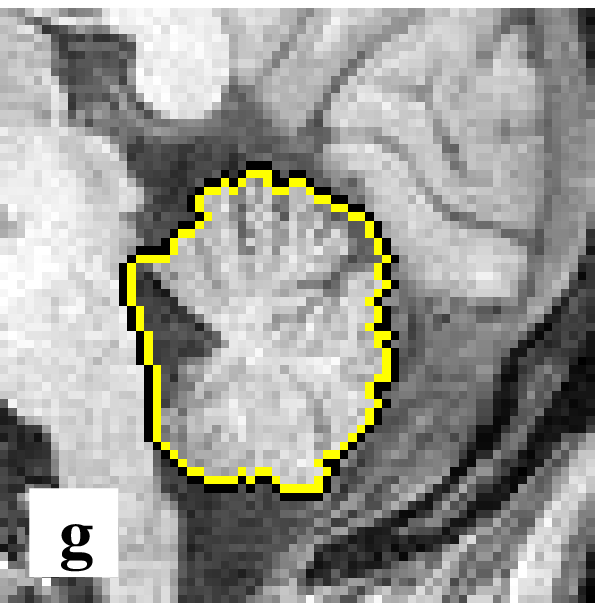}} \put(336,  0){\includegraphics[width=0.16\textwidth] {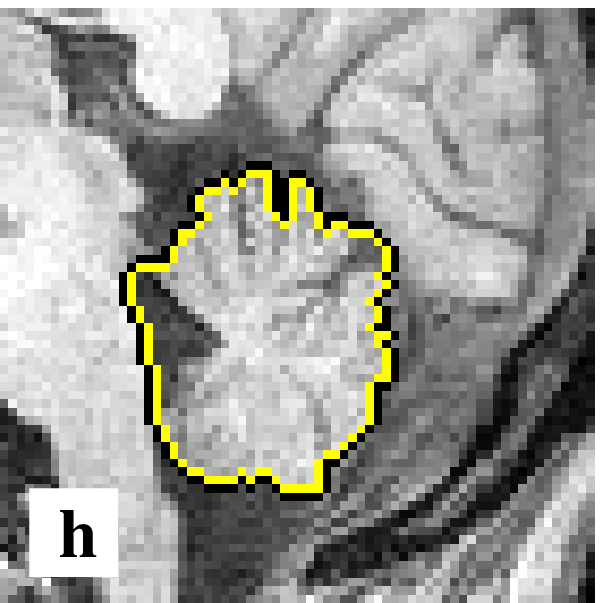}} \put(504,  0){\includegraphics[width=0.16\textwidth] {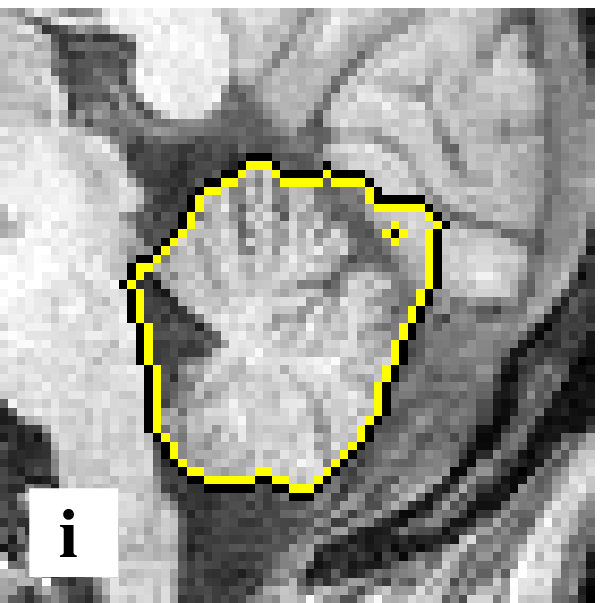}} \put(672,  0){\includegraphics[width=0.16\textwidth] {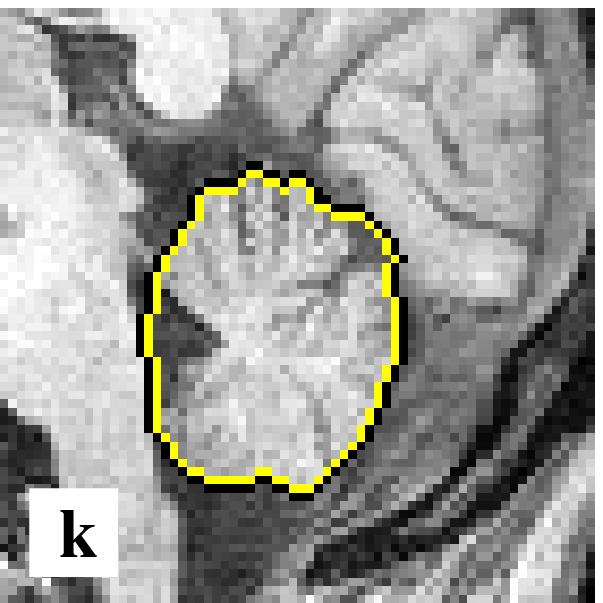}} \put(840,  0){\includegraphics[width=0.16\textwidth] {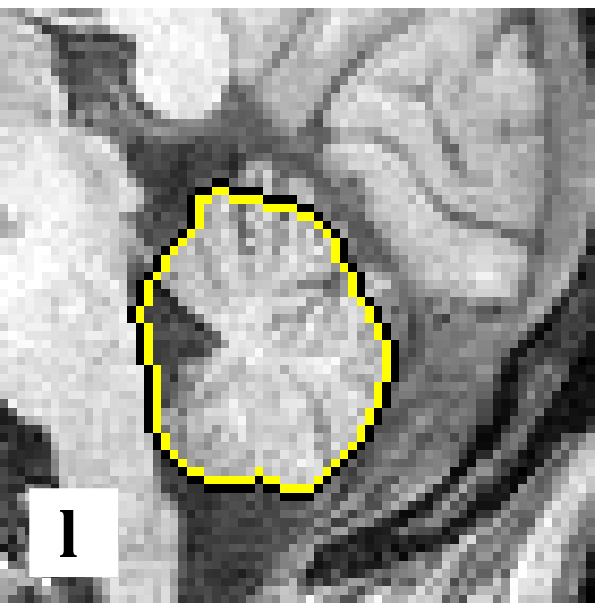}} \end{picture} \caption{\label{fi-cerebellum}Cerebellum segmentation with amoeba and geodesic active contours.  \textbf{Top left to bottom right in rows:} \textbf{(a)} Detail from an MR slice of a human brain with initial contour. -- \textbf{(b)} AAC with $L^2$ amoeba metric, $\beta=0.1$, $\varrho=10$, $20$ iterations. -- \textbf{(c)} AAC with $L^2$ amoeba metric, $\beta=0.1$, $\varrho=12$, $10$ iterations. -- \textbf{(d)} AAC with $L^2$ amoeba metric, $\beta=0.1$, $\varrho=12$, $60$ iterations. -- \textbf{(e)} AAC with $L^1$ amoeba metric, $\beta=0.1$, $\varrho=10$, $20$ iterations. -- \textbf{(f)} AAC with $L^1$ amoeba metric, $\beta=0.1$, $\varrho=12$, $10$ iterations. -- \textbf{(g)} AAC with $L^1$ amoeba metric, $\beta=0.1$, $\varrho=12$, $60$ iterations. -- \textbf{(h)} Biased AAC with $L^1$ amoeba metric, $\beta=0.1$, $\varrho=10$, fixed offset bias with $b=2$, $20$ iterations. -- \textbf{(i)} GAC with Perona-Malik edge-stopping function, $\lambda=10$, explicit scheme with time step $\tau=0.25$, $960$ iterations. -- \textbf{(k)} GAC with same parameters but $3000$ iterations. -- \textbf{(l)} Same with $57\,600$ iterations. -- (a--d) and (i--l) from \cite{Welk-ssvm11}.  } \end{figure*}

\begin{figure*}[t!] \setlength{\unitlength}{0.001\textwidth} \begin{picture}(1000,526) \put(  0,356){\includegraphics[width=0.244\textwidth] {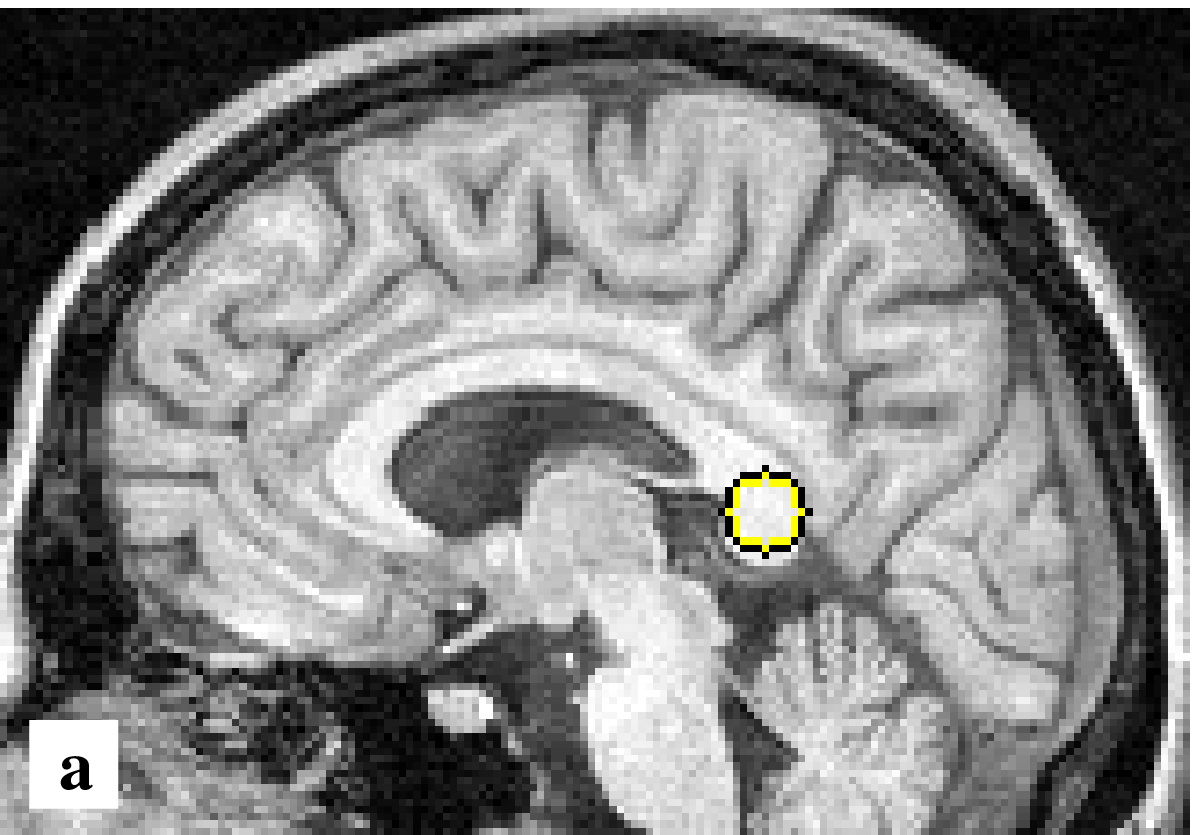}} \put(252,356){\includegraphics[width=0.244\textwidth] {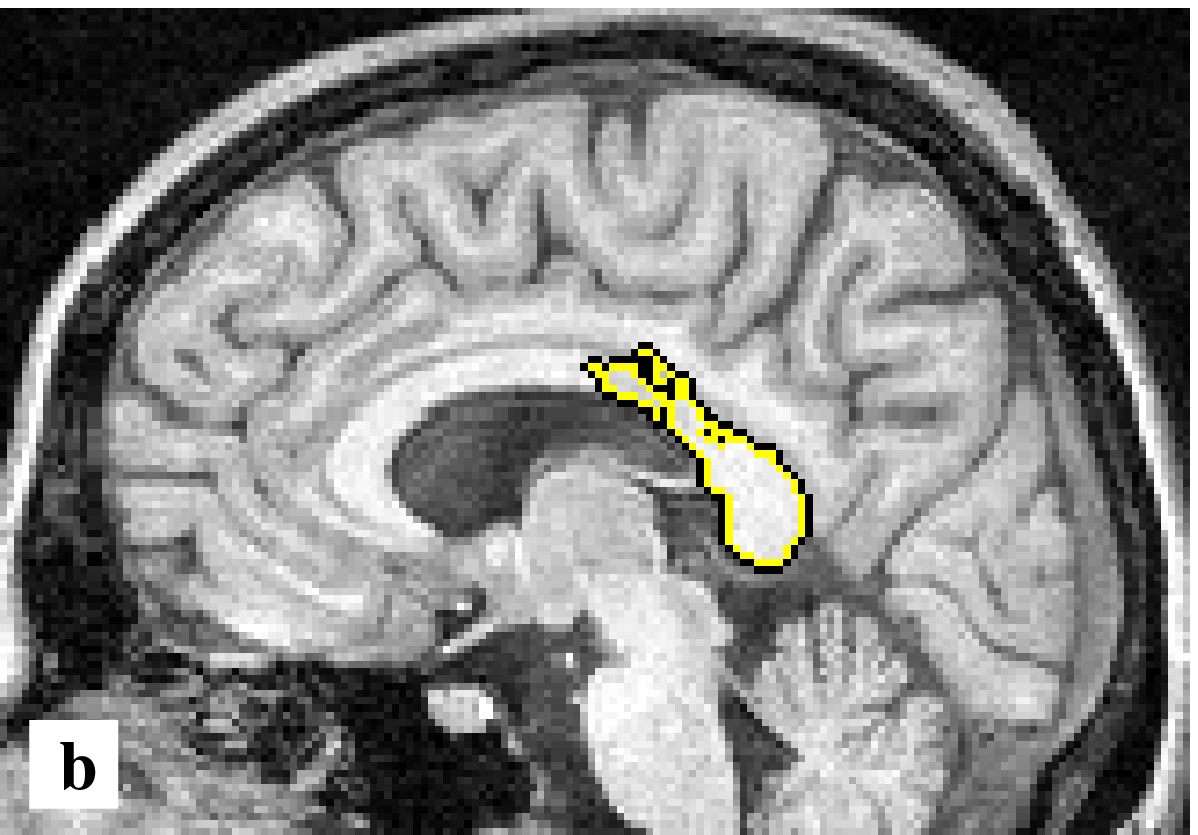}} \put(504,356){\includegraphics[width=0.244\textwidth] {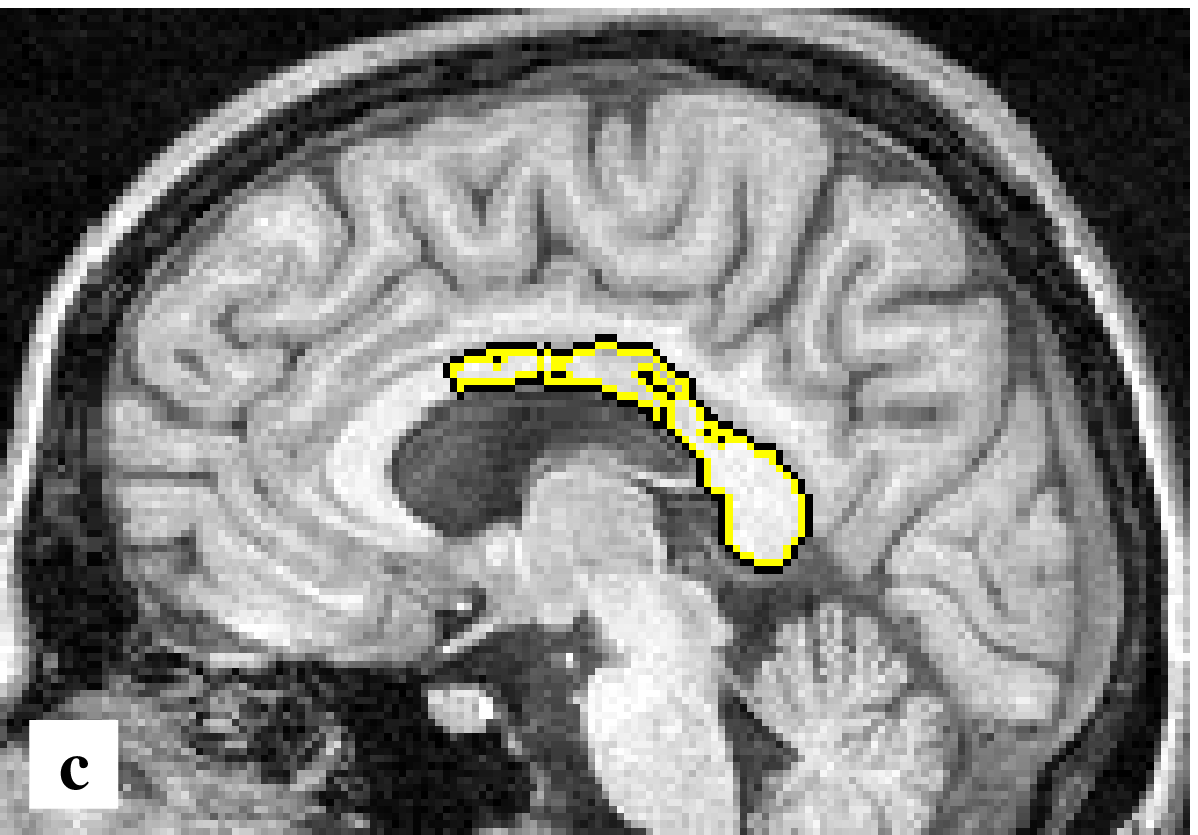}} \put(756,356){\includegraphics[width=0.244\textwidth] {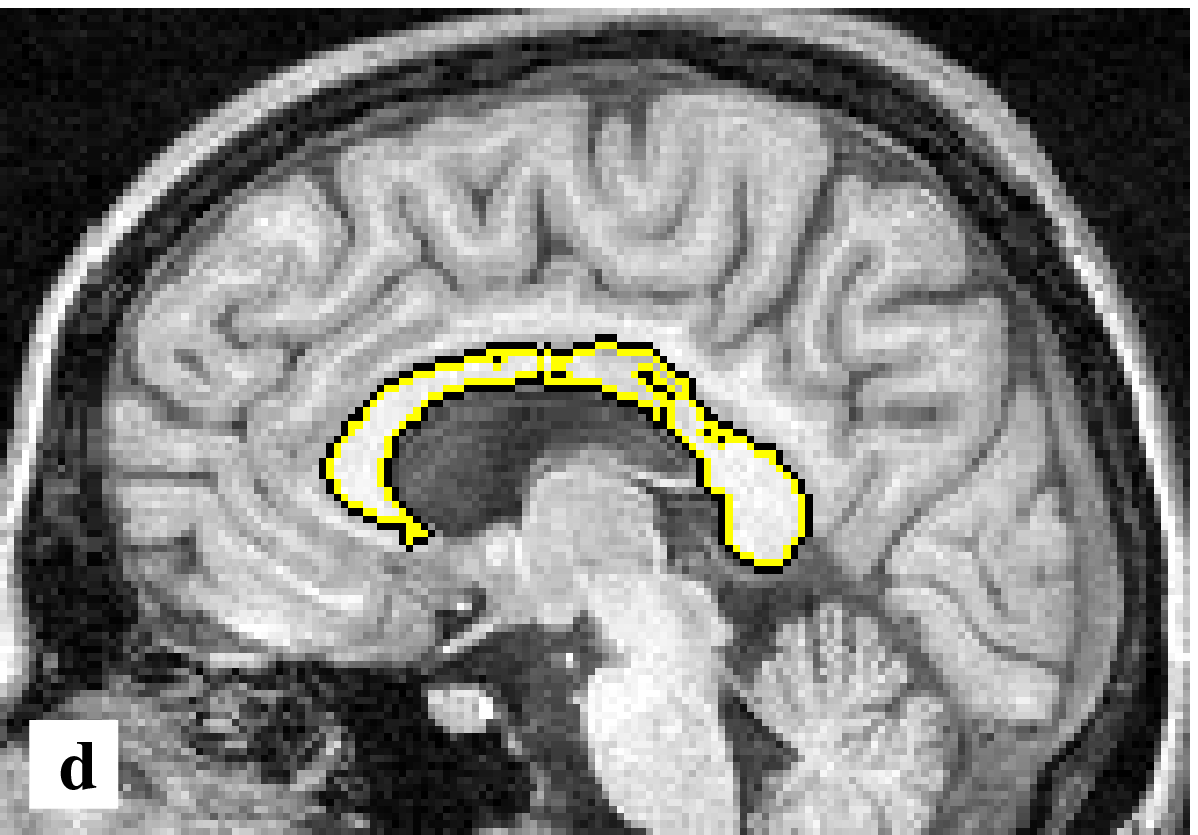}} \put(  0,178){\includegraphics[width=0.244\textwidth] {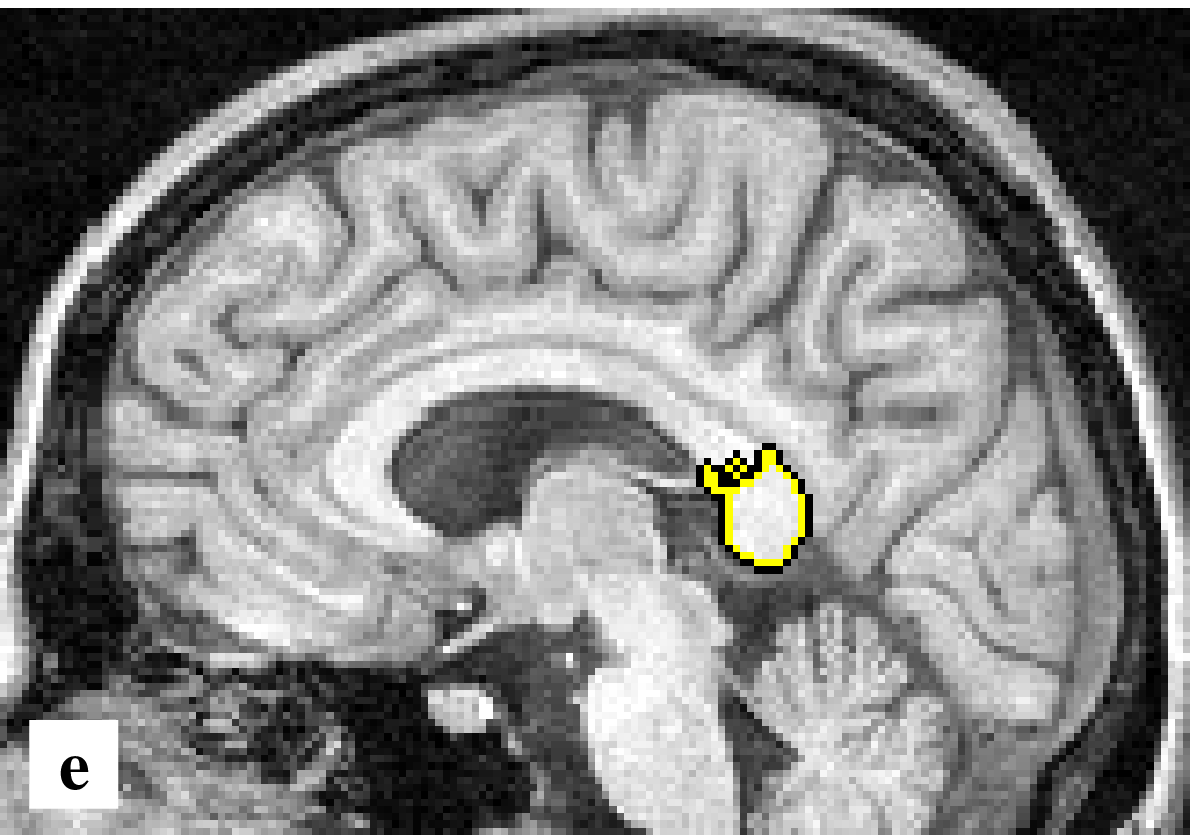}} \put(252,178){\includegraphics[width=0.244\textwidth] {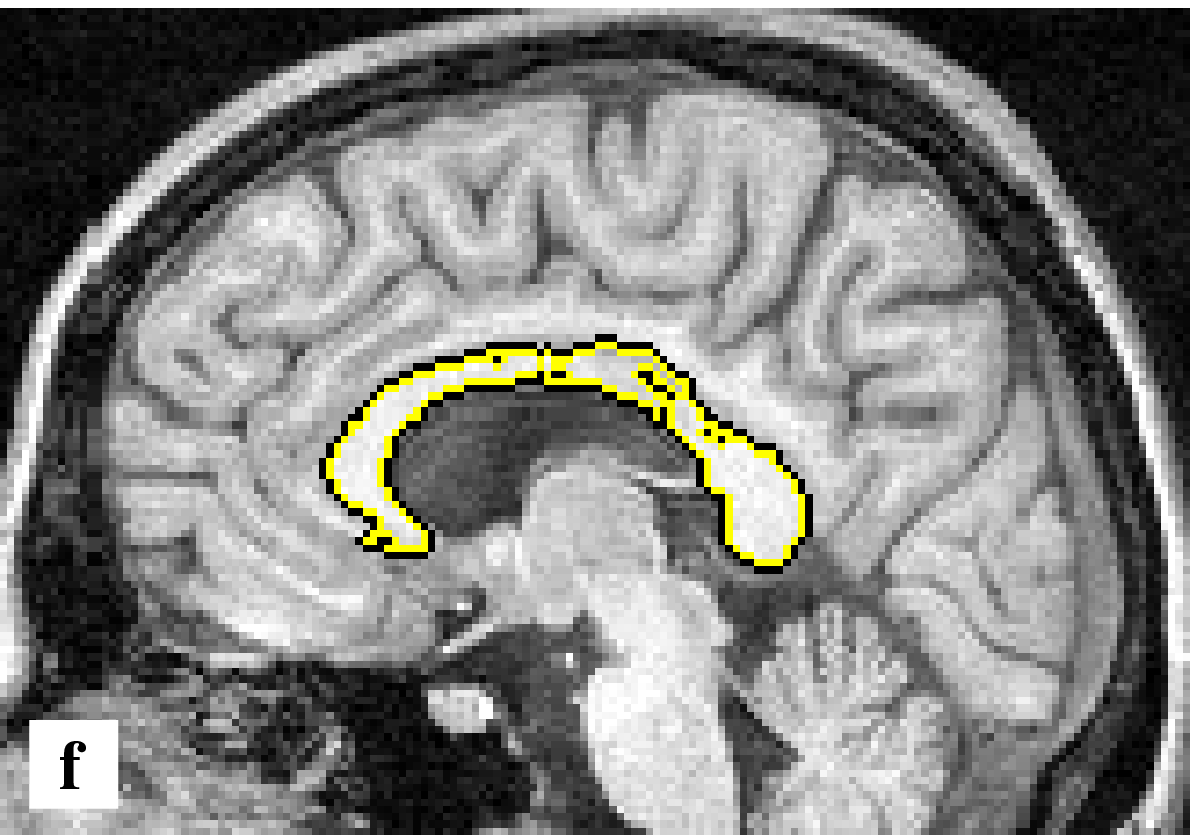}} \put(504,178){\includegraphics[width=0.244\textwidth] {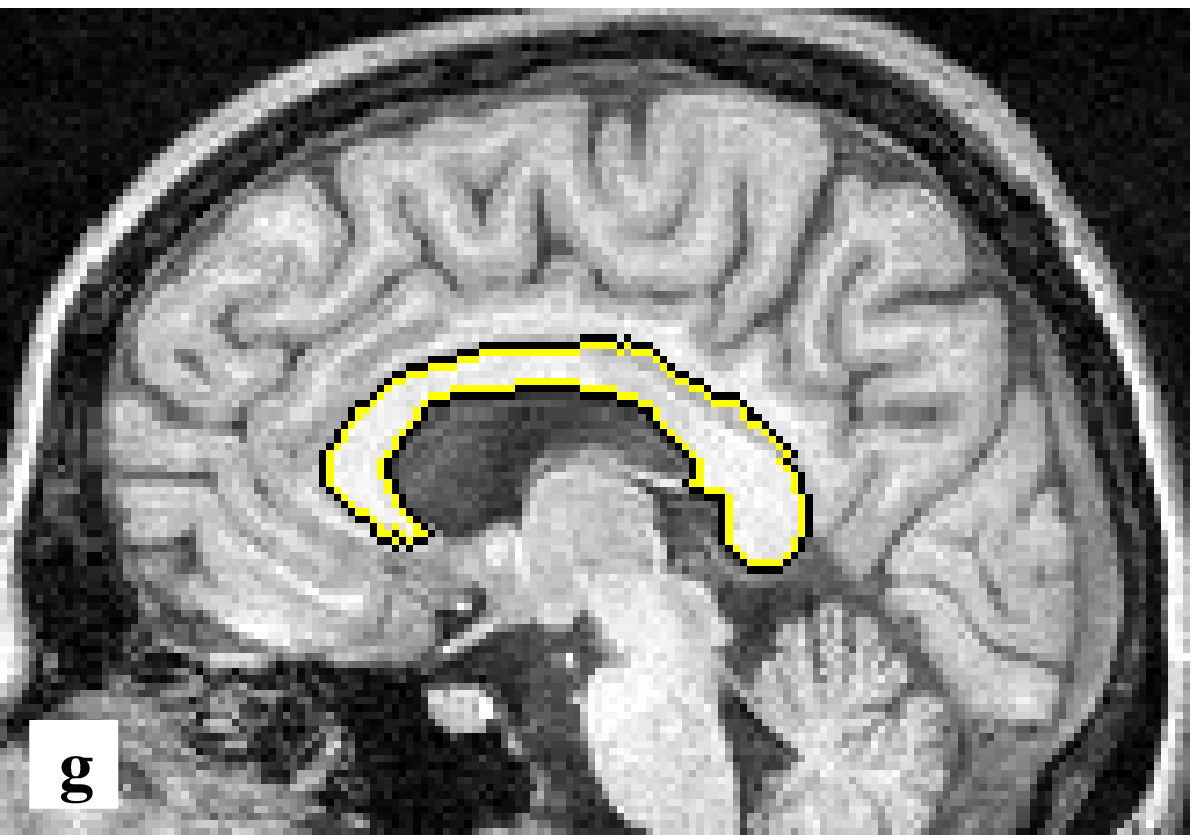}} \put(756,178){\includegraphics[width=0.244\textwidth] {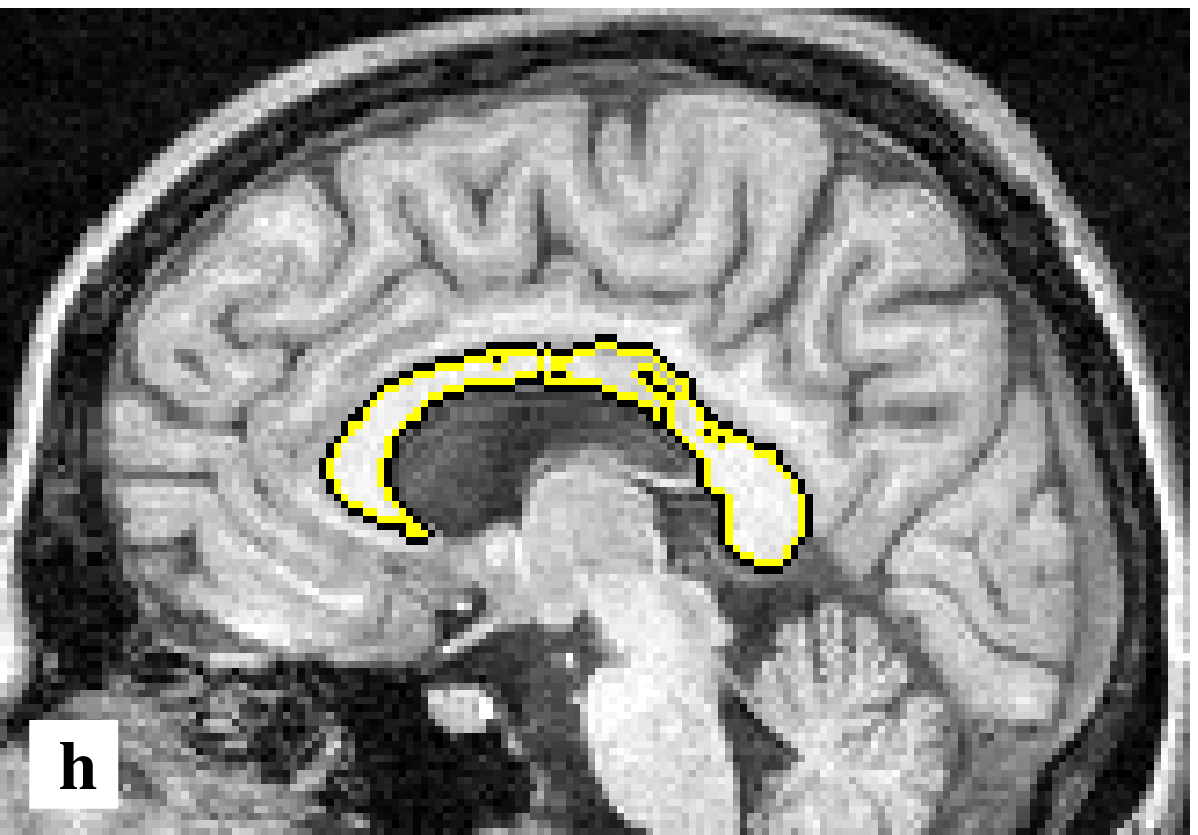}} \put(  0,  0){\includegraphics[width=0.244\textwidth] {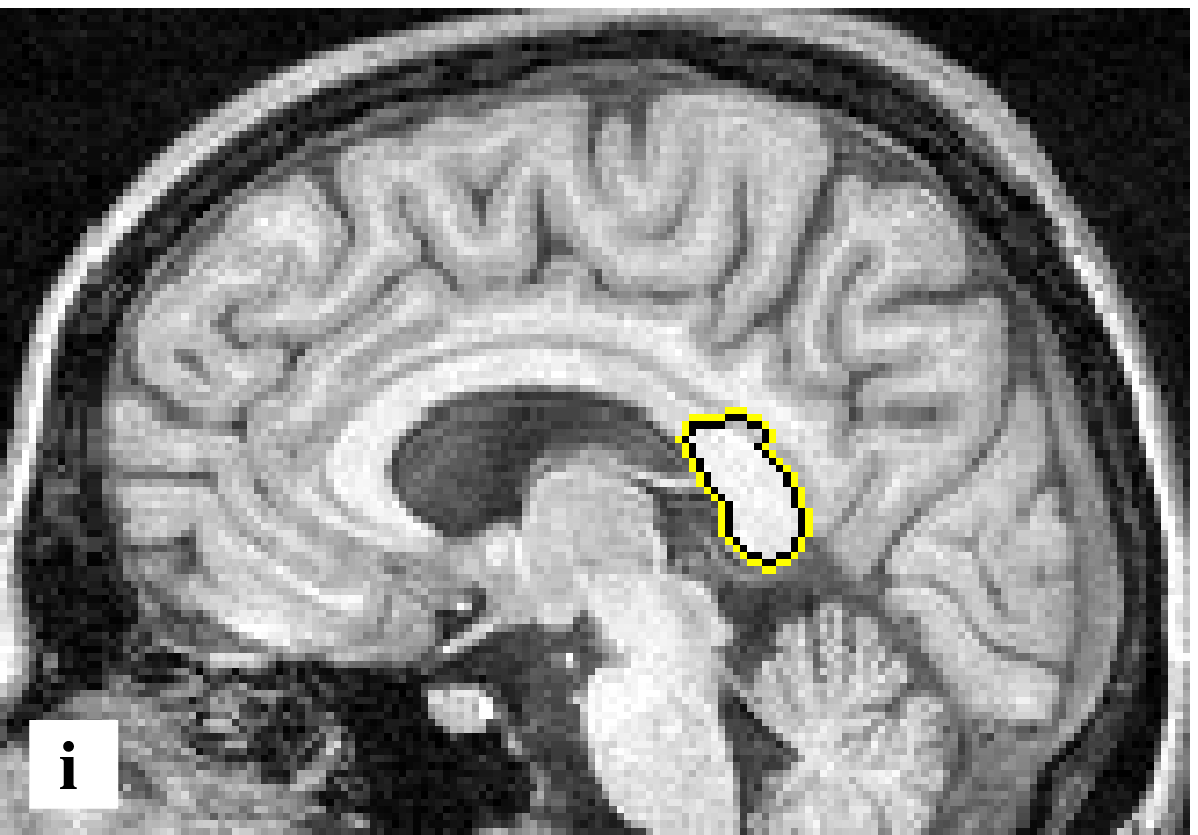}} \put(252,  0){\includegraphics[width=0.244\textwidth] {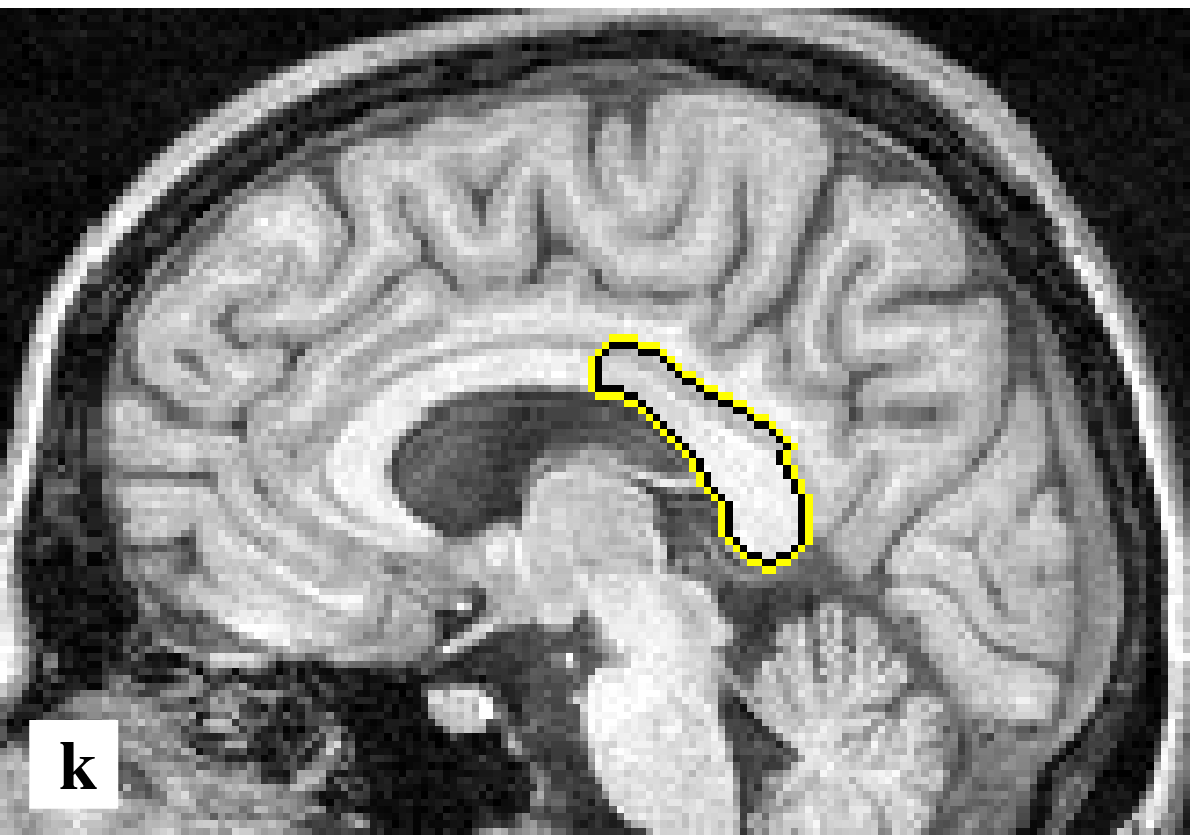}} \put(504,  0){\includegraphics[width=0.244\textwidth] {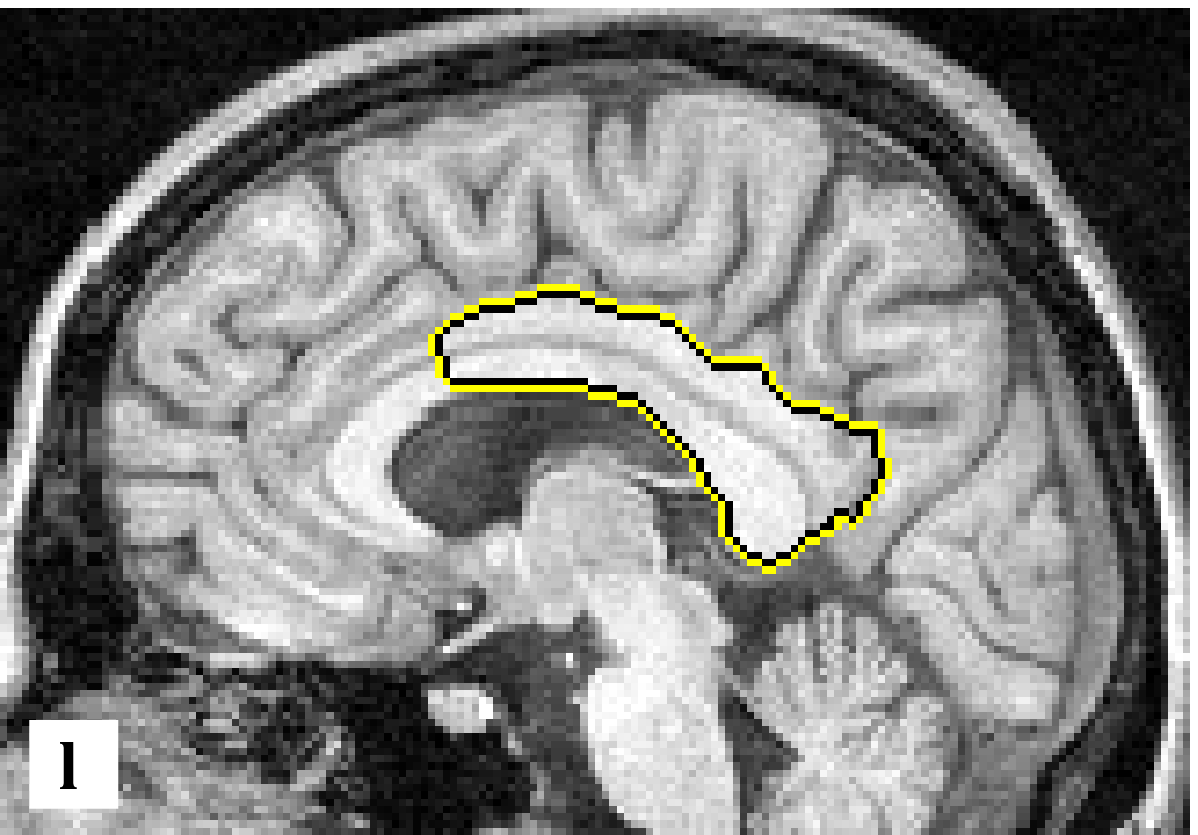}} \put(756,  0){\includegraphics[width=0.244\textwidth] {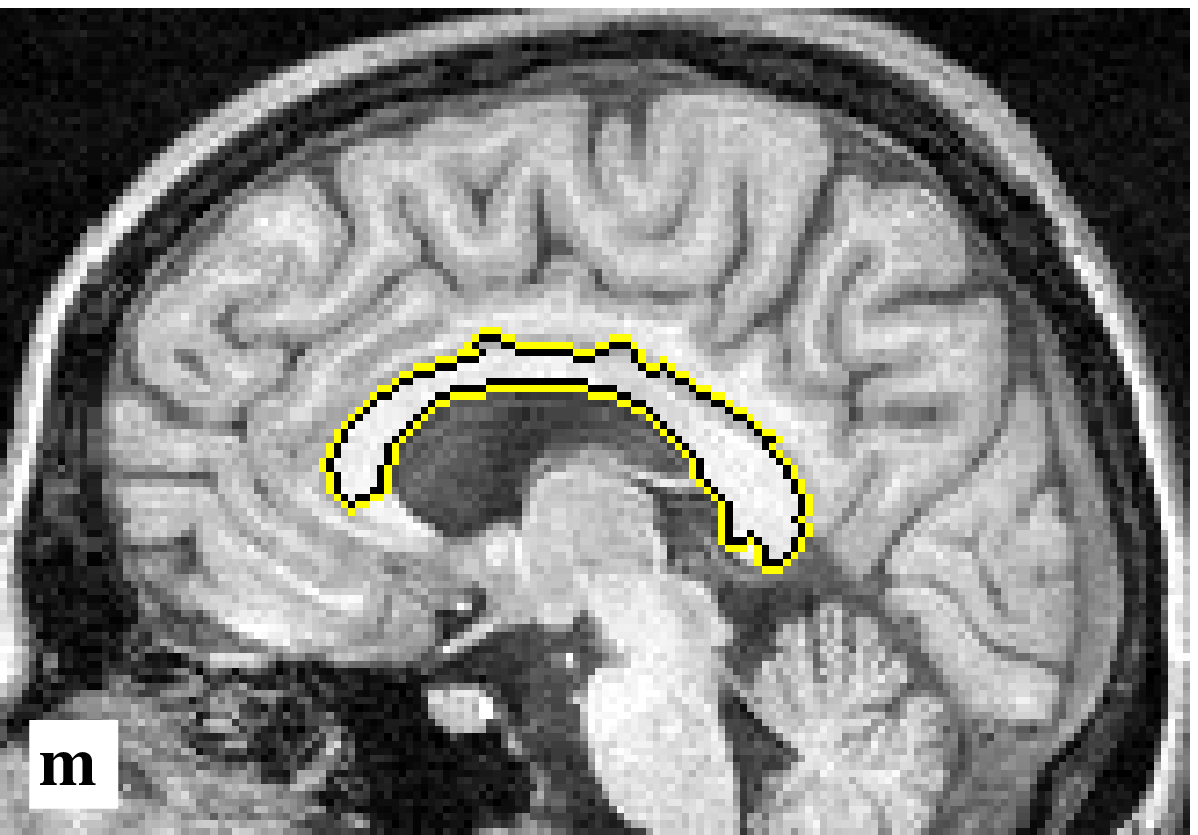}} \end{picture} \caption{\label{fi-corpuscallosum}Corpus callosum segmentation using amoeba active contours with dilation bias, and geodesic active contours with dilating force term.  \textbf{Top left to bottom right in rows:} \textbf{(a)} Detail from an MR slice of a human brain with initial contour. -- \textbf{(b)} AAC with $L^2$ amoeba metric, $\beta=2$, $\varrho=20$, fixed offset $b=10$, $10$ iterations. -- \textbf{(c)} AAC with $L^2$ amoeba metric, $\beta=2$, $\varrho=20$, $b=10$, $20$ iterations. -- \textbf{(d)} AAC with $L^2$ amoeba metric, $\beta=2$, $\varrho=20$, $b=10$, $35$ iterations. -- \textbf{(e)} AAC with $L^2$ amoeba metric, $\beta=2$, $\varrho=20$, $b= 5$, $35$ iterations. -- \textbf{(f)} AAC with $L^2$ amoeba metric, $\beta=2$, $\varrho=20$, $b=15$, $35$ iterations. -- \textbf{(g)} AAC with $L^2$ amoeba metric, $\beta=0.4$, $\varrho=20$, quantile bias with $q=0.7$, $35$ iterations. -- \textbf{(h)} AAC with $L^1$ amoeba metric, $\beta=2$, $\varrho=20$, fixed offset $b=10$, $35$ iterations. -- \textbf{(i)} GAC with Perona-Malik edge-stopping function, $\lambda=1.3$, applied to Gaussian pre-smoothed gradient field, $\sigma=1.9$, with dilation force term, $k=-0.075$, explicit scheme with time step size $\tau=0.25$, $130\,000$ iterations. -- \textbf{(k)} GAC with dilation force term, same parameters as (i) but $160\,000$ iterations. -- \textbf{(l)} GAC with dilation force term, same parameters as (i) but $270\,000$ iterations. -- \textbf{(m)} GAC with modified dilation force term (see text), $\sigma=0$, $\lambda=0.5$, $k=-0.16$, $\bar{k}=0.0005$, $\tau=0.25$, $18\,960\,000$ iterations. -- AAC experiments in (a--h) adapted from \cite{Welk-ssvm11}.  } \end{figure*}

As mentioned in the introduction, this paper aims primarily on theoretical analysis. The experiments presented in this section serve to illustrate the principal behaviour of amoeba active contours and to validate our theoretical findings on its relation to geodesic active contours.

We use magnetic resonance images in the experiments because they allow to judge the performance of active contour segmentation in the presence of high-contrast contours, small-scale details and confusing structures.  With regard to the purpose to compare two active contour formulations, comparisons to unsupervised segmentation approaches based on non-contour information or to anatomical ground truth information do not belong in this context.  One should be aware that neither active contours nor active region methods in their pure form represent the state of the art in medical image segmentation. State-of-the-art segmentation of anatomical structures for medical application is achieved by complex frameworks that include active contours or active regions as a component but combine them with anatomical knowledge encoded e.g.\ in shape and appearance models \cite{Cootes-tr01}, see \cite{Leventon-cvpr00} for an approach combining geodesic active contours with shape models.  While it will definitely be an interesting subject of future work to integrate amoeba active contours in such a framework, this is beyond the scope of the present paper.

Our first experiment aims at unsupervised active contour segmentation of the cerebellum from a human brain MR slice. The initial contour is shown in Figure~\ref{fi-cerebellum}(a). Frame~(b) shows an amoeba active contour (AAC) result with $L^2$ amoeba metric, $\beta=0.1$, amoeba radius $\varrho=10$ and $20$ iterations. The segment boundary is nicely aligned to high-contrast edges but encloses noticeable areas outside the cerebellum region. With slightly larger amoeba radius $\varrho=12$, a better segmentation is achieved, Figure~\ref{fi-cerebellum}(c). More iterations as in frame~(d) let the contour cut off parts of the cerebellum. The behaviour of amoeba active contours with $L^1$ amoeba metric is largely similar, see frames~(e--g), however due to the higher sensitivity of this amoeba metric to smaller contrasts it is the third setting with $60$ iterations that leads to a segment that surrounds the cerebellum without including too much additional area. To speed up this slow evolution of the contour an erosion bias can be used, which has been done in Figure~\ref{fi-cerebellum}(h). Indeed, a result comparable to the previous one is now reached within $20$ iterations.

Finally, Figures~\ref{fi-cerebellum}(i--l) show segmentation results obtained by geodesic active contours (GAC) with a standard explicit scheme based on forward differences in time and central differences in space except for the $\langle\bm{\nabla}g,\bm{\nabla}u\rangle$ term for which an upwind discretisation \cite{Osher-JCP88} is used.  Juxtaposition of the AAC and GAC results confirms the overall similarity between the methods which is also the result of our theoretical analysis. However, GAC have a stronger tendency to smooth away small details from the contour shape. This can partially be attributed to the inevitable numerical dissipation of standard discrete schemes for the active contour PDE. In contrast, the discrete amoeba procedure can adjust to small contour details on pixel resolution level. This aspect is of course not visible in our theoretical analysis that abstracts from the discretisation of the image and considers actually an infinite-resolution limit case.

In our second experiment we consider unsupervised segmentation of the corpus callosum from the same MR slice (a different cutout is shown).  It demonstrates how biased AAC allows initialisation inside the shape. The initial contour is shown in Figure~\ref{fi-corpuscallosum}(a).  Frames~(b--d) of the same figure show progressive stages of an outward curve evolution of AAC with fixed offset bias and $L^2$ amoeba metric, where frame~(d) forms a reasonable segmentation result. The contrast parameter $\beta=2$ in the amoeba metric is relatively large in this case.  The next two frames, (e) and (f), demonstrate that a smaller bias is not strong enough to push the contour to the desired extent while with a larger bias it overruns the desired boundaries.  Figure~\ref{fi-corpuscallosum}(g) demonstrates that a similar segmentation result can also be obtained with the quantile bias, with a smaller contrast parameter $\beta=0.4$ in this case.  In frame~(h) one sees that AAC, again with fixed offset bias but this time with $L^1$ amoeba metric, does also allow a reasonable segmentation of the same structure.

Similar segmentation results can be achieved using GAC with a dilation force term, i.e.\ \eqref{aacgac-force-dm} with negative $k$.  Figure~\ref{fi-corpuscallosum}(i--l) shows such an evolution at three evolution times. In frames (i) and (k) the segment grows successively to cover about half of the corpus callosum region but then it extends to include further regions of brain matter (l).  Due to the thin structures that separate structures in this image, this test case turns out rather hard for GAC segmentation: In most parameter settings tested, the gap between the corpus callosum and the brain matter above even earlier in the evolution. In (m) a result is shown in which the segmented region covers most of the corpus callosum. This was achieved with carefully selected parameters and a slight modification of the dilation force: the force term $g(\lvert\bm{\nabla}f\rvert)\,k\,\lvert\bm{\nabla}u\rvert$ was replaced with $S_{\bar{k}}\bigl(g(\lvert\bm{\nabla}f\rvert)\,k\bigr) \lvert\bm{\nabla}u\rvert$, where $S_{\bar{k}}(X):=\sgn(X)\cdot \max \{ \lvert X\rvert-\bar{k},0\}$ is a soft shrinkage function, and $\bar{k}$ a small shrinkage parameter. Compared to \eqref{aacgac-force-dm}, this modification suppresses the dilation force at high gradient locations, thus allowing the evolution to lock in easier at edges.  However, the very small $\lambda$ parameter in Figure~\ref{fi-corpuscallosum}(m) slows down the GAC evolution a lot, thus necessitating more than 100 times as many iterations as for frame~(k).

Closer inspection of the AAC results in Figure~\ref{fi-corpuscallosum}(b--h) reveals that the detected segments consistently exclude a small number of pixels inside the corpus callosum region, which are visible as small ``isles''. Such a behaviour is not observed in the GAC results in Figure~\ref{fi-corpuscallosum}(i--m). This is another expression of the high sensitivity with which amoebas adapt to image structures on pixel scale.  In fact the isolated pixels are noise pixels with high contrast relative to their neighbourhood. By its high spatial resolution, the AAC model keeps these pixels out of the segment while the more dissipative numerics of the GAC model smoothes the contrast and thereby incorporates the pixels in the segment.

\section{Conclusion}\label{sec-conc} In this paper, we have presented a contribution deepening the theoretical understanding of the relations between adaptive morphology and PDE methods.  We have analysed the amoeba active contour method proposed in \cite{Welk-ssvm11} and derived a partial differential equation that it approximates asymptotically for vanishing structuring element size for general geometric situations and for general amoeba metric. 

Our result reproduces as special cases results from our previous work: the approximation of geodesic active contours in a special geometric setting \cite{Welk-ssvm11} and general geometric situation \cite{Welk-ssvm13}, both in the case of the $L^2$ amoeba metric, and the approximation of self-snakes by iterated amoeba median filtering \cite{Welk-JMIV11}. In the general case, the PDE derived for amoeba active contours differs from the geodesic active contour equation. Implications of the differences for active contour segmentation have been discussed.

Further, we have analysed modifications of the amoeba median algorithm designed to approximate dilation-/erosion-like force terms as are frequently used also with geodesic active contours. Besides two variants of such modifications already proposed in \cite{Welk-ssvm11}, we have formulated a third variant based on our analysis.

By experiments based on those in \cite{Welk-ssvm11}, the basic behaviour of amoeba active contour algorithms in comparison with geodesic active contours has been demonstrated.  A deeper experimental evaluation as well as application in practical segmentation contexts, however, remains as a task for future work.

As a further subject of ongoing research we shortly mention the first results in \cite{Welk-ssvm13} on the relation between amoeba filtering with non-vanishing structuring elements and pre-smoothing in curvature-based PDEs. While we have not followed this direction further in the present paper, a more extensive investigation of this aspect is intended for a forthcoming paper.

\appendix \section{Details of Proofs}

\subsection{Proof of Corollary~\ref{cor-aacl1}} \label{app-proofcor-aacl1} For the $L^1$ amoeba norm, one has $\nu(s)=1+\lvert s\rvert$, thus $\nu'(s)=\sgn s$. Inserting these into \eqref{J1} yields \begin{align} J_1(s,\alpha) &= \int\limits_{\alpha-\pi/2}^{\alpha+\pi/2}\!\!  \frac{\sin^2\vartheta\,\,\sgn\!\cos\vartheta} {(1+s\,\lvert\cos\vartheta\rvert)^4}\dd\vartheta \notag\\* &=\int\limits_{\alpha-\pi/2}^{\pi/2}\!\!  \frac{\sin^2\vartheta} {(1+s\cos\vartheta)^4}\dd\vartheta -\int\limits_{\pi/2}^{\alpha+\pi/2}\!\!  \frac{\sin^2\vartheta} {(1-s\cos\vartheta)^4}\dd\vartheta \notag\\* &=\int\limits_{\alpha-\pi/2}^{\pi/2}\!\!  \frac{\sin^2\vartheta} {(1+s\cos\vartheta)^4}\dd\vartheta -\int\limits_{\pi/2-\alpha}^{\pi/2}\!\!  \frac{\sin^2\vartheta} {(1+s\cos\vartheta)^4}\dd\vartheta \notag\\* &=\int\limits_{-\pi/2+\alpha}^{\pi/2-\alpha}\!\!  \frac{\sin^2\vartheta} {(1+s\cos\vartheta)^4}\dd\vartheta \end{align} where we have assumed without loss of generality $\alpha\in[0,\pi]$.  Evaluating the indefinite integrals \begin{align} \int\frac{\sin^2\vartheta} {(1+s\cos\vartheta)^4}\dd\vartheta &= \begin{cases} \frac1{2\,(s^2-1)^{5/2}} \ln\frac{\sqrt{s+1}+\sqrt{s-1}\tan\frac\vartheta2} {\sqrt{s+1}-\sqrt{s-1}\tan\frac\vartheta2}\;,&s>1\;,\\ \frac1{(1-s^2)^{5/2}} \arctan\left(\sqrt{\frac{1-s}{1+s}}\,\tan\frac\vartheta2\right)\;, &s<1\end{cases} \notag\\*&\qquad{} -\frac{ \left\{ \begin{array}{@{}r@{}} \sin\vartheta\bigl((2\,s^3+s)\cos^2\vartheta+3(s^2+1)\cos\vartheta\\ -2\,s^3+5\,s\bigr) \end{array} \right.  } {6(s^2-1)^2(1+s\cos\vartheta)^2} \end{align} and \begin{align} & \int\frac{\sin^2\vartheta} {(1+\cos\vartheta)^4}\dd\vartheta = \frac{\sin^3\vartheta\,(4+\cos\vartheta)}{15\,(1+\cos\vartheta)^4} \end{align} at the integration boundaries $\pm(\pi/2-\alpha)$ and inserting $\tan\left(\frac\pi4-\frac\alpha2\right)=\frac{\cos\alpha}{1+\sin\alpha}$ yields \eqref{J1L1sgt1}, \eqref{J1L1seq1} and \eqref{J1L1slt1}.  The proof for $\alpha\in[-\pi,0]$ is analogous but the integral is split for the $\sgn\cos\vartheta$ factor at $-\pi/2$, finally leading to integration boundaries $\pm(\pi/2+\alpha)$. As a consequence, all instances of $\sin\alpha$ are replaced with $-\sin\alpha$, which is subsumed by the use of $\lvert\sin\alpha\rvert$ in \eqref{J1L1sgt1}, \eqref{J1L1seq1} and \eqref{J1L1slt1}.  Analogously, one has for $\alpha\in[0,\pi]$
\begin{align} J_3(s,\alpha) &= \int\limits_{\alpha-\pi/2}^{\alpha+\pi/2} \frac{\cos^2\vartheta\,\sgn\cos\vartheta} {(1+s\,\lvert\cos\vartheta\rvert)^4}\dd\vartheta =\int\limits_{-\pi/2+\alpha}^{\pi/2-\alpha} \frac{\cos^2\vartheta} {(1+s\cos\vartheta)^4}\dd\vartheta \end{align} which is evaluated via the indefinite integrals \begin{align} \int\frac{\cos^2\vartheta} {(1+s\cos\vartheta)^4}\dd\vartheta &= \begin{cases} \frac{-4\,s^2-1}{2\,(s^2-1)^{7/2}} \ln\frac{\sqrt{s+1}+\sqrt{s-1}\tan\frac\vartheta2} {\sqrt{s+1}-\sqrt{s-1}\tan\frac\vartheta2}\;,&s>1\;,\\ \frac{4\,s^2+1}{(1-s^2)^{7/2}} \arctan\left(\sqrt{\frac{1-s}{1+s}}\,\tan\frac\vartheta2\right)\;, &s<1\end{cases} \notag\\*&\qquad{} +\frac{ \left\{ \begin{array}{@{}r@{}} \sin\vartheta\bigl( (6\,s^5+10\,s^3-s)\cos^2\vartheta\\ +3(2\,s^4+9\,s^2-1)\cos\vartheta\\ +(2\,s^3+13\,s) \bigr) \end{array} \right.  } {6(s^2-1)^2(1+s\cos\vartheta)^2}\;, \\ \int\frac{\cos^2\vartheta} {(1+\cos\vartheta)^4}\dd\vartheta &=\frac{2\,\sin\vartheta\,(13\cos^3\vartheta+52\cos^2\vartheta +32\cos\vartheta+8)}{105\,(1+\cos\vartheta)^4} \end{align} to obtain \eqref{J3L1sgt1}, \eqref{J3L1seq1} and \eqref{J3L1slt1}.  As before, the inclusion of the case $\alpha\in[-\pi,0]$ implies the use of $\lvert\sin\alpha\rvert$ in all three equations.  Finally, one has for $J_2$ and $\alpha\in[0,\pi]$
\begin{align} J_2(s,\alpha) &= \int\limits_{\alpha-\pi/2}^{\alpha+\pi/2}\!\!  \frac{\sin\vartheta\,\cos\vartheta\,\,\sgn\!\cos\vartheta} {(1+s\,\lvert\cos\vartheta\rvert)^4}\dd\vartheta \notag\\* &=\int\limits_{\alpha-\pi/2}^{\pi/2}\!\!  \frac{\sin\vartheta\,\cos\vartheta} {(1+s\cos\vartheta)^4}\dd\vartheta -\!\!\int\limits_{\pi/2}^{\alpha+\pi/2}\!\!  \frac{\sin\vartheta\,\cos\vartheta} {(1-s\cos\vartheta)^4}\dd\vartheta \notag\\* &=\int\limits_{\alpha-\pi/2}^{\pi/2}\!\!  \frac{\sin\vartheta\,\cos\vartheta} {(1+s\cos\vartheta)^4}\dd\vartheta -\!\!\int\limits_{-\pi/2}^{\alpha-\pi/2}\!\!  \frac{\sin\vartheta\,\cos\vartheta} {(1+s\cos\vartheta)^4}\dd\vartheta \notag\\* &=\underbrace{\int\limits_{-\pi/2}^{\pi/2}\!\!  \frac{\sin\vartheta\,\cos\vartheta} {(1+s\cos\vartheta)^4}\dd\vartheta}_{{}=0} -~2\!\!\!\!\int\limits_{-\pi/2}^{\alpha-\pi/2}\!\!  \frac{\sin\vartheta\,\cos\vartheta} {(1+s\cos\vartheta)^4}\dd\vartheta \notag\\* &=~2\!\!\!\!\int\limits_{\pi/2-\alpha}^{\pi/2}\!\!  \frac{\sin\vartheta\,\cos\vartheta} {(1+s\cos\vartheta)^4}\dd\vartheta \end{align} and the indefinite integral \begin{align} &\int\frac{\sin\vartheta\,\cos\vartheta}{(1+s\cos\vartheta)^4}\dd\vartheta =\frac{3\,s\cos\vartheta+1}{6\,s^2\,(1+s\,\cos\vartheta)^3} \end{align} from which \eqref{J2L1} is obtained in a straightforward way.  As before, the case $\alpha\in[-\pi,0]$ is subsumed by inserting modulus bars around $\sin\alpha$.

\subsection{Relation between $\tilde{J}_1$ and $\tilde{J}_3$} \label{app-j1j3} To complete the proof of Corollary~\ref{cor1}, we show that $h(s)=s\,g'(s)$.  We notice first that \begin{align} &\frac{\dd}{\dd\vartheta}\left(\frac{\sin\vartheta} {\nu(\beta\,s\cos\vartheta)^3}\right) = \frac{\cos\vartheta}{\nu(\beta\,s\cos\vartheta)^3} +3\,\beta\,s\,\frac{\nu'(\beta\,s\cos\vartheta)} {\nu(\beta\,s\cos\vartheta)^4}\sin^2\vartheta \label{dsinovernu3} \end{align} where the last summand is essentially the integrand of $\tilde{J}_1(\beta\,s)$.  By integration it follows that \begin{align} 3\,\beta\,s\,\tilde{J}_1(\beta\,s) &= \underbrace{\left[\frac{\sin\vartheta}{\nu(\beta\,s\cos\vartheta)^3}\right] ^{\vartheta=+\pi/2}_{\vartheta=-\pi/2}}_{{}=2} \!\!-\!\!  \int\limits_{-\pi/2}^{+\pi/2}\!\!  \frac{\cos\vartheta}{\nu(\beta\,s\cos\vartheta)^3}\dd\vartheta\;.  \end{align} Substituting this into \eqref{g} yields \begin{align} g(s)&=\frac12\int\limits_{-\pi/2}^{+\pi/2} \frac{\cos\vartheta}{\nu(\beta\,s\cos\vartheta)^3}\dd\vartheta \end{align} from which one easily calculates \begin{align} s\,g'(s) &= \frac{s}{2}\,\int\limits_{-\pi/2}^{+\pi/2} \frac{-3\,\beta\,\cos\vartheta}{\nu(\beta\,s\cos\vartheta)^4} \nu'(\beta\,s\cos\vartheta)\cos\vartheta\dd\vartheta =-\frac32\,\beta\,s\,\tilde{J}_3(\beta\,s) = h(s)\;.  \end{align} 

\subsection{Equivalence of Corollary~\ref{cor1} to the Result from \cite{Welk-JMIV11}} \label{app-ssn-jmiv11} In \cite{Welk-JMIV11} it was shown that iterated amoeba median filtering approximates the PDE \eqref{genaacmed} as in Corollary~\ref{cor1} with the edge-stopping function $g$ given by \begin{align} g(s) &= \frac{3\,I_1(\beta\,s)}{\beta^2\,s^2\,\psi\left(\frac1{\beta\,s} \right)^3}\;, \label{g-jmiv11} \\ I_1(\beta\,s) &= \int\limits_0^1\xi^2\,\sqrt{ \left(\psi^{-1}\left(\frac1\xi\psi\left(\frac1{\beta\,s}\right)\right) \right)^2-\frac1{\beta^2\,s^2}\,}\dd\xi\;, \end{align} where the function $\psi$ is related to $\nu$ via \begin{align} \psi(q) &= \hat{\nu}(q,1)=q\,\nu\left(\frac1q\right)\;, \label{psi-jmiv11-nu} \end{align} and $\psi^{-1}$ denotes the inverse function of $\psi$.  Substituting \begin{align} \xi &= \frac{\psi\left(\frac1{\beta\,s}\right)} {\psi\left(\frac1{\beta\,s\cos\vartheta}\right)}\;, \\ \dd\xi &= -\frac{\psi\left(\frac1{\beta\,s}\right)}{\beta\,s}\, \frac{\psi'\left(\frac1{\beta\,s\cos\vartheta}\right)} {\psi\left(\frac1{\beta\,s\cos\vartheta}\right)^2}\, \frac{\sin\vartheta}{\cos^2\vartheta}\dd\vartheta \end{align} into $I_1$ yields \begin{align} I_1(\beta\,s)&= -\int\limits_{\pi/2}^0 \frac{\psi\left(\frac1{\beta\,s}\right)^2} {\psi\left(\frac1{\beta\,s\cos\vartheta}\right)^2} \sqrt{\left(\psi^{-1}\left( \frac{\psi\left(\frac1{\beta\,s\cos\vartheta}\right)} {\psi\left(\frac1{\beta\,s}\right)}\, \psi\left(\frac1{\beta\,s}\right)\right)\right) -\frac1{\beta^2\,s^2}\,} \times{}\notag\\*&\qquad{}\times \frac{-\psi\left(\frac1{\beta\,s}\right)}{\beta\,s}\, \frac{\psi'\left(\frac1{\beta\,s\cos\vartheta}\right)} {\psi\left(\frac1{\beta\,s\cos\vartheta}\right)^2}\, \frac{\sin\vartheta}{\cos^2\vartheta}\dd\vartheta \notag\\ &=-\frac{\psi\left(\frac1{\beta\,s}\right)^3}{\beta^2\,s^2} \int\limits_0^{\pi/2} \frac{\psi'\left(\frac1{\beta\,s\cos\vartheta}\right)} {\psi\left(\frac1{\beta\,s\cos\vartheta}\right)^4}\, \frac{\sin^2\vartheta}{\cos^3\vartheta}\dd\vartheta \end{align}
where the inverse function has been cancelled due to $\psi^{-1}\circ\psi\equiv\mathrm{id}$.  Inserting this into \eqref{g-jmiv11} and rewriting $\psi$ into $\nu$ via \eqref{psi-jmiv11-nu} and \begin{equation} \psi'(q) = \nu\left(\frac1q\right)-\frac1q\,\nu'\left(\frac1q\right) \end{equation} gives \begin{align} g(s)&=\frac3{\beta^4s^4}\!  \int\limits_0^{\pi/2} \frac{\nu(\beta\,s\cos\vartheta)-\beta\,s\cos\vartheta\, \nu'(\beta\,s\cos\vartheta)} {\frac1{\beta^4s^4\cos^4\vartheta}\,\nu(\beta\,s\cos\vartheta)^4}\, \frac{\sin^2\vartheta}{\cos^3\vartheta}\dd\vartheta \notag\\ &=3\int\limits_0^{\pi/2} \frac{\sin^2\vartheta\cos\vartheta}{\nu(\beta\,s\cos\vartheta)^3} \dd\vartheta -3\,\beta\,s \int\limits_0^{\pi/2} \frac{\nu'(\beta\,s\cos\vartheta)} {\nu(\beta\,s\cos\vartheta)^4} \sin^2\vartheta\cos^2\vartheta\dd\vartheta\;.  \label{g-jmiv11-trf1} \end{align} Integration by parts using \eqref{dsinovernu3} gives for the first summand \begin{align} \int\limits_0^{\pi/2} \frac{\sin^2\vartheta\cos\vartheta}{\nu(\beta\,s\cos\vartheta)^3} \dd\vartheta &= \int\limits_0^{\pi/2} \frac{\sin\vartheta\cos\vartheta}{\nu(\beta\,s\cos\vartheta)^3}\, \frac{\sin(2\vartheta)}{2}\dd\vartheta \notag\\ &= \left[ \frac{\sin\vartheta}{\nu(\beta\,s\cos\vartheta)^3}\, \left(-\frac14\cos(2\vartheta)\right)\right] _{\vartheta=0}^{\vartheta=\pi/2} \notag\\*&\quad{} +\frac14\int\limits_0^{\pi/2} \frac{\cos\vartheta}{\nu(\beta\,s\cos\vartheta)^3}\,\cos(2\vartheta) \dd\vartheta \notag\\*&\quad{} +\frac34\,\beta\,s\int\limits_0^{\pi/2} \frac{\nu'(\beta\,s\cos\vartheta)} {\nu(\beta\,s\cos\vartheta)^4} \sin^2\vartheta\cos(2\vartheta)\dd\vartheta \notag\\ &=\frac14 +\frac14\int\limits_0^{\pi/2} \frac{\cos\vartheta}{\nu(\beta\,s\cos\vartheta)^3}\dd\vartheta -\frac12\int\limits_0^{\pi/2} \frac{\cos\vartheta\sin^2\vartheta}{\nu(\beta\,s\cos\vartheta)^3} \dd\vartheta \notag\\*&\quad{} +\frac34\,\beta\,s\int\limits_0^{\pi/2} \frac{\nu'(\beta\,s\cos\vartheta)} {\nu(\beta\,s\cos\vartheta)^4} \sin^2\vartheta\cos(2\vartheta)\dd\vartheta \end{align} and after reordering of terms and division by $3/2$
\begin{align} \int\limits_0^{\pi/2} \frac{\sin^2\vartheta\cos\vartheta}{\nu(\beta\,s\cos\vartheta)^3} \dd\vartheta &= \frac16 +\frac16\int\limits_0^{\pi/2} \frac{\cos\vartheta}{\nu(\beta\,s\cos\vartheta)^3}\dd\vartheta \notag\\*&\quad{} +\frac12\,\beta\,s\int\limits_0^{\pi/2} \frac{\nu'(\beta\,s\cos\vartheta)} {\nu(\beta\,s\cos\vartheta)^4} \sin^2\vartheta\cos(2\vartheta)\dd\vartheta\;.  \label{intsin2cosovernu3} \end{align} Making once more use of \eqref{dsinovernu3}, we calculate \begin{align} & \int\limits_0^{\pi/2} \frac{\cos\vartheta}{\nu(\beta\,s\cos\vartheta)^3}\dd\vartheta = \underbrace{\left[\frac{\sin\vartheta}{\nu(\beta\,s\cos\vartheta)^3}\right] _{\vartheta=0}^{\vartheta=\pi/2}}_{{}=1} -3\,\beta\,s\int\limits_0^{\pi/2} \frac{\nu'(\beta\,s\cos\vartheta)} {\nu(\beta\,s\cos\vartheta)^4} \sin^2\vartheta\dd\vartheta\;.  \label{intcosovernu3} \end{align} Substituting \eqref{intsin2cosovernu3} and \eqref{intcosovernu3} into \eqref{g-jmiv11-trf1} eventually leads to \begin{align} g(s) &= 1-3\,\beta\,s\int\limits_0^{\pi/2} \frac{\nu'(\beta\,s\cos\vartheta)} {\nu(\beta\,s\cos\vartheta)^4}\sin^2\vartheta(\sin^2\vartheta+\cos^2\vartheta) \dd\vartheta \notag\\ &=1-\frac32\,\beta\,s\,\tilde{J}_1(\beta\,s) \end{align} in accordance with the representation from Corollary~\ref{cor1}.  This completes the proof.
\end{document}